\documentclass{hld2026} %

\usepackage{titletoc}
\usepackage{booktabs}
\usepackage{amsmath}
\usepackage{listings}
\usepackage{umoline}
\usepackage{array}
\usepackage{mleftright}
\usepackage{amsfonts}
\usepackage{dirtytalk}
\usepackage{enumitem} 
\usepackage{mathtools}
\usepackage[english]{babel}
\usepackage{graphicx} 
\usepackage{bbm}

\DeclareMathOperator*{\argmin}{argmin}

\title[Effects of width-dependent Hyperparameters and $\ell_2$-regularization on two-layer ReLU]{Effects of width-dependent model hyperparameters and $\ell_2$-regularization on the loss landscape of two-layer ReLU networks}

\hldauthor{%
\Name{Haruka Eshima} \Email{haruka.eshima@oist.jp}\\
\addr {Okinawa Institute of Science and Technology, Okinawa, Japan} 
\AND
\Name{Makoto Yamada} \Email{makoto.yamada@oist.jp}\\
\addr {Okinawa Institute of Science and Technology, Okinawa, Japan} 
}

\begin{document}
\maketitle

\begin{abstract}
Understanding deep neural networks remains a central challenge in machine learning. In particular, the theoretical properties of even two-layer ReLU networks, especially in the presence of weight decay, remain poorly understood. To this end, we derive a sufficient condition on the hyperparameter settings under which the global minima collapse to the zero solution. Interestingly, our experiments reveal that using AdamW as an optimizer prevents the collapse of the learned parameters, whereas using SGD does not, which may help explain the success of AdamW in deep learning training. In addition, when restricting the input dimension to one, we derive an analytical solution for the globally optimal parameter sets of two-layer ReLU networks and show that $\ell_2$-regularization has a width-invariant effect on connectivity, but its dimensionality-reducing effect becomes stronger as the network width increases. These results provide insight into how width-dependent hyperparameters influence the geometry of regularized loss landscapes. 

\end{abstract}
\section{Introduction}
Loss landscape analysis characterizes quantitative variations of the loss function in finite-width machine learning models, and studying the loss landscape of two-layer linear unit activation networks (ReLU) \cite{Nair2010RectifiedLU} has been a key topic of interest in the machine learning community \citep{boursier2025benignity,karhadkar2024mildly, wang2022the, kim_exploring_2025, sonodaRidgeRegressionOverparametrized2021}. 
Recent work has extended the theoretical analysis to $\ell_2$-regularized two-layer ReLU networks, including characterizations of global optima and their geometric structure \citep{boursier2025benignity, kim_exploring_2025}. However, the effect of explicit width-dependent hyperparameter scalings \cite{yang_feature_2022, ghosh2026understanding, kosson2026weight} on these regularized loss landscapes does not appear to have been studied in a systematic way. 

In this paper, we analyze the global minima of two-layer ReLU neural networks for $\ell_2$-regularized loss in a width-dependent hyperparameter setting. We show that, when the $\ell_2$-regularization coefficient scales faster than the scaling of the network, zero is the unique global minimum for a sufficiently overparametrized model, and experimentally show that whether the learned parameters collapse to zero depends on the choice of optimizer. We further analyze the effects of $\ell_2$-regularization on the dimension and connectivity of global minima for one-dimensional input. Our results imply that the effect of $\ell_2$-regularization on the connectivity of global minima is negligible regardless of how the model width is scaled up, whereas its dimensional reduction effect becomes stronger.\\
\\
\textbf{Problem Setting.} The network model $f_{\theta}:\mathbb{R}^d\rightarrow\mathbb{R}$ is 
\begin{equation}\label{eq:g_2ReLu}
    f_{\theta}(x)=\frac{1}{\alpha}\sum_{j=1}^{m}(xu_{j})_{+}\omega_{j},
\end{equation}
where the trainable parameter of the model is $\theta\!=\!(u_i,\omega_i)_{i=1}^{m}\!\in \!(\mathbb{R}^d \times \mathbb{R})^m\!\eqqcolon\!\Theta$. $U\!\coloneq\!(u_1,~...~,u_m)\!\in\!\mathbb{R}^{d\times m}$ and $W\!\coloneq\!(\omega_1,~...~,\omega_m)^{T}\!\in\!\mathbb{R}^{m}$ are the first and second layer weights, $m$ is the number of hidden neurons, $(\cdot)_{+}=\max\{\cdot,0\}$ is the ReLU activation, and the hyperparameter $1/\alpha>0$ is the scaling factor.  Each component in the parameter is initialized independently by $u_{i,j}^{0}\sim N(0, \tau_1^2)$ and $\omega_i^{0}\sim N(0, \tau_2^2)$ where $\tau_1,\tau_2\in\mathbb{R}$ are hyperparameters. By abuse of notation, for input data $X\in\mathbb{R}^{n\times d}$, we write $f_{\theta}(X)$ to denote the output of the model $(f_{\theta}(x_1),~...~,f_{\theta}(x_n))^{T}\in\mathbb{R}^{n}$. We aim to analyze the effects of adding $\ell_2$-regularization $\frac{\beta}{2}\!\sum_{j=1}^{m}\!(u_j^{2}+\omega_{j}^{2})$ to the loss function $\ell(\theta)$. The value of the hyperparameter $\beta>0$ matches the weight decay coefficient for the gradient descent method. We introduce width-dependent hyperparameter setting $\alpha=m^{a}$, $\tau_1=m^{-b_1}$, $\tau_2=m^{-b_2}$, $\beta=m^{-\delta}$.
\section{Effects of width-dependent $\ell_2$-reguralization on loss landscape} 
\subsection{Collapse of global optima due to $\ell_2$-regularization}\label{subsec:collapse}
Theorem~\ref{thm:general_convex_zero_solution} shows that $\delta<a$ is a sufficient condition for a hyperparameter setting with which only the zero weight is the globally optimal parameter for the model \eqref{eq:g_2ReLu} with sufficiently large $m$. This collapse of global optima to zero happens for most of the convex loss functions $\ell(\theta)$ used in practice. We will discuss this result with numerical experiments in Section~\ref{sec:ne}.
\begin{theorem}\label{thm:general_convex_zero_solution}
We consider minimizing the following loss function 
\begin{equation}\label{eq:general_prob}
\min_{\theta\in\Theta}L_{\ell}(\theta)= \ell(\theta)+\!\frac{\beta}{2}\sum_{j=1}^{m}(\|u_j\|_2^{2}+\omega_{j}^{2})
\end{equation}
where $\ell$ is finite on $\Theta$, and can be written as a convex function of $f_{\theta}(X)$. For example, $\ell(\theta)=\|f_{\theta}(X)-Y\|^2$. \\ If $\delta\!<\!a$, then for sufficiently large $m$,
$\argmin_{\theta\in\Theta} L_{\ell}(\theta)\!=\!\{(0,0)_{i=1}^{m}\}$.
\end{theorem}
Theorem~\ref{thm:general_convex_zero_solution} and experimental results (Section~\ref {sec:ne} and Appendix~\ref{app:ne}) tell us undesirable values of the weight decay coefficient $\beta$ for the gradient descent algorithm.
\subsection{Change in dimension and connectivity of global optima due to $\ell_2$-regularization}\label{subsec:go}
When $\delta>a$, the regularization is not necessary strong enough to force the global minima to collapse, and our focus becomes how $\ell_2$-regularization changes the geometry of the nontrivial global minima. This question is motivated by recent work showing that $\ell_2$-regularized two-layer ReLU networks admit exact convex formulations, which is used to study optimal solution sets \cite{pilanci_neural_2020,mishkin2023optimal}. 
We specifically focus on the one-dimensional ReLU setting \eqref{eq:g_min_problem}. This setting exposes geometric structure that is difficult to access in the higher-dimensional setting considered by \citep{pilanci_neural_2020,mishkin2023optimal}, which then allows us to derive closed-form descriptions of the $\ell_2$-regularized or unregularized global minima (Theorem~\ref{prop:sqrd_loss},~\ref{thm:global_optim_explicit_main}) and to compare their dimension, boundedness, and connectivity directly. 
\begin{equation}\label{eq:g_min_problem}\min_{\theta \in \mathbb{R}^{2m}} L(\theta)\!=\! \frac{1}{2}\left\|\frac{1}{\alpha}\!\sum_{j=1}^{m}(Xu_j)_{+}\!\omega_{j}\!-\!Y\right\|_{2}^{2}+\frac{\beta}{2}\sum_{j=1}^{m}(u_j^{2}+\omega_{j}^{2}). 
\end{equation}
For notations, we introduce $D(S)\!\coloneq\!\operatorname{Diag}(\mathbf{1}[X\geq 0])$ and $D(S^{c})\!\coloneq\!\operatorname{Diag}(\mathbf{1}[X\leq 0])$ where $\mathbf{1}[X\leq0]$ is an indicator vector with $(\mathbf{1}[X\leq0])_i\!=\!1[x_i\le0]$. Also, $X_S\!\coloneq\!\operatorname{Diag}(\mathbf{1}[X\geq 0])X$ and $X_{S^{c}}\!\coloneq\!\operatorname{Diag}(\mathbf{1}[X\leq0])X$. To prevent trivial solutions, we assume that $X$ has at least one positive element and one negative element (so $\|X_S\|_2\neq0$ and $\|X_{S^c}\|_2\neq0$), that the input training data $X$ and the output training data $Y$ are independent of $m$, and that $0<\min\{|X_S^{T}Y|,~|X_{S^C}^{T}Y|\}$.
\begin{theorem}\label{prop:sqrd_loss}
The set of globally optimal parameters $\varphi^*(m)\coloneq\{\theta\in\mathbb{R}^{2m}~:~\theta\!=\!\operatorname{argmin}_{\theta\in\mathbb{R}^{2m}}L(\theta)$ with $\beta=0\}$ for unregularized squared loss is 
\begin{equation}
\varphi^*(m) \!=\!
\Bigl\{(u_j,\omega_j)_{j=1}^{m}\ \Bigm|\ 
\sum_{j:\,u_j\ge0} u_j\omega_j
= \frac{\alpha X_S^{\top}Y}{\lVert X_S\rVert_2^2}, 
\sum_{j:\,u_j\le0} u_j\omega_j
= \frac{\alpha X_{S^{c}}^{\top}Y}{\lVert X_{S^{c}}\rVert_2^2}
\Bigr\}.
\end{equation}
\end{theorem}
Adding $\ell_2$-regularization ($\beta>0$) restricts the solution set because we need $|u_i|=|\omega_i|~\forall i\in[m]$ for all globally optimal parameters $(u_i,\omega_i)_{i=1}^{m}$. (See the proof of Theorem~\ref{thm:global_optim_explicit_main} in Appendix~\ref{app:proof2.2}.) 
The visualization of global minima in Figure~\ref{fig:main_varphi_3} (details are in Appendix~\ref{app:visualizations}) illustrates that adding $\ell_2$-regularization significantly restricts the set of globally optimal parameters.
\begin{theorem}
  \label{thm:global_optim_explicit_main}
  The set of globally optimal parameters $\Theta^*(m)\coloneq\{\theta\in\mathbb{R}^{2m}~:~\theta\!=\!\operatorname{argmin}_{\theta\in\mathbb{R}^{2m}}L(\theta)$ with $\beta>0\}$ for $\ell_2$-regularized squared loss is 
  \begin{equation}
\Theta^{*}(m)
=\Bigl\{(u_i,\omega_i)_{i=1}^{m}\;\Bigm|\;
\sum_{i:\,u_i\ge0} u_i^{2}=|\gamma_P^*|,
\sum_{i:\,u_i\le0} u_i^{2}=|\gamma_N^*|,\;
\omega_i=
\begin{cases}
\operatorname{sign}(\gamma_P^*)\,u_i & (u_i>0),\\
\operatorname{sign}(\gamma_N^*)\,u_i & (u_i<0),\\
0 & (u_i=0)
\end{cases}
\Bigr\}
\end{equation}
where $\gamma_P^*=\alpha\frac{\mathcal{S}_{\alpha\beta}(X_{S}^TY)}{\|X_{S}\|_2^2}$ and $\gamma_N^*=\alpha\frac{\mathcal{S}_{\alpha\beta}(X_{S^c}^TY)}{\|X_{S^c}\|_2^2}$. ($\mathcal{S}_{\alpha\beta}(b):=\textnormal{sign}(b)\max(|b|-\alpha\beta,0)=\textnormal{sign}(b)\max(|b|-m^{a-\delta},0)$ is the soft-thresholding operator.)
\end{theorem}
The ReLU output is decomposed into the sum of different contributions based on activation patterns, and the original minimization problem \eqref{eq:general_prob} can be written as a convex problem \citep{pilanci_neural_2020, ergen2021global}. The proofs of Theorem \ref{prop:sqrd_loss}, \ref{thm:global_optim_explicit_main} (Appendix \ref{app:proof2.2}) utilize the fact that we can remove the constraints for the convex problem introduced in \citep{pilanci_neural_2020, ergen2021global} if the data is one-dimensional (Proposition \ref{prop:convex_main}). This simplification is not feasible for general $d$-dimensional input because the activation patterns are more complicated.
\begin{figure}[t]
\centering
\begin{minipage}[b]{0.24\textwidth}
    \centering
    \includegraphics[width=\linewidth]{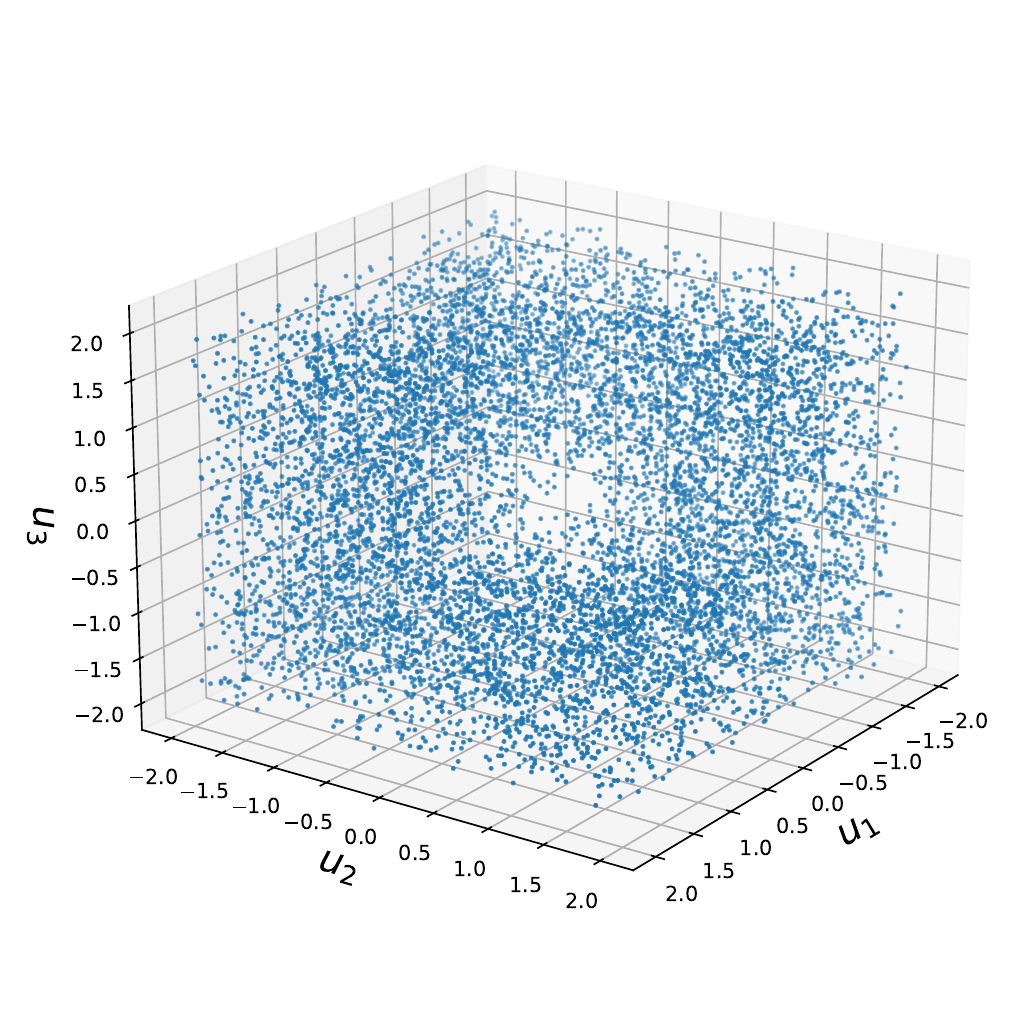}

    \small (a) Projection of $\varphi^*(m)$ in $u$-coordinates
\end{minipage}\hfill
\begin{minipage}[b]{0.24\textwidth}
    \centering
    \includegraphics[width=\linewidth]{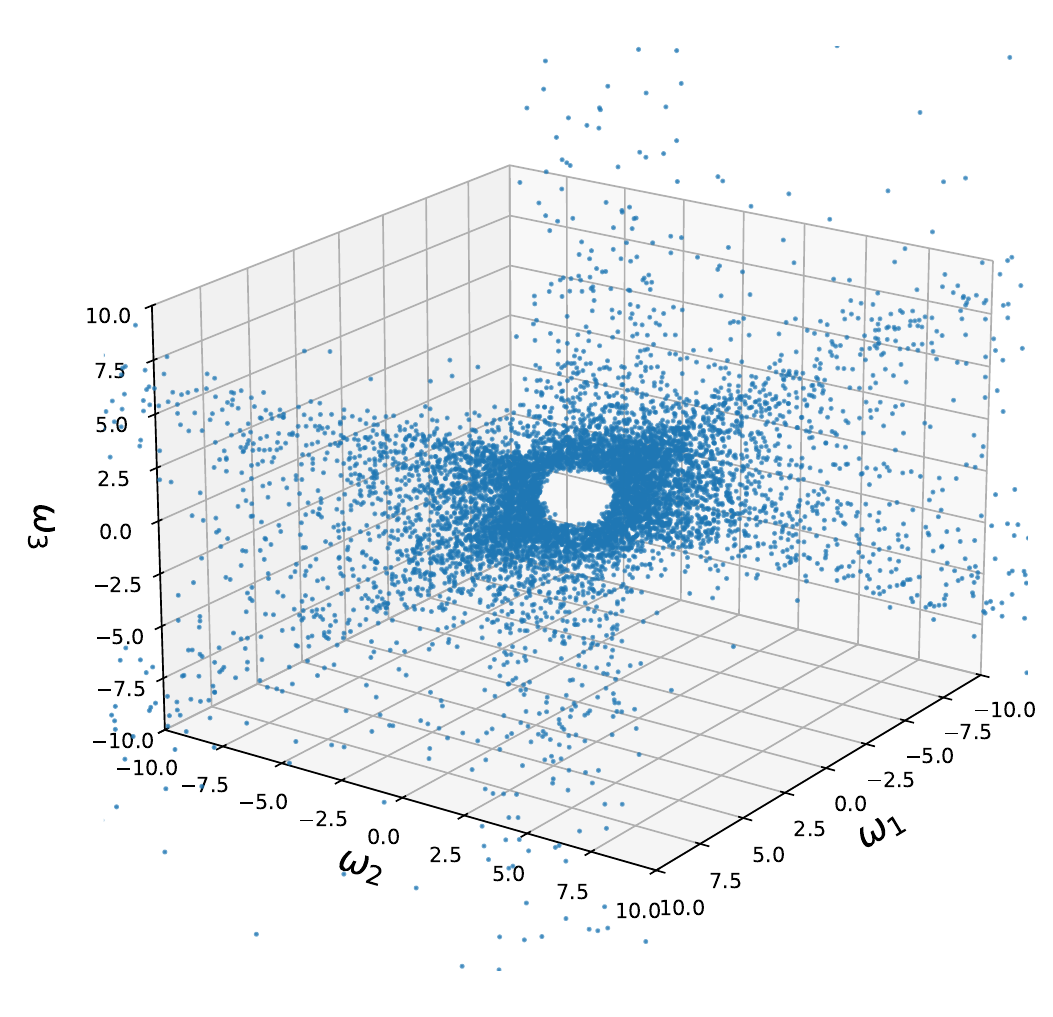}

    \small (b) Projection of $\varphi^*(3)$ in $\omega$-coordinates
\end{minipage}\hfill
\begin{minipage}[b]{0.24\textwidth}
    \centering
    \includegraphics[width=\linewidth]{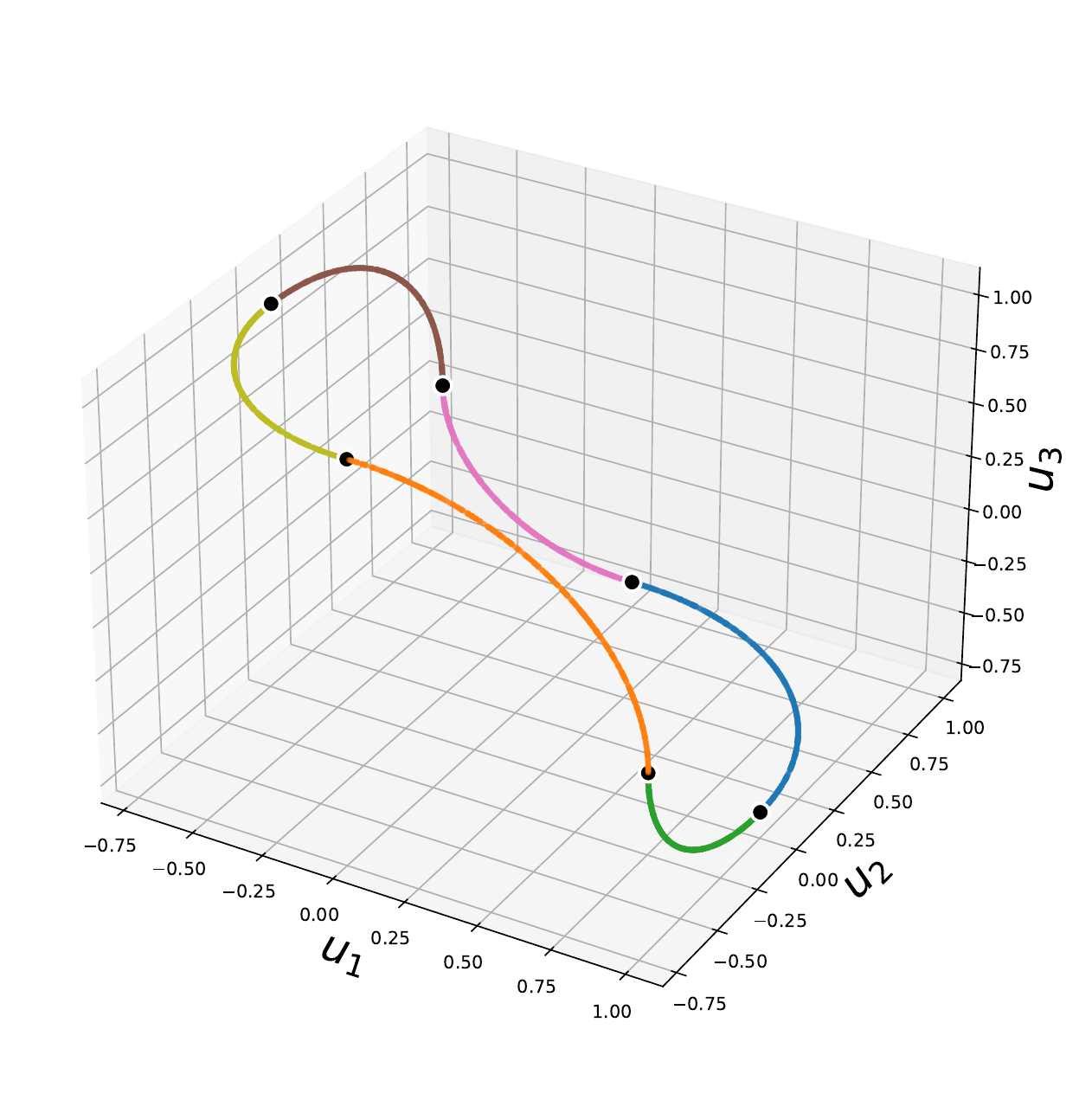}

    \small (c) Projection of $\Theta^*(3)$ in $u$-coordinates
\end{minipage}\hfill
\begin{minipage}[b]{0.24\textwidth}
    \centering
    \includegraphics[width=\linewidth]{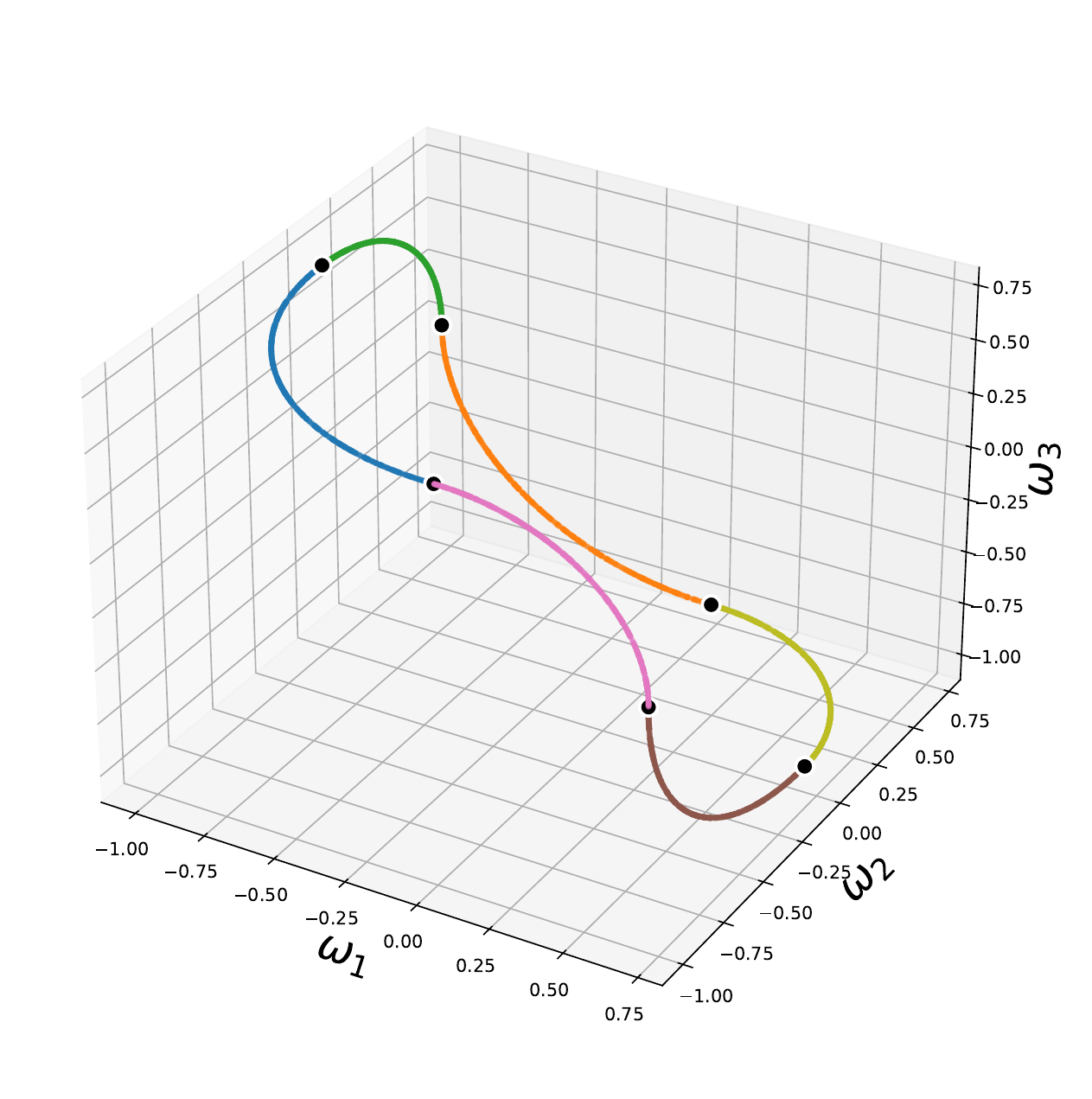}

    \small (d) Projection of $\Theta^*(3)$ in $\omega$-coordinates
\end{minipage}

\caption{Illustration of global optima of width-3 two-layer ReLU neural netwwork \eqref{eq:g_2ReLu} in the parameter space.
More details of this figure are in Appendix~\ref{app:visualizations}.}
\label{fig:main_varphi_3}
\end{figure}
\subsubsection{Connectivity}\label{subsubsec:conn}
In this section, we discuss how the connectivity of globally optimal solutions $\varphi^*(m)$ and $\Theta^*(m)$ changes with respect to $m$. For a set $S$, we say $x, y \in S$ are \textbf{connected} in $S$ if $\exists$ a continuous function $f : [0,1]\rightarrow S$ that satisfies $f (0) = x$ and $f (1) = y$. We say $S$ is \textbf{connected} if, for any two points $x, y \in S$, $x$ and $y$ are connected in $S$. Theorem~\ref{thm:noweight_connectivity_main} shows phase transitional behavior of the connectivity of the global minima for the unregularized squared loss.
\begin{theorem} 
\label{thm:noweight_connectivity_main}
We have the following connectivity results for the solution set $\varphi^*(\theta)$ to the unregularized squared loss.\\
(1) For $m=1$, $\varphi^*(m)=\emptyset$.\\
(2) For $m=2$, $\varphi^*(m)$ has exactly 2 connected components.\\
(3) For $m\ge 3$, $\varphi^*(m)$ is connected.
\end{theorem}
If we add the $\ell_2$-regularization term, we find critical widths $\underline{M}$ and $\overline{M}$ for phase
transition behaviors, which depend on the hyperparameter setting.
\begin{theorem}\label{thm:hyp_connectivity_main}
Assume $\delta>a$. Define $M^*(m)=1[|X_{S}^TY|>m^{a-\delta}]+1[|X_{S^c}^TY|>m^{a-\delta}]$, $\underline{M}=\min\{m\in\mathbb{N}_{\ge 1}~|~M^*(m)\geq1\}$ and $\overline{M}=\min\{m\in\mathbb{N}_{\ge 1}~|~M^*(m)=2\}$.
$\underline{M}$ and $\overline{M}$ are well-defined because $M^*(m)\in\{0,1,2\}$ is increasing with $m$ and $M^*(m)=2$ for sufficiently large $m$.\\
We have the following connectivity results for the solution set $\Theta^*(\theta)$ to the $\ell_2$-reguralized loss.\\
(1) For $m<\underline{M}$, $\Theta^*(m)$ is a singleton ($\{(0,0)_{i=1}^{m}\}$).\\
(2) If $\underline{M}=\overline{M}= m=1$, $\Theta^*(m)=\emptyset$.\\
(3) If $\underline{M}\le m=1<\overline{M}$ or $\underline{M}\le\overline{M}\le m=2$, $\Theta^*(m)$ is a finite set.\\
(4) Otherwise, $\Theta^*(m)$ is connected.
\end{theorem}
Theorem~\ref{thm:hyp_connectivity_main} implies that there are only three connectivity behaviors for $\Theta^*(m)\neq\emptyset$. As \citet[Theorem 2]{kim_exploring_2025} showed, this is not always the case for general $d$-dimensional input data. We provide explanations for this limited connectivity result for the one-dimensional input data from a perspective on the roles of unnecessary neurons (Appendix~\ref{app:neurons}) and a convex formulation of neural networks introduced by \citet{pilanci_neural_2020} (Appendix~\ref{app:convex}).
\subsubsection{Dimension}\label{subsubsec:dim}
In this section, we compare the dimensionality and bounds for the set of optimal parameters. For a subset $A\subset\mathbb{R}^{2m}$, we define the \textbf{dimension} of $A$ 
denoted by $\dim(A)$ to be the maximum $k$ such that $A$ contains a $k$-dimensional
embedded $C^1$ submanifold (equivalently, the maximal stratum dimension). 
For unregularized squared loss, the globally optimal parameters are dense and spread out without a bound outside the sphere of radius $\|\theta^*\|_2$, where $\theta^*\!\in\!\Theta^*(\theta)$, in the parameter space $\mathbb{R}^{2m}$. The dimensional difference between $\varphi^*(m)$ and the parameter space $\mathbb{R}^{2m}$ is two, which is independent of $m$ as long as $m$ is sufficiently large. 
\begin{proposition}\label{prop:sqrd_dim_main}For sufficiently large $m$, $\dim(\varphi^*(m))= 2m-2$. 
\end{proposition}

\begin{proposition}\label{prop:sqrd_unbound_main}
For sufficiently large $m$,~$\varphi^*(m)$ is unbounded. Especially, $\forall n\geq\frac{a}{2}$, $\exists\varphi_n^*(m)\subset\varphi^*(m)~s.t.~\forall\phi_n^*\in\varphi_n^*(m),~\|\phi^*_n\|_2=\Theta\!(m^n)$ and $\dim(\varphi_n^*(m))= 2m-2$.
\end{proposition}
We find that adding the $\ell_2$-regularization term reduces the dimension of the set of optimal parameters by $m$ and imposes a bound on the set. There is always ($m+2$)-dimensional difference between $\Theta^*(m)$ and the parameter space $\mathbb{R}^{2m}$. The dimensional difference grows with order $m$ by increasing the width $m$.
\begin{proposition}\label{prop:global_reg_dim_main}
Under Assumption $\delta>a$, for sufficiently large $m$,~$\dim(\Theta^*(m))=m-2$.
\end{proposition}

\begin{proposition}\label{thm:global_opt_location_main}
Under Assumption $\delta>a$, for sufficiently large $m$, $\Theta^*(m)$ is bounded and\\ $\forall\,\theta^*\in\Theta^*(m)$, $\|\theta^*\|_2= \sqrt{2\alpha\left(
\frac{|\mathcal{S}_{\alpha\beta}(X_{S}^{\top}Y)|}{\|X_{S}\|_2^{2}}
+\frac{|\mathcal{S}_{\alpha\beta}(X_{S^{c}}^{\top}Y)|}{\|X_{S^{c}}\|_2^{2}}
\right)}= \Theta\!\big(m^{\frac{a}{2}}\big)$.
\end{proposition}
\section{Numerical Experiments}\label{sec:ne}
\begin{figure}[t]
\centering
\begin{minipage}[t]{0.49\textwidth}
    \centering
    \includegraphics[width=\linewidth]{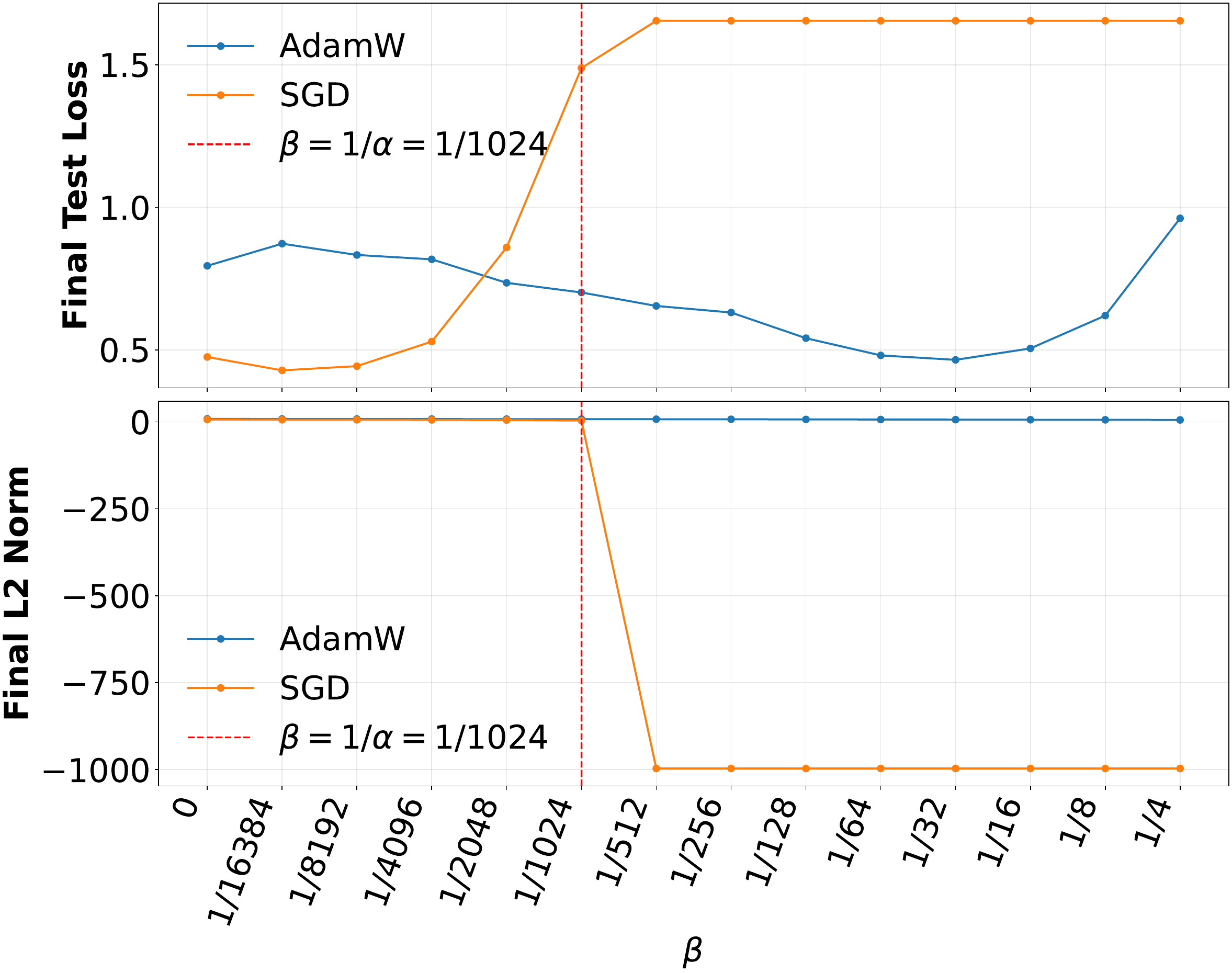}
\end{minipage}\hfill
\begin{minipage}[t]{0.51\textwidth}
    \centering
    \includegraphics[width=\linewidth]{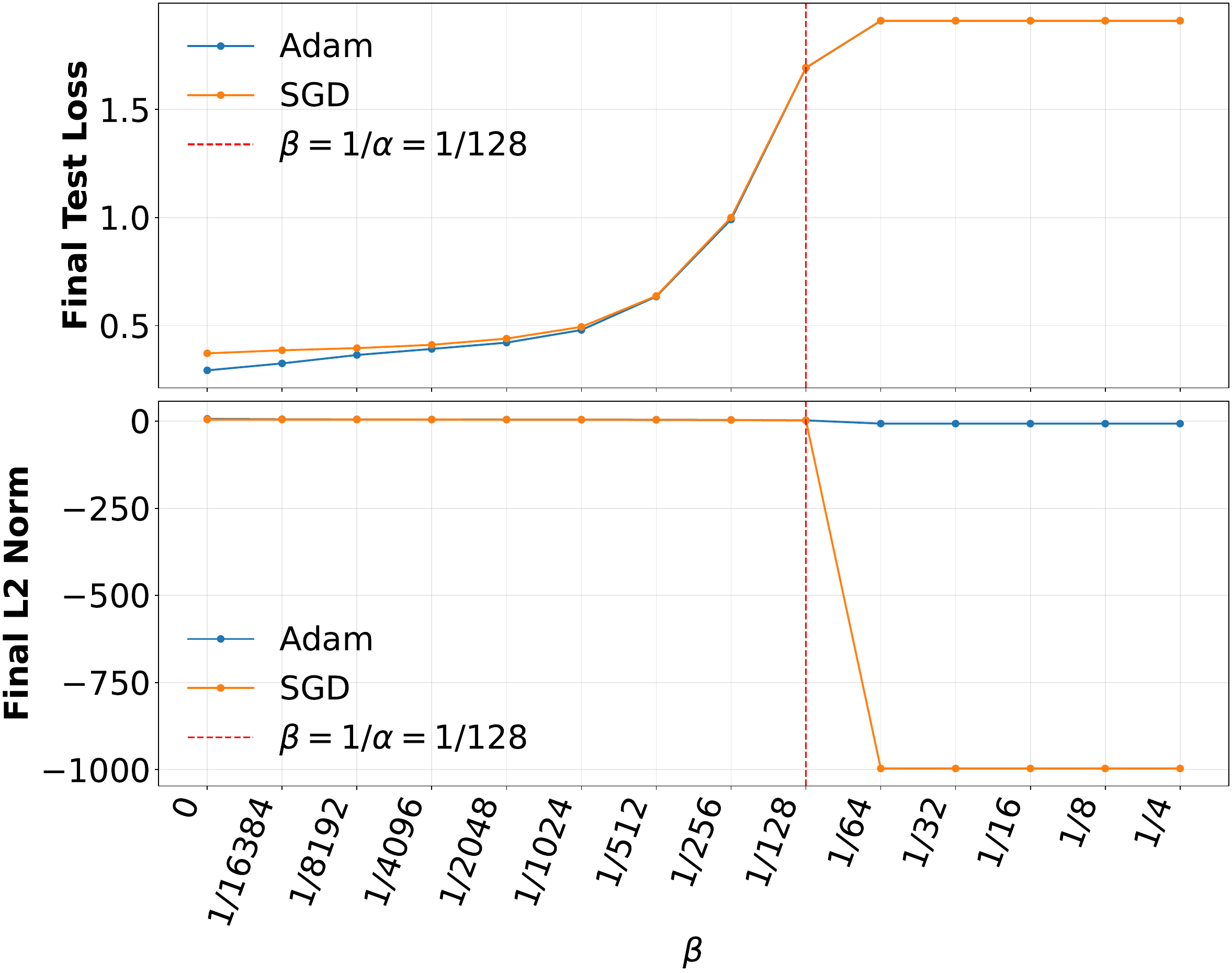}
\end{minipage}
\caption{Final Test loss (above) and Final $\ell_2$-norm of parameters (bottom, $y$-axis values are taken in logarithmic base 2) after training two-layer ReLU networks initialized by $\tau_1$=$\tau_2$=0.01 ($b_1$=$b_2$=0) and trained on Yacht Hydrodynamics \cite{yacht_hydrodynamics_243}. $x$ axis is taken over \textbf{weight decay coefficient} $\beta$ in logarithmic scale, network has width 1024 and  scaled by $1/\alpha=1/1024$ and trained with SGD vs AdamW for $10^5$ epochs (left). $\beta$ is \textbf{$\ell_2$-regularization coefficient} as in \eqref{eq:general_prob}, network has width 128 and scaled by $1/\alpha=1/128$ and trained with SGD vs Adam for $10^6$ epochs (right) The final test loss does not include $\ell_2$-regularization term. }
\label{fig:main_collapse}
\end{figure}
Theorem~\ref{thm:general_convex_zero_solution} itself does not inform us whether the learned parameter collapses to zero. For behaviors of learned parameters, several works discuss convergence to global minima using Adam \cite{kingma2014adam} or stochastic gradient descent (SGD) (details are in Appendix~\ref{app:related}).\\
\\
We trained two-layer ReLU neural networks \eqref{eq:g_2ReLu} with different widths ($m$= 64, 128, 256, 512, 1024, 2048) using the Yacht Hydrodynamics data \cite{yacht_hydrodynamics_243} (squared loss), and MNIST \cite{MNIST} (cross entropy loss). We trained them with SGD or AdamW \cite{loshchilov2018decoupled} using $\beta$ as weight decay coefficient, with a fixed learning rate ($\eta$=0.01). (More details are in Appendix~\ref{app:ne}.) Our numerical experiments show that when $\delta>a$ (i.e. $\beta>\frac{1}{\alpha}$), the learned parameters by SGD collapse to zero as expected by Theorem~\ref{thm:general_convex_zero_solution}, while using AdamW \cite{loshchilov2018decoupled} prevents the learned parameters from collapsing to zero (e.g. Figure~\ref{fig:main_collapse} left). The same phenomenon is seen across different hyperparameter settings or widths. The results for other hyperparameter settings, different widths, and for MNIST are shown in Appendix~\ref{app:ne}.\\
\\
We then trained two-layer ReLU neural networks \eqref{eq:g_2ReLu} using the Yacht Hydrodynamics data \cite{yacht_hydrodynamics_243} (squared loss). We trained them with SGD or Adam \cite{kingma2014adam} using $\beta$ as the coefficient for $\ell_2$-regularization as in \eqref{eq:general_prob}, with a fixed learning rate ($\eta=0.01$). Our numerical experiments show that when $\delta>a$ (i.e. $\beta>\frac{1}{\alpha}$), the learned parameters by SGD collapse to zero, which is the global minima shown in Theorem~\ref{thm:general_convex_zero_solution}, while using Adam \cite{kingma2014adam} prevents the learned parameters from collapsing to zero (Figure~\ref{fig:main_collapse} right).
\section{Conclusion}
We analyzed the set of globally optimal parameters for two-layer ReLU networks under width-dependent hyperparameters and $\ell_2$-regularization, and derived a sufficient condition for the collapse to the zero-weight solution, showing that width, output scaling, and weight decay can qualitatively affect the regularized loss landscape. For one-dimensional inputs, we explicitly characterized the globally optimal parameter sets, showing that $\ell_2$-regularization has a width-invariant effect on connectivity while its dimensionality-reducing effect strengthens with width. Experiments show that SGD follows the predicted collapse, whereas AdamW or Adam prevent it; this may be related to decoupled weight decay, which is not equivalent to optimizing the explicit $\ell_2$-regularized objective for adaptive methods \cite{loshchilov2018decoupled}, and to AdamW's implicit bias toward $\ell_\infty$-constrained optimization \cite{xie2024implicit}. Explaining this optimizer-dependent behavior theoretically remains as future work.

\bibliography{ICML26}
\newpage
\clearpage
\appendix
\section*{Appendix}
\addcontentsline{toc}{section}{Appendix}

\startcontents[appendix]
\printcontents[appendix]{}{1}{\section*{Appendix Contents}}

\section{Related Work}\label{app:related}
\textbf{Geometrical Analysis of Global minima.}\\
There are many works analyzing the geometrical properties of the loss landscape \citep{achour2024loss, wu2025loss, silva2025hide}. Analyzing global minima has also been an active research topic \cite{Cooper, zhao_understanding_2025,kuditipudi_explaining_2019, simsek_geometry_2021}. For example, \citet{simsek_geometry_2021} explicitly describes the manifold of global minima. 

\citet{Cooper} analyzes the dimension of the manifold in global optima for overparameterized neural networks. \citet{zhao_understanding_2025} and \citet{kuditipudi_explaining_2019} study global optima connectivity.
The global optima for two-layer ReLU neural networks trained with $\ell_2$-regularized loss are analyzed by \citet{kim_exploring_2025}. While they use the convex formulation of neural networks introduced by \citet{pilanci_neural_2020} nicely to avoid the difficulty of analyzing non-convex loss, they did not provide an analytic solution for the globally optimal parameter set nor take into account the effects of scaling hyperparameters with respect to width. Inspired by their work, we explicitly consider the effects of width-dependent hyperparameter setting.\\
\\
\textbf{Convergence to Global minima.}\\
Although the learned parameters do not necessarily converge to a global minimum, there are several works confirming the convergence to global minima \citep{li2017convergence, li2020learning, sonodaRidgeRegressionOverparametrized2021, akiyama21a, ChizatB18, sirignano2020mean}. \citet{akiyama21a} shows that, under a specific teacher-student setting, an overparameterized two-layer ReLU student trained with sparse/path-norm regularization, which is closely related to $\ell_2$-regularization, and norm-dependent gradient descent can recover the teacher parameters with high probability. On the other hand, \citet{j.2018on} shows that the Adam \citep{kingma2014adam} algorithm fails to converge to a global optimum. We experimentally demonstrated differences in convergence results depending on the choice of optimizers in a setting distinct from previous work.\\
\\
\textbf{Hyperparameter setting and Training Dynamics.}\\
Different training dynamics behaviors are observed to be dependent on a choice of hyperparmeters. In one regime (linear regime, lazy training, kernel regime), the training can be approximated by kernel gradient descent \citep{chizat_lazy_2020, eilers2024generalized}. In the other regime, nerual networks learn features beyond their initialization \citep{frei2023random, dandi2024two}. A well-known example in this regime is mean-field regime \citep{mei2018mean,mei2019mean}, and the training admits feature learning \citep{yang_feature_2022} or learns adaptively from samples \citep{williams_gradient_2019}. These training dynamics regimes relate to the performance of neural networks. For example, the poor generalization performance of lazy training has been reported. \citep{chizat_lazy_2020, buffelli2024exact}. The hyperparameter settings for our numerical experiments includes both for the kernel regime and for the mean-field regime.

\section{Proofs for main}

\subsection{Proofs for Subsection~\ref{subsec:collapse}}
\begin{theorem}\label{thm:app_1}(Theorem \ref{thm:general_convex_zero_solution} in main)
We consider minimizing the following loss function
\[\min_{\theta\in\Theta}L_{\ell}(\theta)\!=\! \ell(\theta)+\frac{\beta}{2}\!\sum_{j=1}^{m}(u_j^{2}+\omega_{j}^{2}).\]
where $\ell$ is finite on $\Theta$ and can be written as a convex function of $f_{\theta}(X)$. For example, $\ell(\theta)=\|f_{\theta}(X)-Y\|^2$. \\ If $\delta<a$, then there exists $m_0\in\mathbb{N}$ such that for all $m\ge m_0$,
$\argmin_{\theta\in\Theta} L_{\ell}(\theta)=\{(0,0)_{i=1}^{m}\}$.
\end{theorem}
\begin{proof}
Theorem~\ref{thm:app_1} is proven by the following lemmas. Lemma~\ref{lem:l1} proves that Theorem~\ref{thm:app_1} holds when $\partial \tilde{\ell}(f_0(X))\neq\emptyset$. Lemma~\ref{lem:l2} proves that $\partial \tilde{\ell}(f_0(X))\neq\emptyset$ holds when $\ell$ is finite on $\Theta$ and can be written as a convex function of $f_{\theta}(X)$.
\end{proof}
We now prove Lemma~\ref{lem:l1} and Lemma~\ref{lem:l2}. 
\begin{lemma}\label{lem:l1}
Assume that $\partial \tilde{\ell}(f_0(X))\neq\emptyset$. If $\delta<a$, then there exists $m_0\in\mathbb{N}$ such that for all $m\ge m_0$,
\[
\argmin_{\theta\in\Theta} L_{\ell}(\theta)=\{(0,0)_{i=1}^{m}\}.
\]
\end{lemma}

\begin{proof}
$\ell(\theta)=\tilde{\ell}(f_{\theta}(X))$ for some convex function $\tilde{\ell}$.
Fix an element $g_0\in\partial \tilde{\ell}(f_{0}(X))$. \\
By the subgradient inequality, $\forall z\in\mathbb{R}^n,~\tilde{\ell}(z)\ge \tilde{\ell}(f_0(X))+\langle g_0,z\rangle.$
By applying this inequality to $z=f_{\theta}(X)$, we obtain
\[
L_{\ell}(\theta)
\ge
\tilde{\ell}(f_0(X))
+
\langle g_0,f_{\theta}(X)\rangle
+
\frac{\beta}{2}\sum_{j=1}^{m}\bigl(\|u_j\|_2^2+\omega_j^2\bigr).
\]
Now we are going to lower bound $\langle g_0,f_{\theta}(X)\rangle$.\\ For each $j$, by the operator norm inequality for a matrix $X$, $\|(Xu_j)_+\|_2
\le
\|Xu_j\|_2
\le
\|X\|_{\mathrm{op}}\|u_j\|_2.$
Hence, 
\[
\|f_{\theta}(X)\|_2
\le
\frac{1}{\alpha}\sum_{j=1}^{m}|\omega_j|\,\|(Xu_j)_+\|_2
\le
\frac{\|X\|_{\mathrm{op}}}{\alpha}\sum_{j=1}^{m}|\omega_j|\,\|u_j\|_2.
\]
Using the basic inequality $2ab\le a^2+b^2$, we further obtain$
\sum_{j=1}^{m}|\omega_j|\,\|u_j\|_2
\le
\frac12\sum_{j=1}^{m}\bigl(\|u_j\|_2^2+\omega_j^2\bigr).$\\
Therefore,
\[
\|f_{\theta}(X)\|_2
\le
\frac{\|X\|_{\mathrm{op}}}{2\alpha}
\sum_{j=1}^{m}\bigl(\|u_j\|_2^2+\omega_j^2\bigr).
\]
It follows that
\[
\langle g_0,f_{\theta}(X)\rangle
\ge
-
\|g_0\|_2\,\|f_{\theta}(X)\|_2
\ge
-
\frac{\|g_0\|_2\|X\|_{\mathrm{op}}}{2\alpha}
\sum_{j=1}^{m}\bigl(\|u_j\|_2^2+\omega_j^2\bigr).
\]
By substituting this into the lower bound for $L_{\ell}(\theta)$, we obtain
\[
L_{\ell}(\theta)
\ge
\tilde{\ell}(f_0(X))
+
\frac12
\left(
\beta-\frac{\|g_0\|_2\|X\|_{\mathrm{op}}}{\alpha}
\right)
\sum_{j=1}^{m}\bigl(\|u_j\|_2^2+\omega_j^2\bigr).
\]
In our hyperparameter setting, $\alpha\beta=m^{a-\delta}$, so 
the assumption $\delta<a$ implies $
\alpha\beta \to \infty\text{ as }m\to\infty$.
Therefore, there exists $m_0\in\mathbb{N}$ such that for all $m\ge m_0$, $
\beta>\frac{\|g_0\|_2\|X\|_{\mathrm{op}}}{\alpha}.$
For such $m$, we have
\[
L_{\ell}(\theta)
\ge
\tilde{\ell}(f_0(X))
+
c_m
\sum_{j=1}^{m}\bigl(\|u_j\|_2^2+\omega_j^2\bigr) \text{ for some positive constant $c_m>0$. }
\]
On the other hand,
\[
L_{\ell}((0,0)_{i=1}^{m})=\tilde{\ell}(f_0(X)).
\]
Thus, for every nonzero $\theta$,
\[
L_{\ell}(\theta)>L_{\ell}((0,0)_{i=1}^{m}).
\]
Hence, $(0,0)_{i=1}^{m}$ is the unique global minimizer of $L_{\ell}$ for all sufficiently large $m$.
\end{proof}

\begin{lemma}\label{lem:l2}
If $\tilde{\ell}$ is finite and convex on all of $\mathbb{R}^n$, the assumption $\partial\tilde{\ell}(f_0(X))\neq\emptyset$ holds. If $\tilde{\ell}$ is differentiable at $0$, one may take $g_0=\nabla \tilde{\ell}(f_0(X))$ in the proof of Lemma~\ref{lem:l1}.
\end{lemma}

\begin{proof}
Recall that the subdifferential of $\tilde{\ell}$ at 0 is
$\partial \tilde{\ell}(f_0(X))
=
\left\{
g\in\mathbb{R}^n:\;
\tilde{\ell}(z)\ge \tilde{\ell}(f_0(X))+\langle g,z\rangle
\quad\forall z\in\mathbb{R}^n
\right\}.$\\
If $\tilde{\ell}$ is differentiable at $0$, 
$\partial\tilde{\ell}(f_0(X))=\{\nabla\tilde{\ell}(f_0(X))\}$.\\
\\
Now we assume more generally that $\tilde{\ell}:\mathbb{R}^n\to\mathbb{R}$ is finite and convex. Consider the epigraph
\[
\operatorname{epi}(\tilde{\ell})
=
\{(z,r)\in\mathbb{R}^n\times\mathbb{R}: r\ge \tilde{\ell}(z)\}.
\]
As $\tilde{\ell}$ is convex and finite everywhere, $\tilde{\ell}$ is locally Lipschitz at every point, and so is continuous on all of $\mathbb{R}^n$. Take any sequence $(z_k,r_k)\in\operatorname{epi}(\tilde{\ell})$ such that $(z_k,r_k)\rightarrow (z,r)\in\mathbb{R}^n\times\mathbb{R}$. Then, $r_k\geq\tilde{\ell}(z_k)$ for all $k$, and by continuity of $\tilde{\ell}$, $r\geq\tilde{\ell}(z)$. Hence, $(z,r)\in\operatorname{epi}(\tilde{\ell})$, and so $\operatorname{epi}(\tilde{\ell})$ is closed.\\
\\
As $\tilde{\ell}$ is convex, $\operatorname{epi}(\tilde{\ell})$ is a convex set. The point $(0,\tilde{\ell}(f_0(X)))$ lies on the boundary of $\operatorname{epi}(\tilde{\ell})$. By the supporting hyperplane theorem, there exist non-zero $(a,b)\neq (0,0)\in\mathbb{R}^n\times\mathbb{R}$ such that
\[
\forall (z,r)\in \operatorname{epi}(\tilde{\ell}),~\langle a,z\rangle + br \ge 
\langle a,0\rangle + b\tilde{\ell}(f_0(X)).
\]
We prove $b>0$ by contradiction. If $b<0$, then for any fixed $z$, letting $r\to\infty$ contradicts the inequality above. If $b=0$, $\langle a,z\rangle \ge 0$ for all $z\in\mathbb{R}^n$.
Applying this also to $-z$ yields $\langle a,z\rangle=0$ for all $z$. This leads to $a=0$, a contradiction. Therefore, $b>0$.\\
\\
For every $z\in\mathbb{R}^n$, the point $(z,\tilde{\ell}(z))$ belongs to $\operatorname{epi}(\tilde{\ell})$, so $\langle a,z\rangle + b\tilde{\ell}(z)\ge b\tilde{\ell}(f_0(X))$.
By rearranging the equations,
\[
\forall z\in\mathbb{R}^n,~\tilde{\ell}(z)\ge \tilde{\ell}(f_0(X))+\left\langle -\frac{a}{b},\,z\right\rangle.
\]
Hence, $-\frac{a}{b}\in \partial \tilde{\ell}(f_0(X))$ and therefore $\partial \tilde{\ell}(f_0(X))\neq\emptyset$.
\end{proof}

\subsection{Proofs for Subsection~\ref{subsec:go}}\label{app:proof2.2}
\begin{theorem}\label{prop:sqrd_loss_app} (Theorem~\ref{prop:sqrd_loss} in main)
The set of global optimal parameters 
    \[
    \varphi^*(m)=\{\theta\in\mathbb{R}^{2m}~:~\arg\min_{\theta\in\mathbb{R}^{2m}}\frac{1}{2}\left\|
f_{\theta}(x) - Y
\right\|_{2}^{2}\}\]
for squared loss is 
\begin{equation*}
\begin{aligned}
\varphi^*(m) \coloneq
\Bigl\{(u_j,\omega_j)_{j=1}^{m}\ \Bigm|\ 
& \sum_{j:\,u_j\ge0} u_j\omega_j
= \frac{\alpha X_S^{\top}Y}{\lVert X_S\rVert_2^2}, 
\sum_{j:\,u_j\le0} u_j\omega_j
= \frac{\alpha X_{S^{c}}^{\top}Y}{\lVert X_{S^{c}}\rVert_2^2}
\Bigr\}.
\end{aligned}
\end{equation*}
\end{theorem}
\begin{proof}
    By writing
$\gamma_{P}=\sum_{i:u_i\ge0}u_i\omega_i~\text{and~}
\gamma_{N}=\sum_{i:u_i\le0}u_i\omega_i,$
\[
\left\|
f_{\theta}(x) - Y
\right\|_{2}^{2}
=\Big\|\frac{1}{\alpha}\sum_{j=1}^{m}(Xu_j)_{+}\omega_j - Y\Big\|_2^2 
=\|\,\gamma_P \frac{X_S}{\alpha} + \gamma_N \frac{X_{S^c}}{\alpha} - Y\,\|_2^2.
\]
By solving the KKT condition, (LHS) is minimized at $(\gamma_P,~\gamma_N)=(\gamma_P^*,\gamma_N^*)$ where
\begin{align*}
    &0=\gamma_P^*\|\frac{X_S}{\alpha}\|_2^2-\frac{X_S}{\alpha}^{T}Y,~0=\gamma_N^*\|\frac{X_{S^c}}{\alpha}\|_2^2-\frac{X_{S^c}}{\alpha}^{T}Y\\
    \iff~&\gamma_P^*=\frac{\alpha X_S^TY}{\|X_S\|_2^2},~\gamma_N^*=\frac{\alpha X_{S^c}Y}{\|X_{S^c}\|_2^2}.
\end{align*}
Hence, 
\[
\Big\|\frac{1}{\alpha}\sum_{j=1}^{m}(Xu_j)_{+}\omega_j - Y\Big\|_2^2 
\geq\|\,\gamma_P^* \frac{X_S}{\alpha} + \gamma_N^* \frac{X_{S^c}}{\alpha} - Y\,\|_2^2.
\]
with equality iff 
$\sum_{j:\,u_j\ge0} u_j\omega_j
= \frac{\alpha X_S^{\top}Y}{\lVert X_S\rVert_2^2} \text{ and }
\sum_{j:\,u_j\le0} u_j\omega_j
= \frac{\alpha X_{S^{c}}^{\top}Y}{\lVert X_{S^{c}}\rVert_2^2}$.
\end{proof}
For the one-dimensional input case, we prove that the minimum value for the original non-convex problem (\ref{eq:g_min_problem}) with $\beta>0$ equals the minimum value for the convex problem (\ref{eq:gamma_min}) in Lemma~\ref{lem:gamma_min}.
\begin{lemma}\label{lem:gamma_min}
    Consider the optimization problem 
    \begin{align}\label{eq:gamma_min}
\min_{\gamma_P,\gamma_N\in\mathbb{R}}
&~\frac{1}{2}\|\,\gamma_P \frac{X_S}{\alpha} + \gamma_N \frac{X_{S^c}}{\alpha} - Y\,\|_2^2
+ \beta(|\gamma_P| + |\gamma_N|).
\end{align}
    The optimal argument ($\gamma_P^*$, $\gamma_N^*$) is uniquely determined as
    \begin{align}\label{eq:opt_gamma}
        (\gamma_P^*, \gamma_N^*)=\left(\alpha\frac{\mathcal{S}_{\alpha\beta}(X_{S}^TY)}{\|X_{S}\|_2^2}, \alpha\frac{\mathcal{S}_{\alpha\beta}(X_{S^c}^TY)}{\|X_{S^c}\|_2^2}\right).
    \end{align}
    where $\mathcal{S}_\beta(b):=sign(b)\max(|b|-\beta,0)$ is the softmax function.\\
\end{lemma}
\begin{proof}
    Note that the minimization problem (\ref{eq:gamma_min}) is a convex minimization problem.
By KKT condition, 
\begin{align*}
    0\in \{\gamma_P\|\frac{X_S}{\alpha}\|_2^2-\frac{X_S}{\alpha}^{T}Y+\beta S~|~S\in\partial|\gamma_P|\}
\end{align*}
at $\gamma_P=\gamma_P^*$. $\partial|\gamma_P|$ is the subdifferential of the absolute value function $|\gamma_P|$ at $\gamma_P$ found as
\[
\partial\,|\gamma_P|
=
\begin{cases}
\{1\}, & \gamma_P>0,\\[4pt]
[-1,\,1], & \gamma_P=0,\\[4pt]
\{-1\}, & \gamma_P<0.
\end{cases}
\]
Substituting $\partial|\gamma_P|$ gives the solution form as follows
\[
\gamma_P^*=\frac{sign(\frac{X_S}{\alpha}^TY)\,\max(|\frac{X_S}{\alpha}^TY|-\beta,0)}{\|\frac{X_S}{\alpha}\|_2^2}
=\frac{\mathcal{S}_\beta(\frac{X_S}{\alpha}^TY)}{\|\frac{X_S}{\alpha}\|_2^2}
=\alpha\frac{\mathcal{S}_{\alpha\beta}(X_{S}^TY)}{\|X_{S}\|_2^2}.
\]
$\gamma_N^*=\alpha\frac{\mathcal{S}_{\alpha\beta}(X_{S^c}^TY)}{\|X_{S^c}\|_2^2}$
is found by replacing $X_S$ by $X_{S^c}$.
\end{proof}

We derive the exact solution to the global minima for regularized squared loss by using Lemma~\ref{lem:gamma_min}
\begin{theorem}
  \label{thm:global_optim_explicit}
  (Theorem~\ref{thm:global_optim_explicit_main} in the main). The global optimum $\Theta^*(m)=\{\theta\in\mathbb{R}^{2m}~:~\arg\min_{\theta\in\mathbb{R}^{2m}}L(\theta)\}$ with $\beta>0$ is 
  \begin{equation}\label{eq:explicit}
\Theta^{*}(m)
=\Bigl\{(u_i,\omega_i)_{i=1}^{m}\;\Bigm|\;
\sum_{i:\,u_i\ge0} u_i^{2}=|\gamma_P^*|,\;
\sum_{i:\,u_i\le0} u_i^{2}=|\gamma_N^*|,\;
\omega_i=
\begin{cases}
\operatorname{sign}(\gamma_P^*)\,u_i & (u_i>0),\\
\operatorname{sign}(\gamma_N^*)\,u_i & (u_i<0),\\
0 & (u_i=0)
\end{cases}
\Bigr\}
\end{equation}
where $\gamma_P^*=\alpha\frac{\mathcal{S}_{\alpha\beta}(X_{S}^TY)}{\|X_{S}\|_2^2}$ and $\gamma_N^*=\alpha\frac{\mathcal{S}_{\alpha\beta}(X_{S^c}^TY)}{\|X_{S^c}\|_2^2}$. 
\end{theorem}
\begin{proof}
    By writing
\[
\gamma_{P}=\sum_{i:u_i\ge0}u_i\omega_i, \quad 
\gamma_{N}=\sum_{i:u_i\le0}u_i\omega_i, \quad
X_S=\operatorname{Diag}(\mathbf{1}[X\geq 0])X, \quad \text{and} \quad X_{S^{c}}=\operatorname{Diag}(\mathbf{1}[X\le 0])X,
\]
we can write $\Big\|\frac{1}{\alpha}\sum_{j=1}^{m}(Xu_j)_{+}\omega_j - Y\Big\|_2^2$ as $\|\,\gamma_P \frac{X_S}{\alpha} + \gamma_N \frac{X_{S^c}}{\alpha} - Y\,\|_2^2$.
Also,
\begin{align}
&\frac{\beta}{2}\sum_{j=1}^{m}(u_j^2 + \omega_j^2)\notag\\
&\geq
\frac{\beta}{2}\sum_{j:u_j>0}(u_j^2 + \omega_j^2)
+ \frac{\beta}{2}\sum_{j:u_j<0}(u_j^2 + \omega_j^2)\quad(\text{equality holds iff $u_j=0\Rightarrow\omega_j=0$})\label{ineq:zero}\\
&\geq
\beta\sum_{j:u_j\ge0}|u_j||\omega_j|
+ \beta\sum_{j:u_j\le0}|u_j||\omega_j|
=
\beta\sum_{j:u_j\ge0}|u_j\omega_j|
+ \beta\sum_{j:u_j\le0}|u_j\omega_j|\quad(\text{equality holds iff $\forall~j,~|u_j|=|\omega_j|$})\label{ineq:abs}\\
&\geq
\beta\left|\sum_{j:u_j\ge0}u_j\omega_j\right|
+ \beta\left|\sum_{j:u_j\le0}u_j\omega_j\right|
=
\beta |\gamma_P|
+ \beta |\gamma_N|\label{ineq:sgn}\\
&(\text{equality holds iff $\forall j,i\text{ s.t. the product }u_ju_i>0,~\mathrm{sign}(u_j\omega_j)=\mathrm{sign}(u_i\omega_i)$}).\notag
\end{align}

Hence,
\begin{align}
&\min_{\{u_i,\omega_i\}_{i=1}^{m}\in(\mathbb{R}\times\mathbb{R})^m} 
~\frac{1}{2}\Big\|\frac{1}{\alpha}\sum_{j=1}^{m}(Xu_j)_{+}\omega_j - Y\Big\|_2^2 
+ \frac{\beta}{2}\sum_{j=1}^{m}(u_j^2 + \omega_j^2)\notag\\
&\geq~\min_{\gamma_P,\gamma_N\in\mathbb{R}}
~\frac{1}{2}\|\,\gamma_P \frac{X_S}{\alpha} + \gamma_N \frac{X_{S^c}}{\alpha} - Y\,\|_2^2
+ \beta(|\gamma_P| + |\gamma_N|)\label{ineq:gamma}\\
&=\frac{1}{2}\|\,\gamma_P^* \frac{X_S}{\alpha} + \gamma_N^* \frac{X_{S^c}}{\alpha} - Y\,\|_2^2
+ \beta(|\gamma_P^*| + |\gamma_N^*|).\notag
\end{align}
The equality in the last holds by Lemma~\ref{lem:gamma_min} where ($\gamma_P^*$, $\gamma_N^*$) is defined as (\ref{eq:opt_gamma}).\\
By \eqref{ineq:zero}, \eqref{ineq:abs} and \eqref{ineq:sgn}, 
\begin{align*}
    &\frac{1}{2}\Big\|\frac{1}{\alpha}\sum_{j=1}^{m}(Xu_j)_{+}\omega_j - Y\Big\|_2^2 
+ \frac{\beta}{2}\sum_{j=1}^{m}(u_j^2 + \omega_j^2)
=
\frac{1}{2}\Big\|\gamma_P^* \frac{X_S}{\alpha} + \gamma_N^* \frac{X_{S^c}}{\alpha} - Y\,\Big\|_2^2
+ \beta |\gamma_P^*|
+ \beta |\gamma_N^*|
\end{align*}
holds iff 
$(u_i, \omega_i)_{i=1}^{m}=(u_i^*, \omega_i^*)_{i=1}^{m}$ such that 
\begin{align}
    &|\gamma_P^*|=\left|\sum_{i:u_i^*\ge0}u_i^*\omega_i^*\right|=\left|\sum_{i:u_i^*\ge0}{u_i^*}^2\right|
    =\sum_{i:u_i^*\ge0}{u_i^*}^2,\label{eq:cond1}\\
    &|\gamma_N^*|=\left|\sum_{i:u_i^*\le0}u_i^*\omega_i^*\right|=\sum_{i:u_i^*\le0}{u_i^*}^2,\label{eq:cond2}\\
&\omega_i^*=
\begin{cases}
\text{sign}(\gamma_P^*)u_i^* & (u_i^*>0),\\
\text{sign}(\gamma_N^*)u_i^* & (u_i^*<0),\\
0 & (u_i^*=0).
\end{cases}\label{eq:cond3}
\end{align}
\eqref{ineq:abs} and \eqref{ineq:sgn} imply that $\sum_{i:u_i^*\ge0}u_i^*\omega_i^*=\sum_{i:u_i^*\ge0}{u_i^*}^2$ or $-\sum_{i:u_i^*\ge0}{u_i^*}^2$. This implies \eqref{eq:cond1}. Similarly, \eqref{ineq:abs} and \eqref{ineq:sgn} imply \eqref{eq:cond2}.\\
To satisfy $\gamma_P^*=\sum_{i:u_i^*\ge0}u_i^*\omega_i^*$ and $\gamma_N^*=\sum_{i:u_i^*\le0}u_i^*\omega_i^*$ under \eqref{ineq:abs} and \eqref{ineq:sgn}, we need \eqref{eq:cond3}.\\
Hence, $(u_i^*, \omega_i^*)_{i=1}^{m}\in\Theta^*(m)$ if and only if it satisfies (\ref{eq:cond1}) (\ref{eq:cond2}) (\ref{eq:cond3}).
\end{proof}
\subsection{Proofs for Subsection~\ref{subsubsec:conn}}\label{app:conn_proof}
\begin{theorem} (Theorem~\ref{thm:noweight_connectivity_main} in main)
\label{thm:noweight_connectivity}
We have the phase transitional behavior of the solution set for the unregularized squared loss.\\
(1) For $m=1$, $\varphi^*(m)=\emptyset$.\\
(1) For $m=2$, $\varphi^*(m)$ has exactly 2 connected components.\\
(2) For $m\ge 3$, $\varphi^*(m)$ is connected.
\end{theorem}
\begin{proof}
(1) $m=1$:\\
Since $0<\min\{|X_S^{T}Y|,~|X_{S^C}^{T}Y|\}$, we need at least two non-zero neurons to satisfy $\sum_{j:\,u_j\ge0} u_j\omega_j
= \frac{\alpha X_S^{\top}Y}{\lVert X_S\rVert_2^2}$ and $\sum_{j:\,u_j\le0} u_j\omega_j
= \frac{\alpha X_{S^{c}}^{\top}Y}{\lVert X_{S^{c}}\rVert_2^2}$.\\
\\
(2) $m=2$:\\
The sign pattern for $u_1$ and $u_2$ is $(u_1,u_2)=(+,-)$ or $(-,+)$. For $(u_1,u_2)=(+,-)$, $(\omega_1,~\omega_2)=\left(\frac{\alpha X_S^{\top}Y}{u_1\lVert X_S\rVert_2^2},~\frac{\alpha X_{S^{c}}^{\top}Y}{u_2\lVert X_{S^{c}}\rVert_2^2}\right)$. For $(u_1,u_2)=(-,+)$, $(\omega_1,~\omega_2)=\left(\frac{\alpha X_{S^{c}}^{\top}Y}{u_1\lVert X_{S^{c}}\rVert_2^2},~\frac{\alpha X_S^{\top}Y}{u_2\lVert X_S\rVert_2^2}\right)$.\\
Hence, $\varphi^*(m)$ consists of two connected components $\varphi^*_1(m)$ and $\varphi^*_2(m)$ such that
\begin{align*}
    \varphi^*_1(m)=\Big\{(u_i,\omega_i)_{i=1}^2\Big|u_1>0,~ u_2<0,~\omega_1=\frac{\alpha X_S^{\top}Y}{u_1\lVert X_S\rVert_2^2},~\omega_2=\frac{\alpha X_{S^{c}}^{\top}Y}{u_2\lVert X_{S^{c}}\rVert_2^2}\Big\},\\
    \varphi^*_2(m)=\Big\{(u_i,\omega_i)_{i=1}^2\Big|u_2>0,~ u_1<0,~\omega_2=\frac{\alpha X_S^{\top}Y}{u_2\lVert X_S\rVert_2^2},~\omega_1=\frac{\alpha X_{S^{c}}^{\top}Y}{u_1\lVert X_{S^{c}}\rVert_2^2}\Big\}.
\end{align*}
$\varphi^*_1(m)$ and $\varphi^*_2(m)$ are disjoint because $u_1\neq0$ and $u_2\neq0$.\\
\\
(3) $m\ge3$:\\
For simplicity, write $A\coloneq\frac{\alpha X_S^{\top}Y}{\lVert X_S\rVert_2^2}$ and $B\coloneq\frac{\alpha X_{S^{c}}^{\top}Y}{\lVert X_{S^{c}}\rVert_2^2}$.\\
Fix any $\theta=(u,\omega)\in \varphi^*(m)$, where $u=(u_1,\dots,u_m)$ and $\omega=(\omega_1,\dots,\omega_m)$.
Define the strict sign index sets 
\[
P(u)\coloneq \{j\in[m]: u_j>0\},\qquad N(u)\coloneq \{j\in[m]: u_j<0\}.
\]
We must have $P(u)\neq\emptyset$ and $N(u)\neq\emptyset$.
Choose any indices $p\in P(u)$ and $n\in N(u)$.\\
\\
Keep $u$ fixed and define $\omega(t)$ for $t\in[0,1]$ by
\[
\omega_j(t)\coloneq (1-t)\omega_j \quad \text{for all } j\notin\{p,n\},
\]
and 
\[
\omega_p(t)\coloneq \omega_p + \frac{t}{u_p}\sum_{\substack{j\in P(u)\\ j\neq p}} u_j\omega_j,\qquad
\omega_n(t)\coloneq \omega_n + \frac{t}{u_n}\sum_{\substack{j\in N(u)\\ j\neq n}} u_j\omega_j.
\]
Then, for all $t\in[0,1]$,
\[
\sum_{j:\,u_j\ge 0}u_j\omega_j(t)
=
u_p\omega_p(t)+\sum_{\substack{j\in P(u)\\ j\neq p}}u_j(1-t)\omega_j
=
\sum_{j\in P(u)}u_j\omega_j
=
A,
\]
and likewise
\[
\sum_{j:\,u_j\le 0}u_j\omega_j(t)
=
u_n\omega_n(t)+\sum_{\substack{j\in N(u)\\ j\neq n}}u_j(1-t)\omega_j
=
\sum_{j\in N(u)}u_j\omega_j
=
B.
\]
Hence $\theta(t)\coloneq (u,\omega(t))\in\varphi^*(m)$ for all $t$.
At $t=1$, we have $\omega_j(1)=0$ for all $j\notin\{p,n\}$, so only the two indices $p$ and $n$ carry the constraints $u_p\omega_p(1)=A,~u_n\omega_n(1)=B$  and all other products are $u_j\omega_j(1)=0$.\\
\\
After this deformation, for every $j\notin\{p,n\}$, keep $\omega_j$ fixed and define $u_j(t)$ for $t\in[0,1]$ by 
\[u_j(t)\coloneq u_j(1-t).\]
For other indices, define as follows
\begin{align*}
&u_p(t)\coloneq\
\left\{
\begin{array}{ll}
    u_p~~~\text{if}~~~u_p=1,\\
    \frac{u_p}{1+(u_p-1)t}~~~\text{if}~~~u_p\neq1,
\end{array}
\right.
~
&\omega_p\coloneq\frac{A}{u_p(t)},\\
&u_n(t)\coloneq\
\left\{
\begin{array}{ll}
    u_n~~~\text{if}~~~u_n=-1,\\
    \frac{u_n}{1+(-u_n-1)t}~~~\text{if}~~~u_n\neq-1,
\end{array}
\right.
~
&\omega_n\coloneq\frac{B}{u_n(t)}.
\end{align*}
Then, $u_p(t)>0$ and $u_n(t)<0$ for all $t\in[0,1]$. $u_j(t)\omega_j=0$ for all $t\in[0,1]$. $u_p(t)\omega_p(t)$ and $u_n(t)\omega_n(t)$ are unchanged over $t\in[0,1]$. Hence, both constraints remain unchanged.\\
Thus, $\theta\in\varphi^*(m)$ is connected to a point in $\varphi^*_{min}(m)$ where 
\[
\varphi^*_{min}(m)=\Big\{(u_i,\omega_i)_{i=1}^{m}\Big|\exists~p,n\in[m]~s.t.~(u_j,\omega_j)=(0,0)\text{ for all } j\notin\{p,n\},~u_p=1,~\omega_p=A,~u_n=-1,~\omega_n=-B\Big\}.
\]
To prove the connectivity of $\varphi^*(m)$, it is enough to show that all elements in $\varphi^*_{min}(m)$ are connected in $\varphi^*(m)$.\\
All elements in $\varphi^*_{min}(m)$ are permutations of each other. Denote by $S_m$ the symmetric group on $\{1,2,~...~,m\}$. $S_m$ is the set of all permutations on $\{1,2,~...~,m\}$. It is enough to show that 
        \[
\forall(u_i,\omega_i)_{i=1}^{m}\in \varphi_{min}^*(m),~\forall\sigma\in S_m,~\exists \text{a continuous path in $\varphi^*(m)$ connecting }(u_i,\omega_i)_{i=1}^{m} \text{ and }(u_{\sigma(i)},\omega_{\sigma(i)})_{i=1}^{m}.
\]
Pick any $(u_i,~\omega_i)_{i=1}^{m}\in \varphi_{min}^*(m)$. As $m\ge3$ and $(u_i,~\omega_i)_{i=1}^{m}$ has exactly two non-zero elements, $\exists~p,n,j\in[m]$ such that $u_p=1,~u_n=-1,~u_j=0$. Also, for all $j\in[m]$, the set of transpositions $\mathcal{T}_j\coloneq\{(i,j) \mid i \in [m]~s.t.~i \neq j\}$ (We denote by $(i,j)$ a transposition) generates $S_m$, so it is enough to show that
\[
\forall~T\in \mathcal{T}_j~\text{where }u_j=0,~\exists~\text{a continuous path in $\varphi^*(m)$ connecting $(u_i,\omega_i)_{i=1}^{m} \text{ and }(u_{T(i)},\omega_{T(i)})_{i=1}^{m}$}.
\]
Take any $T\in \mathcal{T}_j$.\\
\\
(case 1) If $u_{T(j)}=0$ ($T=(i,j)$ for $i\neq n,p$), then $(u_i,\omega_i)_{i=1}^{m}=(u_{T(i)},\omega_{T(i)})_{i=1}^{m}$.\\
\\
(case 2) Consider the case $u_{T(j)}=u_p$ ($T=(p,j)$).\\
Construct a path $C$ as 
\[
    C(s)=(u_i(s),\omega_i(s))_{i=1}^{m},~s\in[0,1]
\]
s.t. 
\begin{align*}
    &(u_p(s),\omega_p(s))=(\sqrt{1-s},A\sqrt{1-s}),\\
    &(u_j(s),\omega_j(s))=(\sqrt{s},A\sqrt{s}),\\
    &(u_i(s),\omega_i(s))=(u_i,\omega_i)~\forall~i\neq j,p.
\end{align*}
$C(s)$ is well-defined and connected. As $u_j(s),~u_p(s)\ge 0$ and $u_j(s)\omega_j(s)+u_p(s)\omega_p(s)=A~\forall s\in[0,1]$, $C(s)\in\varphi^*(m)$. $C(0)=(u_i,\omega_i)_{i=1}^{m}$ and $C(1)=(u_{T(i)},\omega_{T(i)})_{i=1}^{m}$. Thus, $C(s)$ is a continuous path in $\varphi^*(m)$ connecting $(u_i,\omega_i)_{i=1}^{m}$ and $(u_{T(i)},\omega_{T(i)})_{i=1}^{m}$.\\
\\
(case 3) Consider the case $u_{T(j)}=u_n$ ($T=(n,j)$).\\
Construct a path $C$ as 
\[
    C(s)=(u_i(s),\omega_i(s))_{i=1}^{m},~s\in[0,1]
\]
s.t. 
\begin{align*}
    &(u_n(s),\omega_n(s))=(-\sqrt{1-s},-B\sqrt{1-s}),\\
    &(u_j(s),\omega_j(s))=(-\sqrt{s},-B\sqrt{s}),\\
    &(u_i(s),\omega_i(s))=(u_i,\omega_i)~\forall~i\neq j,n.
\end{align*}
$C(s)$ is well-defined and connected. As $u_j(s),~u_n(s)\le 0$ and $u_j(s)\omega_j(s)+u_n(s)\omega_n(s)=B~\forall s\in[0,1]$, $C(s)\in\varphi^*(m)$. $C(0)=(u_i,\omega_i)_{i=1}^{m}$ and $C(1)=(u_{T(i)},\omega_{T(i)})_{i=1}^{m}$. Thus, $C(s)$ is a continuous path in $\varphi^*(m)$ connecting $(u_i,\omega_i)_{i=1}^{m}$ and $(u_{T(i)},\omega_{T(i)})_{i=1}^{m}$.\\
\\
Thus, the claim is proved.
\end{proof}

We firstly prove the connectivity result of global optimal parameters $\Theta^*(m)$ using notations $\alpha=m^a$ and $\beta=m^{-\delta}$ in two different ways.
\begin{theorem} 
\label{thm:connectivity}
We have a critical width $M^*=1[\mathcal{S}_{\alpha\beta}(X_S^TY)\neq0]+1[\mathcal{S}_{\alpha\beta}(X_{S^c}^TY)\neq0]$ that determines the phase transitional behavior of the solution set.\\
(1) For $M^*=0$, $\Theta^*(m)$ is a singleton ($\{(0,0)_{i=1}^{m}\}$)\\
(2) For $m<M^*$, $\Theta^*(m)=\emptyset$\\
(3) For $m=M^*>0$, $\Theta^*(m)$ is a finite set.\\
(4) For $m> M^*>0$, $\Theta^*(m)$ is connected.
\end{theorem}

 The first proof is simply to use the explicit form of $\Theta^*(m)$. The second proof follows the principal ideas introduced by \citet{kim_exploring_2025}. We provide the first proof in this section, and the second proof is given in the next section (Appendix \ref{app:connectivity_convex}).
 \begin{proof}(The first proof)\\
 \[M^*=1[\mathcal{S}_{\alpha\beta}(X_S^TY)\neq0)+1[\mathcal{S}_{\alpha\beta}(X_{S^c}^TY)\neq0]=1[|\gamma_P^*|\neq0]+1[|\gamma_N^*|\neq0].\]
 (1) $M^*=0$:\\
 $M^*=0$, implies $|\gamma_P^*|=0$ and $|\gamma_N^*|=0$. For $(u_i,\omega_1)_{i=1}^{m}\in\Theta^*(m)$,
   $\sum_{i:\,u_i\ge0} u_i^{2}=0$ and $\sum_{i:\,u_i\le0} u_i^2=0$, so $\forall i\in[m],~u_i=0$. Also, $\forall i\in[m],~|u_i|=|\omega_i|$, so $\forall i\in[m],~\omega_i=0$.\\
   \\
 (2) $m<M^*$:\\
 $\Theta^*(m)=\emptyset$ because we need at least $M^*$ non-zero neurons for $\sum_{i:\,u_i\ge0} u_i^{2}=|\gamma_P^*|$ and $\sum_{i:\,u_i\le0} u_i^{2}=|\gamma_N^*|$ to hold.\\
 \\
     (3) $m=M^*$:\\
     To satisfy (\ref{eq:cond1}) (\ref{eq:cond2}) (\ref{eq:cond3}), 
\begin{align*}\label{eq:fin_set}
    (u_i,\omega_i)_{i=1}^{M^{*}}\in\Theta^*(M^*)\Rightarrow\forall i\in[M^*],~(u_i,~\omega_i)\in\Big\{\text{non-zero elements in }\\
    \{(\sqrt{|\gamma_P^*|}, sign(\gamma_P^*)\sqrt{|\gamma_P^*|}), (-\sqrt{|\gamma_N^*|}, -sign(\gamma_N^*)\sqrt{|\gamma_N^*|})\}\Big\}
\end{align*}
Hence, $|\Theta^*(m)|\leq 2^{M^*}$ and $\Theta^*(m)$ is a finite set.\\
\\
(4) $m>M^*$:\\
Define the $u$-projection
\[
U^*(m):=\left\{u\in\mathbb{R}^m\ \middle|\ \sum_{i:u_i>0}u_i^2=|\gamma_P^*|,\ \sum_{i:u_i<0}u_i^2=|\gamma_N^*|\right\}.
\]
Define $\Omega:\mathbb{R}^m\to\mathbb{R}^m$ coordinatewise by
\[
\Omega_i(u)=
\begin{cases}
\operatorname{sign}(\gamma_P^*)u_i & (u_i>0),\\
\operatorname{sign}(\gamma_N^*)u_i & (u_i<0),\\
0 & (u_i=0).
\end{cases}
\]
$|\operatorname{sign}(\gamma_P^*)|,~|\operatorname{sign}(\gamma_N^*)|\leq1$ implies that $|\Omega_i(a)-\Omega_i(b)|\leq|a-b|$, so $\Omega_i$ is Lipschitz, hence is continuous. As each $\Omega_i$ is continuous, $\Omega$ is continuous. We can write $\Theta^*(m)$ as
\[
\Theta^*(m)=\{(u,\Omega(u))|\ u\in U^*(m)\}.
\]
Therefore, the map $F:U^*(m)\to\Theta^*(m)$ defined as $F(u)=(u,\Omega(u))$ is a homeomorphism with inverse given by the
 projection $(u,\omega)\mapsto u$. Thus, it suffices to prove that $U^*(m)$ is path-connected.\\
Define the minimal optimal subset as
\[
U^*_{min}(m):=\left\{\sqrt {|\gamma_P^*|}\,e_i-\sqrt{|\gamma_N^*|}\,e_j\Big|\ i,j\in\{1,\dots,m\},\ i\neq j\right\}\subset U^*(m),
\]
where $(e_i)_{i=1}^{m}$ are the standard basis vectors.\\
We prove the following two claims.
\begin{enumerate}
    \item Elements in $U^*_{min}(m)$ are connected to each other in $U^*(m)$ when $m>M^*$.
    \begin{proof}
    We prove by a direct construction of a path.
    Denote by $S_m$ the symmetric group on $\{1,2,~...~,m\}$. $S_m$ is the set of all permutations on $\{1,2,~...~,m\}$. Denote $u=(u_1,~...~,u_m)\in\mathbb{R}^m$ by $(u_i)_{i=1}^{m}$.\\
    From the construction of $U^*_{min}(m)$, it is enough to show that 
        \[
\forall(u_i)_{i=1}^{m}\in U_{min}^*(m),~\forall\sigma\in S_m,~\exists \text{a continuous path in $U^*(m)$ connecting }(u_i)_{i=1}^{m} \text{ and }(u_{\sigma(i)})_{i=1}^{m}.
\]
Pick any $(u_i)_{i=1}^{m}\in U_{min}^*(m)$. As $m>M^*=1(|\gamma_P^*|\neq0)+1(|\gamma_N^*|\neq0)$, $\exists j\in[m]$ such that $u_j=0$.
Also, for all $j\in[m]$, the set of transpositions $T_j\coloneq\{(i,j) \mid i \in [m]~s.t.~i \neq j\}$ (We denote by $(i,j)$ a transposition) generates $S_m$, so it is enough to show that
\[
\forall t\in T_j~\text{where }u_j\neq0,~\exists~\text{a continuous path in $U^*(m)$ connecting $(u_i)_{i=1}^{m} \text{ and }(u_{T(i)})_{i=1}^{m}$}.
\]
Take any $t\in T_j$. If $u_{t(j)}=0$, $(u_i)_{i=1}^{m}=(u_{t(i)})_{i=1}^{m}$. If $u_{t(j)}\neq0$, construct a path $C$ as 
\[
    C(s)=(u_i(s))_{i=1}^{m},~s\in[0,1]
\]
s.t. 
\begin{align*}
    u_i(s)=
    \left\{
\begin{array}{ll}
    u_i~~~\text{if}~~~i\neq j\text{ and }i\neq t(j),\\
    u_{t(j)}\sqrt{s}~~~\text{if}~~~i=j,\\
    u_{t(j)}\sqrt{1-s}~~~\text{if}~~~i=t(j).
\end{array}
\right.
\end{align*}
By direct calculation, $C(s)$ is a continuous path in $U^*(m)$ that connects $(u_i)_{i=1}^{m}$ and $(u_{t(i)})_{i=1}^{m}$. Thus, the claim holds.
\end{proof}
\item Elements in $U^*(m)$ are connected to an element in $U_{min}^*(m)$.
\begin{proof}
We prove by direct construction of a path along which we decrease the number of non-zero elements. Take any $u\in U^*(m)$ and apply the following steps.\\
\\
\textbf{Step 1}\\
Define sets of indices.
\[
P\coloneq\{i\in[m]:u_i>0\},~~N\coloneq\{i\in[m]:u_i<0\}
\]
If $|P|\leq1$ and $|N|\leq1$, $u$ is a point in $U_{min}^*(m)$. If not, move to Step 2.\\
\\
\textbf{Step 2}\\
Since  $|P|>1$ or $|N|>1$, $\exists k<l$ such that $k,l\in P$ or $k,l \in N$. For such $k,l$, construct a merging path $M(s)$ as 
\[
    M(s)=(u_i(s))_{i=1}^{m},~s\in[0,1]
\]
s.t. 
\begin{align*}
    u_i(s)=
    \left\{
\begin{array}{ll}
    u_i~~~\text{if}~~~i\neq k\text{ and }i\neq l,\\
    \sqrt{u_k^2u_l^2}\cos(\theta(1-s))~~~\text{if}~~~i=k,\\
    \sqrt{u_k^2u_l^2}\sin(\theta(1-s))~~~\text{if}~~~i=l.
\end{array}
\right.
\end{align*}
$\theta\in[0,2\pi)$ satisfies $u_k= \sqrt{u_k^2u_l^2}\cos\theta$ and $u_l= \sqrt{u_k^2u_l^2}\sin\theta$. By direct calculation, $M(s)$ is a continuous path in $U^*(m)$ and $M(0)=u$. Apply Step 1 to $M(1)$ and repeat this merging process until we reach a point in $U_{min}^*(m)$. The number of non-zero elements in $M(1)$ is less than that of $M(0)$, so the process terminates in a finite number of steps.
\end{proof}
\end{enumerate}
These two claims prove that $\Theta^*(m)$ is connected for $m>M^*$.
\end{proof}

The second proof is provided in Appendix~\ref{app:connectivity_convex}.

\begin{theorem}\label{thm:hyp_connectivity}
Assume $a<\delta$. Define $M^*(m)=1[|X_{S}^TY|>m^{a-\delta}]+1[|X_{S^c}^TY|>m^{a-\delta}]$. Define
\begin{align*}
&\underline{M}=\min\{m\in\mathbb{N}_{\ge 1}~|~M^*(m)\geq1\},\\
&\overline{M}=\min\{m\in\mathbb{N}_{\ge 1}~|~M^*(m)=2\}.
\end{align*}
$M^*(m)\in\{0,1,2\}$ is increasing with $m$ and $M^*(m)=2$ for sufficiently large $m$. Hence, $\underline{M}$ and $\overline{M}$ are well-defined.\\
We have the following connectivity results.\\
(1) For $m<\underline{M}$, $\Theta^*(m)$ is a singleton ($\{(0,0)_{i=1}^{m}\}$).\\
(2) If $\underline{M}=\overline{M}= m=1$, $\Theta^*(m)=\emptyset$.\\
(3) If $\underline{M}\le m=1<\overline{M}$ or $\underline{M}\le\overline{M}\le m=2$, $\Theta^*(m)$ is a finite set.\\
(4) Otherwise, $\Theta^*(m)$ is connected.
\end{theorem}

\begin{proof}
   $m^{a-\delta}$ is strictly decreasing and $m^{a-\delta}\rightarrow 0$ as $m\rightarrow\infty$. Hence, $M^*(m)\in\{0,1,2\}$ is increasing with $m$ and $M^*(m)=2$ for sufficiently large $m$.\\
   As the product $\alpha\beta$ depends on $m$, $\gamma_P^*=\alpha\frac{\mathcal{S}_{\alpha\beta}(X_{S}^TY)}{\|X_{S}\|_2^2}$ and $\gamma_N^*=\alpha\frac{\mathcal{S}_{\alpha\beta}(X_{S^c}^TY)}{\|X_{S^c}\|_2^2}$ depends on $m$.
   \\
   \[
   M^*(m)=1[|X_{S}^TY|>m^{a-\delta}]+1[|X_{S^c}^TY|>m^{a-\delta}]=1[|\gamma_P^*|>0]+1[|\gamma_N^*|>0].
   \]
   Also,
   \begin{align}
       &M^*(m)=0 \iff m<\underline{M},\label{eq:1}\\
       &M^*(m)=1 \iff \underline{M}\le m< \overline{M},\label{eq:2}\\
       &M^*(m)=2 \iff \overline{M}\le m.\label{eq:3}
   \end{align}
   (1) $m<\underline{M}$:\\
   From \eqref{eq:1}, $M^*(m)=0$, so the result from Theorem~\ref{thm:connectivity} (1) implies that $\Theta^*(m)$ is a singleton.\\
   (2) $\underline{M}=\overline{M}=m=1$:\\
   We need at least $M^*(m)$ non-zero neurons for $\sum_{i:\,u_i\ge0} u_i^{2}=|\gamma_P^*|$ and $\sum_{i:\,u_i\le0} u_i^{2}=|\gamma_N^*|$ to hold. Hence, $\Theta^*(m)=\emptyset$ iff $m<M^*(m)$. As $M^*(m)\in\{0,1,2\}$, $m<M^*(m)$ happens iff $m=1$ and $M^(1)=2$. As $1\le\underline{M}\le\overline{M}$ and by \eqref{eq:3}, $M^(1)=2$ holds iff $\underline{M}=\overline{M}=1=m$.\\
   (3) $\underline{M}\le m=1<\overline{M}$ or $\underline{M}\le\overline{M}\le m=2$:\\
   As $M^*(m)\in\{0,1,2\}$, $m=M^*(m)$ iff $1=m=M^*(1)$ or $2=m=M^*(2)$. By \eqref{eq:2}, $1=m=M^*(1)$ iff $\underline{M}\le 1=m< \overline{M}$. By \eqref{eq:3}, $2=m=M^*(2)$ iff $\overline{M}\le m=2$.\\
   \\
   (4) Otherwise:\\
   From the above arguments, $m<M^*(m)$ iff conditions for (2) hold, $m=M^*(m)$ iff conditions for (3) hold, and $M^*(m)=0$ iff conditions for (1) hold. Thus, $m>M^*(m)>0$ for this case. By the result from Theorem~\ref{thm:connectivity} (4), $\Theta^*(m)$ is connected.
\end{proof}
\subsection{Proofs for Subsection~\ref{subsubsec:dim}}
\noindent\textbf{Notation:} For a subset $A\subset\mathbb{R}^d$, we define
$\dim(A)$ to be the maximum $k$ such that $A$ contains a $k$-dimensional
embedded $C^1$ submanifold (equivalently, the maximal stratum dimension).

\begin{lemma}\label{prop:ass_reg}
    Under Assumption $\delta>a$, for sufficiently large $m$, 
    \begin{align*}
    \mathcal{S}_{\alpha\beta}(X_S^{\mathsf T}Y)\neq 0 \wedge \mathcal{S}_{\alpha\beta}(X_{S^c}^{\mathsf T}Y)\neq 0
     \end{align*}
\end{lemma}
\begin{proof}
    Denote $\{\min\{|X_S^{T}Y|,~|X_{S^C}^{T}Y|\}$ by $C$. By our assumption, $C$ is positive and independent of $m$. For $m>1$,
    \begin{align*}
    &\mathcal{S}_{\alpha\beta}(X_S^{\mathsf T}Y)\neq 0 \wedge \mathcal{S}_{\alpha\beta}(X_{S^c}^{\mathsf T}Y)\neq 0\\
        \iff~&\beta<\frac{1}{\alpha}\min\{|X_S^{T}Y|,~|X_{S^C}^{T}Y|\}
        \iff~m^{-\delta}<m^{-(a)}C\\
        \iff~&-d<-a+\log_{m}C
        \iff~a-d<\log_mC.
    \end{align*}
    $\log_{m}C\rightarrow0$ as $m\rightarrow\infty$. Under Assumption $\delta>a$, $a-d<\log_mC$ holds for sufficiently learge $m$. 
\end{proof}

\begin{proposition}\label{prop:global_reg_dim}(Proposition~\ref{prop:global_reg_dim_main} in main)
Under Assumption~$\delta>a$, for sufficiently large $m$,
    \[
    \dim(\Theta^*(m))=m-2.
    \]
\end{proposition}
\begin{proof}
By the explicit form of $\Theta^*(m)$ as in Theorem~\ref{thm:global_optim_explicit}, $\omega_i$ is uniquely determined by $u_i$ for every $i$; hence the degrees of freedom of $\Theta^*(m)$ are entirely captured by the vector $u=(u_1,\dots,u_m)\in\mathbb{R}^m$.
We have two conditions $\sum_{i:\,u_i\ge0} u_i^{2}=|\gamma_P^*|$ and $\sum_{i:\,u_i\le0} u_i^{2}=|\gamma_N^*|$. Since $\gamma_P^*$ (resp.\ $\gamma_N^*$) is a nonzero scalar multiple of $\mathcal{S}_{\alpha\beta}(X_S^{\mathsf T}Y)$ (resp.\ $\mathcal{S}_{\alpha\beta}(X_{S^c}^{\mathsf T}Y)$), we have
\begin{align*}
    \gamma_P^*\neq 0 \wedge \gamma_N^*\neq 0&\iff \mathcal{S}_{\alpha\beta}(X_S^{\mathsf T}Y)\neq 0 \wedge \mathcal{S}_{\alpha\beta}(X_{S^c}^{\mathsf T}Y)\neq 0.
\end{align*}
By Lemma~\ref{prop:ass_reg}, this holds for sufficiently large $m$. Therefore, both of the quadratic equalities impose nontrivial conditions on $u$.\\
\\
To satisfy these equalities, we need both $\{i:u_i>0\}$ and $\{i:u_i<0\}$ to be non-empty. The two constraints are independent because they involve disjoint sets of indices ($\{i:u_i\ge0\}\neq\{i:u_i\le0\}$). Hence, in $\mathbb{R}^m$, the dimension is reduced by exactly one for each condition. 
\end{proof}

\begin{proposition}\label{thm:global_opt_location} (Proposition~\ref{thm:global_opt_location_main} in main)
Under Assumption $\delta>a$, for sufficiently large $m$, $\Theta^*(m)$ is bounded and
\begin{align*}
\forall\,\theta^*\in\Theta^*(m),\quad \|\theta^*\|_2
&= \sqrt{2\alpha\left(
\frac{|\mathcal{S}_{\alpha\beta}(X_{S}^{\top}Y)|}{\|X_{S}\|_2^{2}}
+\frac{|\mathcal{S}_{\alpha\beta}(X_{S^{c}}^{\top}Y)|}{\|X_{S^{c}}\|_2^{2}}
\right)}\\
&= \Theta\!\big(m^{\frac{a}{2}}\big).
\end{align*}
\end{proposition}
\begin{proof}
    By the explicit form of $\Theta^*(m)$ in Theorem~\ref{thm:global_optim_explicit_main}, $\forall~\theta^*=(u_i,\omega_i)_{i=1}^{m}\in\Theta^*(m)$,
    \begin{align*}
        \|\theta^*\|_2^2&=2\left(\sum_{i:\,u_i\ge0} u_i^{2}+\sum_{i:\,u_i\le0} u_i^{2}\right)\\
        &=2(|\gamma_P^*|+|\gamma_N^*|)\\
        &=2\alpha\left(
\frac{|\mathcal{S}_{\alpha\beta}(X_{S}^{\top}Y)|}{\|X_{S}\|_2^{2}}
+\frac{|\mathcal{S}_{\alpha\beta}(X_{S^{c}}^{\top}Y)|}{\|X_{S^{c}}\|_2^{2}}
\right)
    \end{align*}
 $\|\theta^*\|_2^2=\Theta\!(m^{a})\Rightarrow\|\theta^*\|_2=\Theta\!(m^{\frac{a}{2}}).$
\end{proof}
For the rest of the proofs, keep in mind that we made an assumption that $0<\min\{|X_S^{T}Y|,~|X_{S^C}^{T}Y|\}$.
\begin{lemma}\label{prop:out_of_sphere}
    Under Assumption $\delta>a$, for sufficiently large $m$,
    \[
    \forall\phi^*\in\varphi^*(m),~\forall \theta^*\in\Theta^*(m),\quad\|\phi^*\|_2\geq\|\theta^*\|_2.
    \]
\end{lemma}
\begin{proof}
$\forall~\phi^*\in\varphi^*(m)$,~$\forall~\theta^*\in\Theta^*(m)$,
\begin{align*}
    \|\phi^*\|_2^2&=\sum_{i=1}^{m}u_i^2+\sum_{i=1}^{m}\omega_i^2\\
    &\ge 2\sum_{i=1}^{m}|u_i\omega_i|\quad(\text{by AM-GM inequality)}\\
    &= 2\sum_{i:u_i>0}|u_i\omega_i|+2\sum_{i:u_i<0}|u_i\omega_i|\\
    &\ge 2\left|\sum_{i:u_i>0}u_i\omega_i\right|+2\left|\sum_{i:u_i<0}u_i\omega_i\right|\\
    &=2\left|\frac{\alpha X_S^{\top}Y}{\|X_S\|_2^{2}}\right|+2\left|\frac{\alpha X_{S^c}^{\top}Y}{\|X_{S^c}\|_2^{2}}\right|\quad(\text{by Theorem~\ref{prop:sqrd_loss_app})}\\
    &=2\alpha\left(\frac{|X_S^{\top}Y|}{\|X_S\|_2^{2}}+\frac{|X_{S^c}^{\top}Y|}{\|X_S^c\|_2^{2}}\right)\\
    &\ge 2\alpha\left(
\frac{|\mathcal{S}_{\alpha\beta}(X_{S}^{\top}Y)|}{\|X_{S}\|_2^{2}}
+\frac{|\mathcal{S}_{\alpha\beta}(X_{S^{c}}^{\top}Y)|}{\|X_{S^{c}}\|_2^{2}}
\right)\\
&=\|\theta^*\|_2^2\quad(\text{by Theorem~\ref{thm:global_opt_location})}
\end{align*}
\end{proof}

\begin{proposition}\label{prop:sqrd_dim}(Proposition~\ref{prop:sqrd_dim_main} in main) For sufficiently large $m$,
    \[\mathrm{dim}(\varphi^*(m))= 2m-2.\] 
\end{proposition}

\begin{proof}
For notations, write
\[
c_1\coloneq\frac{\alpha X_S^{\top}Y}{\lVert X_S\rVert_2^2}
\quad\text{and}\quad
c_2\coloneq\frac{\alpha X_{S^{c}}^{\top}Y}{\lVert X_{S^{c}}\rVert_2^2}.
\]
We assume $m\ge2$. As $0<\min\{|X_S^{\top}Y|,~|X_{S^c}^{\top}Y|\}$, we have $c_1\neq 0$ and $c_2\neq 0$.
Hence any $(u_j,\omega_j)_{j=1}^m\in\varphi^*(m)$ must have at least one strictly positive $u_j$
and at least one strictly negative $u_j$.

Fix a sign pattern by choosing a partition $P,N\subset\{1,\dots,m\}$ with
\[
P\neq\varnothing,\quad N\neq\varnothing,\quad P\cup N=\{1,\dots,m\},\quad P\cap N=\varnothing,
\]
and consider the open orthant
\[
\mathcal{U}_{P,N}\coloneq \Bigl\{(u,\omega)\in\mathbb{R}^{2m}:\ u_j>0\ (j\in P),\ u_j<0\ (j\in N)\Bigr\}.
\]
On $\mathcal{U}_{P,N}$ we have $\{j:u_j\ge 0\}=P$ and $\{j:u_j\le 0\}=N$, so the defining
conditions of $\varphi^*(m)$ become the two smooth equations
\begin{align*}
f_1(u,\omega) &\coloneq \sum_{j\in P} u_j\omega_j - c_1 = 0,\\
f_2(u,\omega) &\coloneq \sum_{j\in N} u_j\omega_j - c_2 = 0.
\end{align*}
Let $F\coloneq(f_1,f_2):\mathbb{R}^{2m}\to\mathbb{R}^2$. Then
\[
\varphi^*(m)\cap \mathcal U_{P,N}
=\bigl\{(u,\omega)\in \mathcal U_{P,N}:\ F(u,\omega)=0\bigr\}
=F^{-1}(0)\cap \mathcal U_{P,N}.
\]

Take any $(u,\omega)\in\mathcal{U}_{P,N}$. Choose $j\in P$ and $k\in N$ (possible since both are nonempty).
Then
\[
\frac{\partial f_1}{\partial \omega_j}(u,\omega)=u_j\neq 0,\qquad
\frac{\partial f_2}{\partial \omega_j}(u,\omega)=0,
\]
and
\[
\frac{\partial f_2}{\partial \omega_k}(u,\omega)=u_k\neq 0,\qquad
\frac{\partial f_1}{\partial \omega_k}(u,\omega)=0.
\]
Hence the $2\times 2$ submatrix of $DF(u,\omega)$ formed by the columns corresponding to
$\omega_j$ and $\omega_k$ is
\[
\begin{pmatrix}
u_j & 0\\
0 & u_k
\end{pmatrix},
\]
which is invertible. Therefore $\mathrm{rank}\,DF(u,\omega)=2$ at every $(u,\omega)\in\mathcal{U}_{P,N}$.
In particular, $0\in\mathbb{R}^2$ is a regular value of $F$ on $\mathcal{U}_{P,N}$.
By the preimage theorem (a version of the implicit function theorem),
$F^{-1}(0)\cap\mathcal{U}_{P,N}$ is a submanifold of $\mathcal{U}_{P,N}$ with codimension $2$.
Hence,
\[
\mathrm{dim}\,\bigl(F^{-1}(0)\cap\mathcal{U}_{P,N}\bigr)=2m-2.
\]

The set $\varphi^*(m)$ is the union over all such admissible sign patterns $(P,N)$
of the pieces $\varphi^*(m)\cap\mathcal{U}_{P,N}$, together with boundary parts where some $u_j=0$.
Each interior piece has dimension $2m-2$, while boundary parts (where at least one additional
equality $u_j=0$ holds) have dimension at most $2m-3$. Therefore, the (maximal) manifold
dimension of $\varphi^*(m)$ is
\[
\dim(\varphi^*(m)) = 2m-2 \quad \text{for } m\ge 2.
\]
\end{proof}

\begin{proposition}\label{prop:sqrd_unbound}(Proposition~\ref{prop:sqrd_unbound_main} in main)
For sufficiently large $m$,~$\varphi^*(m)$ is unbounded. 
Especially, $\forall a\geq\frac{a}{2}$, $\exists\varphi_a^*(m)\subset\varphi^*(m)~s.t.~\forall\phi_a^*\in\varphi_a^*(m),~\|\phi^*_a\|_2=\Theta\!(m^a)$ and $\dim(\varphi_a^*(m))= 2m-2$.
\end{proposition}
\begin{proof}
    For simplicity, define the nonzero constants $b_1 \coloneqq \frac{X_S^\top Y}{\|X_S\|_2^2}\neq 0,~ 
b_2 \coloneqq \frac{X_{S^c}^\top Y}{\|X_{S^c}\|_2^2}\neq 0$
and write $c_1(m)\coloneqq \alpha b_1,~ c_2(m)\coloneqq \alpha b_2$ where $\alpha=m^{a}.$
Assume $m>2$. Then,$\forall t>0,~\phi(t)=(u_i(t),\omega_i(t))_{i=1}^{m}\in\varphi^*(m)$
where $\phi(t)$ is defined as 
\[
u_1=t,~\omega_1=\frac{c_1(m)}{t},~u_2=-t,~\omega_2=-\frac{c_2(m)}{t},~u_i=0,~\omega_i=0\quad(i\in\{3,~...~,m\}).
\]
Its squared norm is $\|\phi(t)\|_2^2=2t^2+\frac{c_1(m)^2+c_2(m)^2}{t^2}\rightarrow\infty\quad\text{as }t\rightarrow\infty.$
Hence, $\varphi^*(m)$ is unbounded for $m>2$.\\
\\
As in the proof for Proposition~\ref{prop:sqrd_dim}, define a sign pattern partition $P,N\subset\{1,\dots,m\}$ with
\[
P\neq\varnothing,\quad N\neq\varnothing,\quad P\cup N=\{1,\dots,m\},\quad P\cap N=\varnothing,
\]
and consider the open orthant
\[
\mathcal{U}_{P,N}\coloneq \Bigl\{(u,\omega)\in\mathbb{R}^{2m}:\ u_j>0\ (j\in P),\ u_j<0\ (j\in N)\Bigr\}.
\]
We set $P=\{1\}$ and $N=\{2,~...~,m\}$ and denote $\mathcal{U}_{P,N}$ by $\mathcal{U}$. From the proof for Proposition~\ref{prop:sqrd_dim}, $\dim\bigl(\varphi^*(m)\cap\mathcal{U}\bigr)=2m-2.$\\
Pick any $a\ge \frac{a}{2}$. Introduce notations $\beta \coloneqq \min\{|b_1|,|b_2|\}>0$ and $\kappa\coloneqq \sqrt{\frac{\beta}{2}}.$\\
Define $\varphi_a^*(m)$ to be the subset of $\varphi^*(m)\cap\mathcal{U}$ consisting of points satisfying
\[
u_1\in(m^a,2m^a),\qquad u_2\in(-2m^a,-m^a),
\]
and for $j=3,\dots,m$,
\[
u_j\in\Bigl(-\kappa\, m^{\frac{a-1}{2}},\ -\frac{\kappa}{2}\, m^{\frac{a-1}{2}}\Bigr),
\qquad
\omega_j\in\Bigl(-\kappa\, m^{\frac{a-1}{2}},\ \kappa\, m^{\frac{a-1}{2}}\Bigr),
\]
with $(\omega_1,\omega_2)$ determined by the constraints:
\[
\omega_1=\frac{c_1(m)}{u_1},\qquad
\omega_2=\frac{c_2(m)-\sum_{j=3}^m u_j\omega_j}{u_2}.
\]
From the proof for Proposition~\ref{prop:sqrd_dim}, $\varphi^*(m)\cap\mathcal{U}$ is an embedded submanifold of $\mathbb{R}^{2m}$ (by the implicit function theorem), so the manifold topology on $\varphi^*(m)\cap\mathcal{U}$ is the subspace topology inherited from $\mathbb{R}^{2m}$. The set of points satisfying the constraints is open in $\mathbb{R}^{2m}$. Hence, $\varphi_a^*(m)$ is an open subset of the $(2m-2)$-dimensional manifold $\varphi^*(m)\cap\mathcal{U}$. Thus, $\dim(\varphi_a^*(m))=2m-2.$\\
\\
Take any $\phi\in\varphi_a^*(m)$. By construction, $|u_1|+|u_2| = \Theta(m^a)~\Rightarrow~\|\phi\|_2 \ge \sqrt{u_1^2+u_2^2}=\Omega(m^a)$. Also, $|\omega_1|=\left|\frac{c_1(m)}{u_1}\right|=\Theta\!\left(\frac{m^{a}}{m^a}\right)=\Theta\!\bigl(m^{a-a}\bigr)$.
For $j\ge 3$ we have $|u_j\omega_j|\le \kappa^2 m^{a-1}$, hence
\[
\left|\sum_{j=3}^m u_j\omega_j\right|
\le (m-2)\kappa^2 m^{a-1}\le \kappa^2 m^{a}.
\]
Therefore, for all sufficiently large $m$,
\[
\left|c_2(m)-\sum_{j=3}^m u_j\omega_j\right|
\ge |c_2(m)|-\left|\sum_{j=3}^m u_j\omega_j\right|
\ge |c_2(m)|-\kappa^2 m^{a}
\ge \left(|b_2|-\kappa^2\right)m^{a}
\ge \frac{\beta}{2}m^{a}.
\]
Also, 
\[
\left|c_2(m)-\sum_{j=3}^m u_j\omega_j\right|
\le |c_2(m)|+\left|\sum_{j=3}^m u_j\omega_j\right|
\le |c_2(m)|+\kappa^2 m^{a}
\le \left(|b_2|+\kappa^2\right)m^{a}
\le \frac{3|b_2|}{2}m^{a}.
\]
Hence, $\left|c_2(m)-\sum_{j=3}^m u_j\omega_j\right|=\Theta\!(m^{a})$. By using $|u_2|=\Theta(m^a)$, we get $|\omega_2|=\Theta\!\left(\frac{m^{a}}{m^a}\right)=\Theta\!\bigl(m^{a-a}\bigr)$.\\
Moreover, by construction, $|u_j|=|\omega_j|=\Theta\!\left(m^{\frac{a-1}{2}}\right)\quad (j\ge 3)$. Hence, 
\[
\sum_{j=3}^m (u_j^2+\omega_j^2)
=O\!\left(2(m-2)m^{a-1}\right)=O\!\left(m^{a}\right).
\]
Their total contribution satisfies $\left(\sum_{j=3}^m (u_j^2+\omega_j^2)\right)^{1/2}=O\!\left(m^{\frac{a}{2}}\right).$
Since $a\ge \frac{a}{2}$, we have $m^{\frac{a}{2}}\le m^a$, so $\left(\sum_{j=3}^m (u_j^2+\omega_j^2)\right)^{1/2}=O\!\left(m^{a}\right)$. Also,
$a\ge \frac{a}{2}$ implies $a-a\le a$, so $|\omega_1|,|\omega_2|=O(m^a)$. By construction, $|u_1|,|u_2|=O(m^a)$.
Combining these bounds yields
\[
\|\phi\|_2=\left(\sum_{j=1}^m (u_j^2+\omega_j^2)\right)^{1/2}
\le \sqrt{u_1^2+u_2^2} + \sqrt{\omega_1^2+\omega_2^2}
     +\left(\sum_{j=3}^m (u_j^2+\omega_j^2)\right)^{1/2}
= O(m^a).
\]
Together with the lower bound $\|\phi\|_2=\Omega(m^a)$, we conclude $\forall \phi\in\varphi_a^*(m),~\|\phi\|_2=\Theta(m^a)$.
This completes the proof.
\end{proof}
\section{More analysis on connectivity result}
\subsection{Roles of Permutation Symmetry and Inactive Neurons for Connectivity}\label{app:neurons}
Both of the two proofs in Appendix~\ref{app:conn_proof} for Theorem~\ref{thm:hyp_connectivity_main} exploit two basic facts about the neural network model: hidden neurons have a permutation symmetry, and overparameterized models have unnecessary hidden neurons to express an optimal function. In this section, we explain their roles for the connectivity result.

The neurons in the neural network model (\ref{eq:g_2ReLu}) have a permutation symmetry i.e. changing $(u_i, v_i)_{i=1}^{m}$ to $(u_{\pi(i)}, v_{\pi(i)})_{i=1}^m$ for a permutation $\pi$ does not change its output or the training loss (\ref{eq:g_min_problem}). The $\ell_2$ regularization forces any neuron $(u_i,\omega_i)$ in optimal parameter $(u_i, v_i)_{i=1}^{m}\in\Theta^*(m)$ to be an inactive neuron i.e. $(u_i, v_i)=(0,0)$ or an active neuron $u_i\neq0$ and $\omega_i\neq0$. (See Proposition~\ref{prop:basic_result_2} in Appendix.) We define, by applying the concept of minimal optimal networks defined in \citet{kim_exploring_2025} to one-dimensional data, the set of Minimal Optimal Solutions as
\begin{equation}
\label{eq:theta-min}
\begin{aligned}
\Theta^*_{\min}(m) &\coloneq
\Bigl\{(u_j,\omega_j)_{j=1}^{m}\ \Bigm|\ 
\forall\, p \neq q \in [m], \omega_p \omega_q > 0
\Rightarrow
u_pu_q<0
\Bigr\}.
\end{aligned}
\end{equation}
This is the set of parameters that has the least possible number of active neurons. We find that all elements in $\Theta^*_{min}(m)$ are permutations of each other. (See Appendix Lemma~\ref{lem:cvx_conn_1}). In Figure~\ref{fig:main_varphi_3} (c) (d), Minimal Optimal Solutions are denoted by black points.
A key observation from the connectivity result is that once we have an inactive neuron in $(u_i,\omega_i)_{i=1}^{m}\in\Theta^*_{min}(m)$, $\Theta^*(m)$ becomes a connected set. This is because an inactive neuron plays a role in proving that permutations $\{(u_{\pi(i)}, v_{\pi(i)})_{i=1}^m\Big|~\pi\text{ is a permutation on \{1,~...~.m}\}\}$ are connected in $\Theta^*(m)$. (See Collorary~\ref{col:perm} in Appendix.) Additionally, a merging process, a process that is analogous to the merging process defined in \citet{kim_exploring_2025}, proves that all the elements in $\Theta^*(m)$ are connected to a point in $\Theta^*_{min}(m)$. (See Lemma~\ref{lem:cvx_conn_3}). Paths created in the merging process are illustrated by colored paths in Figure~\ref{fig:main_varphi_3} (c) (d).
\subsection{Connectivity and Convex Formulation}\label{app:convex}
In this section, we provide an explanation of why the connectivity phase transition is restricted for the one-dimensional input case (Theorem~\ref{thm:hyp_connectivity_main}) from a perspective on a convex formulation introduced by \citet{pilanci_neural_2020}. \citet{kim_exploring_2025} explored the connectivity of global optimal parameters for general $d$-dimensional input data by applying the convex formulation. We follow their principle strategies, but our proof is simpler because the analysis on one-dimensional input is enough for our purpose. (See Appendix~\ref{app:connectivity_convex}.)

Firstly, we find the convex formulation of the non-convex problem \eqref{eq:g_min_problem}. The convex formulation of the training problem is introduced by \cite{pilanci_neural_2020}. By restricting our attention to one-dimensional input, the convex formulation can be written as follows. By abuse of notation, we denote $(v_1,v_2,t_1,t_2)$ by $(v_i,t_i)_{i=1}^{2}$.
\begin{proposition}\label{prop:convex_main}
Consider the convex problem given as a cone-constrained group LASSO
\newcommand{\A}{(v_1-t_1)D(S) + (v_2-t_2)D(S^{c})}
\begin{equation}
\label{eq:conv_min_main}
    \min_{\substack{v_1, v_2, t_1, t_2\in\mathbb{R}\\v_1,t_1\geq 0\\v_2,t_2\leq 0}}L_{conv}(v_1,v_2,t_1,t_2),
\end{equation}
\begin{equation}\label{eq:conv_loss}
\begin{aligned}
&\text{where }L_{conv}(v_1,v_2,t_1,t_2)=\frac12 \Bigl\lVert
\Bigl((v_1-t_1)D(S) + (v_2-t_2)D(S^{c})\Bigr)\frac{X}{\alpha} - Y
\Bigr\rVert_2^2 \\
&\quad + \beta\bigl(|v_1|+|v_2|+|t_1|+|t_2|\bigr). \\
\end{aligned}
\end{equation}
The convex problem \eqref{eq:conv_min_main} and the non-convex problem \eqref{eq:g_min_problem} have identical optimal value when $m\geq~M^{*}=\sum_{i\in\{1,2\}:v_i^{*}\neq0}1+\sum_{i\in\{1,2\}:t_i^{*}\neq0}1$ where $(v^{*}_i, t^{*}_i)_{i=1}^{2}$ is an optimal solution to (\ref{eq:conv_min_main}). 
\end{proposition}

\begin{remark}\label{rmk:equiv}
$M^*$ defined in Theorem~\ref{thm:hyp_connectivity_main} and $M^*$ defined in Proposition~\ref{prop:convex_main} are the same. This is because $\gamma_P^{*}$ and $\gamma_N^{*}$ introduced in Theorem~\ref{thm:global_optim_explicit_main} can be written as $\gamma_P^{*}=v_1^{*}-t_1^{*}$, $\gamma_N^{*}=v_2^{*}-t_2^{*}$ by Proposition~\ref{prop:convex_soln_main}. We call $M^*$ the critical value.
\end{remark}

Since the problem is convex, we can solve (\ref{eq:conv_min_main}) directly to find optimal solutions that satisfy the constraints. It turns out that the solution set is a singleton.
\begin{proposition}\label{prop:convex_soln_main}
The set of global optimal solutions 
\begin{align*}
\mathcal{P}^{*}
=
\left\{
(v_1,v_2,t_1,t_2)\;\middle|\;
\operatorname*{arg\,min}_{\substack{
v_1,v_2,t_1,t_2\in\mathbb{R}\\
v_1,t_1\ge 0\\
v_2,t_2\le 0
}}
L_{\mathrm{conv}}(v_1,v_2,t_1,t_2)
\right\}
\end{align*}
for the convex problem (\ref{eq:conv_min_main}) 
is 
$\mathcal{P}^{*}=\{(v^{*}_1, v^{*}_2, t^{*}_1, t^{*}_2)\}$ where
{\small
\setlength{\jot}{2pt}
\renewcommand{\arraystretch}{1.05}
\[
\begin{aligned}
(v_1^{*}, t_1^{*}) &=
\begin{cases}
\left(\alpha\dfrac{X_S^{\top}Y-\alpha\beta}{\|X_S\|_2^{2}},\,0\right)
& \text{if } X_S^{\top}Y>\alpha\beta,\\
(0,0)
& \text{if } -\alpha\beta\le X_S^{\top}Y\le \alpha\beta,\\
\left(0,\,-\alpha\dfrac{X_S^{\top}Y+\alpha\beta}{\|X_S\|_2^{2}}\right)
& \text{if } X_S^{\top}Y<-\alpha\beta,
\end{cases}
\\[2pt]
(v_2^{*}, t_2^{*}) &=
\begin{cases}
\left(0,\,\alpha\dfrac{-X_{S^{c}}^{\top}Y+\alpha\beta}{\|X_{S^{c}}\|_2^{2}}\right)
& \text{if } X_{S^{c}}^{\top}Y>\alpha\beta,\\
(0,0)
& \text{if } -\alpha\beta\le X_{S^{c}}^{\top}Y\le \alpha\beta,\\
\left(\alpha\dfrac{X_{S^{c}}^{\top}Y+\alpha\beta}{\|X_{S^{c}}\|_2^{2}},\,0\right)
& \text{if } X_{S^{c}}^{\top}Y<-\alpha\beta.
\end{cases}
\end{aligned}
\]
}
\end{proposition}
The staircase connectivity shown in \cite{kim_exploring_2025} arises because we can relate the connectivity results of the cardinarity constraint set $\mathcal{P}^*(m)\subseteq\mathcal{P}^*$ defined as 
\begin{equation}
\label{eq:Pstar-m}
\begin{aligned}
\mathcal{P}^{*}(m)
:= &\Bigl\{(u_i,v_i)_{i=1}^{P}\ \Bigm|\ 
(u_i,v_i)_{i=1}^{P} \in \mathcal{P}^{*}, \sum_{i\in[P]:v_i^{*}\neq0}1+\sum_{i\in[P]:t_i^{*}\neq0}1 \le m
\Bigr\}.
\end{aligned}
\end{equation}
with the connectivity results of $\Theta^*(m)$, and $\mathcal{P}^*(m)$ changes with respect to $m$. For 1-dimensional case, $\mathcal{P}^{*}$ is a singleton, so $\mathcal{P}^*(m)=\mathcal{P}^*$ for any $m\ge M^*$ and we observe limited connectivity phase transitions.
\section{Proofs for Appendix}

\subsection{Proofs for Appendix~\ref{app:convex}}\label{app:connectivity_convex}
In general, convex problems are easier to deal with than non-convex problems. \citet{kim_exploring_2025} showed that the staircase connectivity of global optimal solutions to the original non-convex training loss minimization problem can be found by analyzing the connectivity of global optimal solutions to its convex formulation. We apply this strategy to our problem with 1-dimensional input data. Thanks to the simplicity coming from 1-dimensional input, we provide a simpler proof without using all the notations introduced by \citet{kim_exploring_2025}.\\
\\
A benefit from deriving connectivity via convex formulation is that we can derive the connectivity without knowing the explicit form of $\Theta^*(m)$. Therefore, we pretend as if we did not know the explicit form (\ref{eq:explicit}) throughout Section~\ref{app:connectivity_convex}. \\
\\
Firstly, we find the convex formulation of our non-convex problem \eqref{eq:g_min_problem}.
\begin{proposition}\label{prop:convex} (Proposition~\ref{prop:convex_main} in Appendix~\ref{app:convex})
Consider the convex problem given as a cone-constrained group LASSO
\newcommand{\A}{(v_1-t_1)D(S) + (v_2-t_2)D(S^{c})}
\begin{equation}
\label{eq:conv_min}
\begin{aligned}
\min_{v_1,v_2,t_1,t_2}
\frac12 \Bigl\lVert
\Bigl((v_1-t_1)D(S) + (v_2-t_2)D(S^{c})\Bigr)\frac{X}{\alpha} - Y
\Bigr\rVert_2^2 
\quad + \beta\bigl(|v_1|+|v_2|+|t_1|+|t_2|\bigr) \\
\text{s.t.}\quad
v_1 \ge 0,\; t_1 \ge 0,\; v_2 \le 0,\; t_2 \le 0
\end{aligned}
\end{equation}
where $D(S)=\operatorname{Diag}(\mathbf{1}[X\geq 0])$ and $D(S^{c})=\operatorname{Diag}(\mathbf{1}[X\leq 0])$.\\
The convex problem (\ref{eq:conv_min}) and the non-convex problem (\ref{eq:g_min_problem}) have identical optimal value when $m\geq~M^{*}=\sum_{i\in\{1,2\}:v_i^{*}\neq0}1+\sum_{i\in\{1,2\}:t_i^{*}\neq0}1$ where $\{v^{*}_i, t^{*}_i\}_{i=1}^{2}$ is an optimal solution to (\ref{eq:conv_min}). 
\end{proposition}
\begin{proof}
    By our assumption in Appendix~\ref{app:convex}, $X$ has at least one positive element and one negative element. Hence,
    \[
    \{ \operatorname{Diag}(\mathbf{1}[Xu\ge0]~|~u\in\mathbb{R}\}=\{\operatorname{Diag}(\mathbf{1}[X\ge0]),~\operatorname{Diag}(\mathbf{1}[X\le 0]),~I_n\}.
    \]
    \eqref{eq:g_min_problem} can be equivalently written as 
    \begin{align}\label{eq:loss_function_app}
L(\theta)= \frac{1}{2}\|\sum_{j=1}^{m}(\frac{X}{\alpha}u_j)_{+}\omega_{j}-Y\|_{2}^{2}+\frac{\beta}{2}\sum_{j=1}^{m}(u_j^{2}+\omega_{j}^{2})
\end{align}
    Consider the following convex problem
    \begin{align}\label{eq:crude_convex}
    \min_{\{v_i,t_i\}_{i=1}^3}
\frac12 \Bigl\lVert
\Bigl((v_1-t_1)D(S) + (v_2-t_2)D(S^{c}) + (v_3-t_3)I_n)\Bigr)\frac{X}{\alpha} - Y
\Bigr\rVert_2^2 
\quad + \beta\sum_{i=1}^{3}\bigl(|v_i|+|t_i|\bigr) \notag \\
\text{s.t.}\quad
(2D_i-I_n)Xv_i\ge 0,\quad (2D_i-I_n)Xt_i\ge0,~\forall i\in[3].~~~
(D_1=D(S),~D_2=D(S^c),~D_3=I_n)
    \end{align}
By applying Theorem~1 in \citet{pilanci_neural_2020} to our problem with $d=1$ and input matrix $\frac{X}{\alpha}$, the convex problem (\ref{eq:crude_convex}) and the non-convex problem (\ref{eq:g_min_problem}) have identical optimal values if $m\geq~M^{*}=\sum_{i\in[3]:v_i^{*}\neq0}1+\sum_{i\in[3]:t_i^{*}\neq0}1$ where $\{v^{*}_i, t^{*}_i\}_{i=1}^{2}$ is an optimal solution to (\ref{eq:conv_min}).
As $X$ has both positive and negative elements,
\begin{align*}
&(2D_i-I_n)Xv_i\ge 0,\quad (2D_i-I_n)Xt_i\ge0,~\forall i\in[3]\\
\iff& v_1 \ge 0,\; t_1 \ge 0,\; v_2 \le 0,\; t_2 \le 0,\; v_3=t_3=0
\end{align*}
Rewriting (\ref{eq:crude_convex}) with this condition gives (\ref{eq:conv_min}).
\end{proof}

To find the connectivity of $\Theta^*(m)$, we use this convex formulation. We find the solution set $\mathcal{P}^*$ to the convex problem (\ref{eq:conv_min}) is a singleton.

\begin{proposition}\label{prop:convex_soln}(Proposition~\ref{prop:convex_soln_main} in Appendix~\ref{app:convex})
The set of global optimal solutions 
\begin{align*}
\mathcal{P}^{*}
=
\left\{
(v_1,v_2,t_1,t_2)\;\middle|\;
\operatorname*{arg\,min}_{\substack{
v_1,v_2,t_1,t_2\in\mathbb{R}\\
v_1,t_1\ge 0\\
v_2,t_2\le 0
}}
L_{\mathrm{conv}}(v_1,v_2,t_1,t_2)
\right\}
\end{align*}
for the convex problem (\ref{eq:conv_min}) 
is 
$\mathcal{P}^{*}=\{(v^{*}_1, v^{*}_2, t^{*}_1, t^{*}_2)\}$ where
\begin{align*}
&(v_1^{*}, t_1^{*})=  \left\{
\begin{array}{ll}
\left(\alpha\frac{X_S^{T}Y-\alpha\beta}{\|X_S\|_2^{2}},0\right)~~~~\text{if}~~~X_S^{T}Y>\alpha\beta~~~\\
(0,0)~~~~~~~~~~~~~~\text{if}~~~-\alpha\beta\leq X_S^{T}Y\leq \alpha\beta\\
\left(0,-\alpha\frac{X_S^{T}Y+\alpha\beta}{\|X_S\|_2^{2}}\right)~~~\text{if}~~~X_S^{T}Y<-\alpha\beta
\end{array}
\right.
\\
&(v_2^{*}, t_2^{*})=  \left\{
\begin{array}{ll}
\left(0, \alpha\frac{-X_{S^{c}}^{T}Y+\alpha\beta}{\|X_{S^c}\|_2^{2}}\right)~~~~\text{if}~~~X_{S^{c}}^{T}Y>\alpha\beta,~~~\\
(0,0)~~~~~~~~~~~~~~~~\text{if}~~~-\alpha\beta\leq X_{S^{c}}^{T}Y\leq \alpha\beta,\\
\left(\alpha\frac{X_{S^{c}}^{T}Y+\alpha\beta}{\|X_{S^{c}}\|_2^{2}}, 0\right)~~~~~~\text{if}~~~X_{S^{c}}^{T}Y<-\alpha\beta.
\end{array}
\right.
  \end{align*}
\end{proposition}

\begin{proof}
\begin{equation*}
\begin{aligned}
L_{conv}(v_1,v_2,t_1,t_2)=
\frac12 \Bigl\lVert
\Bigl((v_1-t_1)D(S) + (v_2-t_2)D(S^{c})\Bigr)\frac{X}{\alpha} - Y
\Bigr\rVert_2^2 
+ \beta\bigl(|v_1|+|v_2|+|t_1|+|t_2|\bigr). 
\end{aligned}
\end{equation*}
By the KKT condition, at an optimal,
\begin{align*}
    0\in\partial_{v_1}L_{conv}=(v_1-t_1)\Big\|\frac{X_S}{\alpha}\Big\|^2-\frac{X_S^{T}Y}{\alpha}+\beta\partial|v_1|,\\
    0\in\partial_{t_1}L_{conv}=(t_1-v_1)\Big\|\frac{X_S}{\alpha}\Big\|^2+\frac{X_S^{T}Y}{\alpha}+\beta\partial|t_1|.\\
\end{align*}
To satisfy the above conditions and the constraint $v_1,t_1\geq 0$, optimal $v_1^*,t_1^*$ are uniquely defined as
\begin{align*}
&(v_1^{*}, t_1^{*})=  \left\{
\begin{array}{ll}
\left(\alpha\frac{X_S^{T}Y-\alpha\beta}{\|X_S\|_2^{2}},0\right)~~~~\text{if}~~~X_S^{T}Y>\alpha\beta~~~\\
(0,0)~~~~~~~~~~~~~~\text{if}~~~-\alpha\beta\leq X_S^{T}Y\leq \alpha\beta\\
\left(0,-\alpha\frac{X_S^{T}Y+\alpha\beta}{\|X_S\|_2^{2}}\right)~~~\text{if}~~~X_S^{T}Y<-\alpha\beta
\end{array}
\right.
\end{align*}
We can find optimal $v_2^*$, $t_2^*$ by replacing $X_S$ by $X_{S^c}$ and by replacing the constraint by $v_2,t_2\le 0$.
\end{proof}

We state a constraint about the form of solutions in $\Theta^*(m)$. This constraint comes from the $\ell_2$-regularization.
\begin{proposition}\label{prop:basic_result_2}
    $\forall(u_i, \omega_i)_{i=1}^{m}\in\Theta^*(m)$, $|u_i|=|\omega_i|$ holds for all $i\in[m]$.
\end{proposition}
\begin{proof}
 Take any $(u_i, \omega_i)_{i=1}^{m}\in\Theta^*(m)$. We have two cases to consider.\\
 \\
(Case 1)~$u_i=0$ (or $\omega_i=0$):\\
The choice of $\omega_i$ (or $u_i=0$) does not change the model function $f_\theta(X)$ defined in (\ref{eq:g_2ReLu}). For an optimal parameter, we need $\omega_i=0$ (or $u_i=0$) to minimize the $\ell_2$-regularization term.
This implies that $u_i=0\iff\omega_i=0$. Hence, $|u_i|=|\omega_i|$.\\
\\
(Case 2)~$u_i\neq0$:\\
From Case 1, $u_i\neq0\Rightarrow\omega_i\neq0$. $\forall r>0$, changing $(u_i, \omega_i)$ to $(ru_i, \omega_i/r)$ does not change the model function $f_\theta(X)$.  By AM-GM inequality, $r^2 u_i^2+\frac{\omega_i^2}{r^2}\geq2|u_i||\omega_i|$ with equality iff $r^2 u_i^2=\frac{\omega_i^2}{r^2}$. \\
For $r>0$, $r^2 u_i^2=\frac{\omega_i^2}{r^2}$ implies $r=\frac{|\omega_i|}{|u_i|}$. Hence, by minimality, $(u_i, \omega_i)=(ru_i, \omega_i/r)=(\frac{|\omega_i|}{|u_i|}u_i, \frac{|u_i|}{|\omega_i|}\omega_i)$. This implies $|u_i|=|\omega_i|$.
\end{proof}

We take the following steps to prove Theorem~\ref{thm:connectivity}. 
\begin{enumerate}
\item Construct functions that map elements in $\mathcal{P}^*$ and elements in $\Theta^*(m)$. (Definition~\ref{def:Psi} and Definition~\ref{def:phi})
\item Introduce the notion of Minimal Optimal Solution. (Definition~\ref{def:min_param})
\item Prove that Minimal Optimal Solutions are permutations of each other. (Lemma~\ref{lem:cvx_conn_1})
\item Prove that any optimal solution is connected to a Minimal Optimal Solution. (Lemma~\ref{lem:cvx_conn_3})
\item For $m=M^*$, prove that all the optimal solutions are Minimal Optimal Solutions. (Lemma~\ref{lem:cvx_fin_2})
\item For $m>M^*$, prove that all the permutation solutions are connected in $\Theta^*(m)$. (Lemma \ref{lem:cvx_conn_2})
\end{enumerate}
Step 3 and Step 5 together imply that $\Theta^*(M^*)$ is finite. Step 3, Step 4, and Step 6 together imply that $\Theta^*(m)$ is connected when $m>M^*$.

\begin{definition}\label{def:Psi}
    Suppose $m\geq M^*$. Define $\Psi:\mathcal{P}^*\rightarrow\Theta^*(m)$ as
    \[
    \Psi((v_i^{*}, t_i^{*})_{i=1}^{2})
= 
\Big(
\frac{v_i^{*}}{\sqrt{|v_i^{*}|}},~ 
\sqrt{|v_i^{*}|}
\Big)_{v_i^{*} \ne 0}
\oplus
\Big(
\frac{t_i^{*}}{\sqrt{|t_i^{*}|}},~
-\sqrt{|t_i^{*}|}
\Big)_{t_i^{*} \ne 0}
\oplus (0,0)^{m - M^{*}}.
\]
\end{definition}

\begin{definition}\label{def:phi}
    Suppose $m\geq M^*$. Define $\Phi:\Theta^*(m)\rightarrow\mathcal{P}^*$ as
    \begin{align*}  
    \Phi((u_i, \omega_i)_{i=1}^{m})
&=(v_i, t_i)_{i=1}^{2}\\
&\coloneq 
\left\{
\begin{array}{ll}
v_1=\sum_{i\in\mathcal{I}}u_i|\omega_i|~~~\text{where}~~~\mathcal{I}=\{i~|~\omega_i>0, u_i>0\},\\
v_2=\sum_{i\in\mathcal{I}}u_i|\omega_i|~~~\text{where}~~~\mathcal{I}=\{i~|~\omega_i>0, u_i<0\},\\
t_1=\sum_{i\in\mathcal{I}}u_i|\omega_i|~~~\text{where}~~~\mathcal{I}=\{i~|~\omega_i<0, u_i>0\},\\
t_2=\sum_{i\in\mathcal{I}}u_i|\omega_i|~~~\text{where}~~~\mathcal{I}=\{i~|~\omega_i<0, u_i<0\}.
\end{array}
\right.\\
&\left(\sum_{i\in\mathcal{I}}u_i|\omega_i|=0~~~\text{if}~~~\mathcal{I}=\emptyset.\right)
\end{align*}

\end{definition}

For simplicity, we take $\alpha=1$ in the following arguments. We can prove the same connectivity result for our case by replacing $X$ by $\frac{X}{\alpha}$ ($\alpha>0$).
\begin{proposition}\label{prop:map}
    Suppose $m\geq M^*$. The maps $\Psi$ and $\Phi$ are well-defined.
\end{proposition}
\begin{proof}
The function values for $\Phi$ and $\Psi$ are uniquely determined for each input, so it is enough to show that $\forall (u_i, \omega_i)_{i=1}^{m}\in \Theta^*(m)$, $\Phi((u_i, \omega_i)_{i=1}^{m})=(v_i^*, t_i^*)_{i=1}^{2}\in\mathcal{P}^*$ and for $(v_i^*,t_i^*)_{i=1}^2\in\mathcal{P}^*$,  $\Psi((v_i^{*}, t_i^{*})_{i=1}^{2})\in\Theta^*(m)$.\\
To prove the first part, by Proposition~\ref{prop:convex}, it is enough to show that $L_{conv}(\Phi((u_i, \omega_i)_{i=1}^{m}))=L((u_i, \omega_i)_{i=1}^{m})$ holds for $(u_i, \omega_i)_{i=1}^{m}\in \Theta^*(m)$. We denote $\Phi((u_i, \omega_i)_{i=1}^{m})$ by $(v_i^*, t_i^*)_{i=1}^2$.
\begin{align*}
    (v_1^*-t_1^*)D(S)X&=\left(\sum_{\substack{i:u_i>0\\ \omega_i>0}}u_i|\omega_i|-\sum_{\substack{i:u_i>0\\ \omega_i<0}}u_i|\omega_i|\right)D(S)X\\
    &=\sum_{i:u_i>0}u_i\omega_iD(S)X\\
    &=\sum_{i:u_i>0}(Xu_i)_{+}\omega_i.\\
    (v_2^*-t_2^*)D(S^c)X&=\left(\sum_{\substack{i:u_i<0\\ \omega_i>0}}u_i|\omega_i|-\sum_{\substack{i:u_i<0\\ \omega_i<0}}u_i|\omega_i|\right)D(S^c)X\\
    &=\sum_{i:u_i<0}u_i\omega_iD(S^c)X\\
    &=\sum_{i:u_i<0}(Xu_i)_{+}\omega_i.
\end{align*}
$\forall \mathcal{I}, \forall i,j\in\mathcal{I}$, $\operatorname{sign}\left(u_i|\omega_i|\right)=\operatorname{sign}(u_j|\omega_j|)$, so $|\sum_{i\in\mathcal{I}}u_i|\omega_i||=\sum_{i\in\mathcal{I}}|u_i||\omega_i|=\frac{1}{2}\sum_{i\in\mathcal{I}}u_i^2+\omega_i^2$. The last equality holds by the result from Proposition~\ref{prop:basic_result_2}. \\
Hence, $L_{conv}(\Phi((u_i, \omega_i)_{i=1}^{m}))=\frac{1}{2}\|\sum_{j=1}^{m}(Xu_j)_{+}\omega_{j}-Y\|_{2}^{2}+\beta(|v^*_1|+|v^*_2|+|t^*_1|+|t^*_2|)=L((u_i, \omega_i)_{i=1}^{m})$.\\
\\
To prove the second part, by Proposition~\ref{prop:convex}, it is enough to show that $L_{conv}((v_i^*, t_i^*)_{i=1}^2)=L(\Psi((v_i^*, t_i^*)_{i=1}^2))$ holds for $(v_i^*, t_i^*)_{i=1}^2\in \mathcal{P}^*$. We denote $\Psi((v_i^*, t_i^*)_{i=1}^2)$ by $(u_i, \omega_i)_{i=1}^{m}$.
\begin{align*}
    L(\Psi((v_i^*, t_i^*)_{i=1}^2))=&\frac{1}{2}\|\sum_{i\in\{1,2\},v_i^*\neq0}\left(X\frac{v_i^*}{\sqrt{|v_i^{*}|}}\right)_{+}\sqrt{|v_i^{*}|}-\sum_{i\in\{1,2\},t_i^*\neq0}\left(X\frac{t_i^*}{\sqrt{|t_i^*|}}\right)_{+}\sqrt{|t_i^*|}-Y\|_{2}^{2}\\
    &+\frac{\beta}{2}\sum_{i\in\{1,2\},v_i^*\neq0}\left(\left(\frac{v_i^*}{\sqrt{|v_i^{*}|}}\right)^2+\left(\sqrt{|v_i^{*}|}\right)^2\right)+\frac{\beta}{2}\sum_{i\in\{1,2\},t_i^*\neq0}\left(\left(\frac{t_i^*}{\sqrt{|t_i^*|}}\right)^2+\left(\sqrt{|t_i^*|}\right)^2\right)\\
    =&\frac{1}{2}\|\sum_{i\in\{1,2\},v_i^*\neq0}\left(Xv_i^*\right)_{+}-\sum_{i\in\{1,2\},t_i^*\neq0}\left(Xt_i^*\right)_{+}-Y\|_{2}^{2}+\beta\sum_{i\in\{1,2\},v_i^*\neq0}|v_i^{*}|+\beta\sum_{i\in\{1,2\},t_i^*\neq0}|t_i^*|\\
    =&\frac{1}{2}\|\sum_{i\in\{1,2\},v_i^*\neq0}\left(Xv_i^*\right)_{+}-\sum_{i\in\{1,2\},t_i^*\neq0}\left(Xt_i^*\right)_{+}-Y\|_{2}^{2}+\beta\sum_{i\in\{1,2\}}(|v_i^{*}|+|t_i^*|)
\end{align*}
As $v_1^*,t_1^*\ge 0$ and $v_2^*,t_2^*\le 0$, $\sum_{i\in\{1,2\},v_i^*\neq0}\left(Xv_i^*\right)_{+}=D(S)Xv_1^*+D(S^c)Xv_2^*$ and \\$\sum_{i\in\{1,2\},t_i^*\neq0}\left(Xt_i^*\right)_{+}=D(S)Xt_1^*+D(S^c)Xt_2^*$. Hence, $L(\Psi((v_i^*, t_i^*)_{i=1}^2))=L_{conv}((v_i^*, t_i^*)_{i=1}^2)$.
\end{proof}

\begin{definition}\label{def:min_param}
The set of Minimal Optimal Solutions is defined for $m\ge M^*$ as 
\[
\Theta^*_{min}(m)\coloneq\{(u_j, \omega_j)_{j=1}^{m}|\forall p\neq q\in[m],~\omega_p\omega_q>0\Rightarrow u_pu_q<0\}
\]
\end{definition}
A Minimal Optimal Solution has the least possible number of active neurons.
The Lemma~\ref{lem:cvx_conn_1} proves that all the Minimal Optimal Solutions are permutations of each other.
\begin{lemma}\label{lem:cvx_conn_1}
Denote the set of all the permutations on $\{1,~...~,m\}$ by $S_m$. $\forall \pi\in S_m$, we call $(u_{\pi(i)},\omega_{\pi(i)})_{i=1}^m$ a permutation of $(u_{i},\omega_{i})_{i=1}^m$.
Then, every element in $\Theta^*_{min}(m)$ is a permutation of $\Psi((v^{*}_i, t^{*}_i)_{i=1}^{2})$.
\end{lemma}
\begin{proof}
$\Psi((v^{*}_i, t^{*}_i)_{i=1}^{2})\in\Theta_{min}^*(m)$ by its construction as Definition~\ref{def:Psi}, $v_1^*,t_1^*\ge 0$ and $v_2^*,t_2^*\le 0$. Hence, it is enough to show that $\forall (u_i, \omega_i)_{i=1}^{m}\in \Theta^*_{min}(m)$, $(u_i, \omega_i)_{i=1}^{m}$ is a permutation of $\Psi((v_i^{*}, t_i^{*})_{i=1}^{2})$.\\
For any element in $\Theta^*_{min}(m)$, the minimality shown in Definition~\ref{def:min_param} implies that $\mathcal{I}$ introduced in Definition~\ref{def:phi} is a singleton or an empty set. Take any $(u_i, \omega_i)_{i=1}^{m}\in \Theta^*_{min}(m)$. As $\mathcal{I}$ is a singleton,
\begin{align*}  
    \Phi((u_i, \omega_i)_{i=1}^{m})
&=(v^*_i, t^*_i)_{i=1}^{2}\\
&=
\left\{
\begin{array}{ll}
v^*_1=u_i|\omega_i|~~~\text{if}~~~\exists!~i~s.t.~\omega_i>0\wedge u_i>0,~0~otherwise,\\
t^*_1=u_i|\omega_i|~~~\text{if}~~~\exists!~i~s.t.~\omega_i<0\wedge u_i>0,~0~otherwise,\\
v^*_2=u_i|\omega_i|~~~\text{if}~~~\exists!~i~s.t.~\omega_i>0\wedge u_i<0,~0~otherwise,\\
t^*_2=u_i|\omega_i|~~~\text{if}~~~\exists!~i~s.t.~\omega_i<0\wedge u_i<0,~0~otherwise.
\end{array}
\right.
\end{align*}
Hence, $M^*=\sum_{i\in\{1,2\}:v_i^{*}\neq0}1+\sum_{i\in\{1,2\}:t_i^{*}\neq0}1=\sum_{i=1}^{m}1(u_i\omega_i\neq0)$.\\
We write  $P=\sum_{i=1}^{m}1(\omega_i>0)$ and $N=\sum_{i=1}^{m}1(\omega_i<0)$. Using the result from Proposition~\ref{prop:basic_result_2}, $P+Q=\sum_{i=1}^{m}1(u_i\omega_i\neq0)$, so $M^*=P+Q$. Denote the $P$ positive elements $\{\omega_i|\omega_i>0\}$ by $\{\omega_{j_1},~...~,\omega_{j_P}\}$. Denote the $N$ negative elements $\{\omega_i|\omega_i<0\}$ by $\{\omega_{k_1},~...~,\omega_{k_N}\}$.\\
Then, 
\begin{align*}
\Psi((v_i^{*}, t_i^{*})_{i=1}^{2})=
\Psi(\Phi((u_i, \omega_i)_{i=1}^{m}))=
\Big(
\frac{u_{j_i}|\omega_{j_i}|}{\sqrt{|u_{j_i}\omega_{j_i}|}},~ 
\sqrt{|u_{j_i}\omega_{j_i}|}
\Big)_{i=1}^P
\oplus
\Big(
\frac{u_{k_i}|\omega_{k_i}|}{\sqrt{|u_{k_i}\omega_{k_i}|}},~
-\sqrt{|u_{k_i}\omega_{k_i}|}
\Big)_{i=1}^{Q}
\oplus (0,0)^{m - M^{*}}.
\end{align*}
From Proposition~\ref{prop:basic_result_2}, $\forall j,~\sqrt{|u_j\omega_j|}=|\omega_j|$. Therefore, 
\begin{align*}
\Psi((v_i^{*}, t_i^{*})_{i=1}^{2})
=\Psi(\Phi((u_i, \omega_i)_{i=1}^{m}))
&=
\left(
u_{j_i},~ |\omega_{j_i}|
\right)_{i=1}^P
\oplus
\left(
u_{k_i},~ -|\omega_{k_i}|
\right)_{i=1}^{N}
\oplus (0,0)^{m - M^{*}}\\
&=
\left(
u_{j_i},~ \omega_{j_i}
\right)_{i=1}^P
\oplus
\left(
u_{k_i},~ \omega_{k_i}
\right)_{i=1}^{N}
\oplus (0,0)^{m - M^{*}}
\end{align*}
This is a permutation of $(u_i, \omega_i)_{i=1}^{m}$.
\end{proof}

\begin{lemma}\label{lem:cvx_conn_3}
    For any $(u_i, \omega_i)_{i=1}^{m}\in\Theta^{*}(m)$, there exists a point in $\Theta^{*}_{min}(m)$ that is connected to $(u_i, \omega_i)_{i=1}^{m}$.
\end{lemma}
\begin{proof}
    It is enough to show that $\forall (u_i, \omega_i)_{i=1}^{m}\in\Theta^{*}(m)\setminus \Theta^{*}_{min}(m)$, we can construct a continuous path that connects $(u_i, \omega_i)_{i=1}^{m}$ and a point in $\Theta^{*}_{min}(m)$.\\
    As $(u_i, \omega_i)_{i=1}^{m}\notin \Theta^{*}_{min}(m)$, wlog, $\omega_1\omega_2>0$ and $u_1u_2>0$. Denote $\operatorname{sign}(\omega_1)=\operatorname{sign}(\omega_2)$ by $s$. Define a path
    \[
    C(t)=\left(\frac{u_1 |\omega_1|+t u_2|\omega_2|}{\sqrt{|u_1\omega_1+t u_2\omega_2|}}, \sqrt{|u_1 \omega_1+t u_2\omega_2|}s\right)\oplus\left(\sqrt{1-t}\frac{u_2|\omega_2|}{\sqrt{|u_2\omega_2|}}, \sqrt{(1-t)|u_2\omega_2|}s\right)\oplus(u_j, \omega_j)_{j=3}^{m}
    \]
    where $t\in[0,1]$.\\
    $C(t)$ is a part of the path connecting $(u_i, \omega_i)_{i=1}^{m}$ and a point in $\Theta^{*}_{min}(m)$. To show this, we prove the following claims.
    \begin{enumerate}
        \item $C(t)$ is well-defined and is continuous.
        \begin{proof}
            $\omega_1\omega_2>0$ and $u_1u_2>0$ imply that $\operatorname{sign}(u_1\omega_1)=\operatorname{sign}(u_2\omega_2)\neq0$. Hence, $|u_1\omega_1+t u_2\omega_2|\neq 0$ for $t\in[0, 1]$. Also,  $\omega_2\neq0$ implies that $|u_2\omega_2|\neq 0$ by Proposition~\ref{prop:basic_result_2}. The well-definedness of $C(t)$ implies that $C(t)$ is continuous on $[0, 1]$. 
        \end{proof}
        \item The number of zero neurons in $C(1)$ is more than that of $C(0)$.
        \begin{proof}
            By direct calculation, $C(0)=(u_i, \omega_i)_{i=1}^m$ and\\ $C(1)=\left(\frac{u_1|\omega_1|+u_2|\omega_2|}{\sqrt{|u_1\omega_1+u_2\omega_2|}}, \sqrt{|u_1 \omega_1+ u_2\omega_2|}s\right)\oplus\left(0,0\right)\oplus(u_j, \omega_j)_{j=3}^{m}$. As $(u_1,\omega_1)\neq(0,0)$ and $(u_2,\omega_2)\neq(0,0)$, the claim is true.
        \end{proof}
        \item $C(t)$ is a path in $\Theta^*(m)$. 
        \begin{proof}
            By substituting $\theta=C(t)$ to $f_\theta(X)$ in \eqref{eq:g_2ReLu}, the sum of the first two terms is
            \begin{align*}
                &\left(X\frac{u_1 |\omega_1|+t u_2|\omega_2|}{\sqrt{|u_1\omega_1+t u_2\omega_2|}}\right)_{+}\sqrt{|u_1\omega_1+t u_2\omega_2|}s+\left(X\sqrt{1-t}\frac{u_2|\omega_2|}{\sqrt{|u_2\omega_2|}}\right)_{+}\sqrt{(1-t)|u_2\omega_2|}s\\
                =&\left(X (u_1 |\omega_1|+t u_2|\omega_2|)\right)_{+}s+(1-t)\left(Xu_2|\omega_2|\right)_{+}s\\
                =&\left(X u_1 |\omega_1|\right)_{+}s+\left(t X u_2|\omega_2|\right)_{+}s+(1-t)\left(Xu_2|\omega_2|\right)_{+}s~~~~~(\text{since }\operatorname{sign}(u_1|\omega_1|)=\operatorname{sign}(u_2|\omega_2|))\\
                =&(Xu_1)_{+}\omega_1+(Xu_2)_{+}\omega_2
            \end{align*}
            Hence, changing $(u_i, \omega_i)_{i=1}^{m}$ to $C(t)$ does not change the model function.\\
            \begin{align*}
    &\left(\frac{u_1 |\omega_1|+t u_2|\omega_2|}{\sqrt{|u_1\omega_1+t u_2\omega_2|}}\right)^2+\left( \sqrt{|u_1 \omega_1+t u_2\omega_2|}s\right)^2+\left(\sqrt{1-t}\frac{u_2|\omega_2|}{\sqrt{|u_2\omega_2|}}\right)^2+\left(\sqrt{(1-t)|u_2\omega_2|}s\right)^2\\
    =&\frac{(u_1 |\omega_1|+t u_2|\omega_2|)^2}{|u_1\omega_1+t u_2\omega_2|}+|u_1 \omega_1+t u_2\omega_2|+(1-t)\frac{(u_2|\omega_2|)^2}{|u_2\omega_2|}+(1-t)|u_2\omega_2|\\
    =&\frac{(u_1 |\omega_1|+t u_2|\omega_2|)^2}{|u_1\omega_1+t u_2\omega_2|}+|u_1 \omega_1+t u_2\omega_2|+2(1-t)|u_2\omega_2|\\
    =&2|u_1 \omega_1+t u_2\omega_2|+2(1-t)|u_2\omega_2|~~~~~\text{(since }\operatorname{sign}(\omega_1)=\operatorname{sign}(\omega_2))\\
    =& 2|u_1 \omega_1|+2|u_2\omega_2|~~~~~\text{(since $\operatorname{sign}(u_1\omega_1)=\operatorname{sign}(u_1\omega_1)$)}\\
    =&u_1^2+\omega_1^2+u_2^2+\omega_2^2
    \end{align*}
    Hence, changing $(u_i, \omega_i)_{i=1}^{m}$ to $C(t)$ does not change the $\ell_2$-regularization term.
    \\
    From the above arguments, changing $(u_i, \omega_i)_{i=1}^{m}$ to $C(t)$ does not change the loss value i.e. $L(u_i, \omega_i)_{i=1}^{m})=L(C(t))$. Hence, $\forall t\in[0,1]$, $C(t)$ is in $\Theta^*(m)$.
    \end{proof}

    \end{enumerate}
We can see that $C(t)$ merges the two active neurons $\{(u_1, \omega_1), (u_2, \omega_2)\}$ to generate one active neuron $\left(\frac{u_1|\omega_1|+u_2|\omega_2|}{\sqrt{|u_1\omega_1+u_2\omega_2|}}, \sqrt{|u_1 \omega_1+ u_2\omega_2|}s\right)$ and one inactive neuron (0,0). We repeat this merging process until we cannot find such a pair, i.e., we reach a point in $\Theta_{min}^*(m)$. This process should terminate since the merging process strictly decreases the number of active neurons. When the merging process ends, we concatenate all the paths. Then, we have a continuous path in $\Theta^*(m)$ starting from $(u_i, \omega_i)_{i=1}^{m}$. At the end of the path, we have a point in $\Theta^*_{min}(m)$.
\end{proof}

\begin{lemma}\label{lem:cvx_fin_2}
$\Theta^*(M^*)=\Theta^*_{min}(M^*)$ 
\end{lemma}
\begin{proof}
    Assume that there exists $A\in\Theta^{*}(M^*)\setminus \Theta^{*}_{min}(M^*)$. By applying the merging process defined in Lemma~\ref{lem:cvx_conn_3} to $A$, we get a point $B\in\Theta^{*}_{min}(M^*)$ that has at least one inactive neuron $(0,0)$. So, the number of non-zero neurons in $B$ is at most $M^*-1$. From the proof in Lemma~\ref{lem:cvx_conn_1}, for $B=(u_i,\omega_i)_{i=1}^{m}\in\Theta^{*}_{min}(M^*)$, $M^*=\sum_{i=1}^{m}1(u_i\omega_i\neq0)$ i.e. the number of non-zero neurons is $M^*$. 
\end{proof}

\begin{lemma}\label{lem:cvx_conn_2}
For $m\geq M^*+1$, all permutation solutions are connected. 
\end{lemma}
 The idea of the proof is inspired by \citet{kim_exploring_2025}. We prove it using notions of group theory, similarly to our other proofs (e.g., the proof for Theorem~\ref{thm:noweight_connectivity} (3), the first proof for Theorem~\ref{thm:connectivity}).
\begin{proof}
We use $S_m$ to denote the symmetric group on $\{1,2,~...~,m\}$. $S_m$ is the set of all permutations on $\{1,2,~...~,m\}$.\\
Using Lemma~\ref{lem:cvx_conn_1} and Lemma~\ref{lem:cvx_conn_3}, it is enough to show that 
\begin{align*}
\forall(u_i,\omega_i)_{i=1}^{m}\in\Theta_{min}^*(m),~\forall\sigma\in S_m,~\exists \text{a continuous path in $\Theta^*(m)$ connecting }\\(u_i,\omega_i)_{i=1}^{m} \text{ and }(u_{\sigma(i)},\omega_{\sigma(i)})_{i=1}^{m}.
\end{align*}
We take any $(u_i,\omega_i)_{i=1}^{m}\in\Theta_{min}^*(m)$. From the same argument in Lemma~\ref{lem:cvx_conn_1}, $m>M^*=\sum_{i=1}^{m}1(u_i\omega_i\neq0)$ implies that $(u_i,\omega_i)_{i=1}^{m}$ has at least one inactive neuron (0,0). Denote the index for the inactive neuron as $j\in[m]$, so $(u_j,\omega_j)=(0,0)$. As for all $j\in[m]$, $\mathcal{T}_j=\{(i,j) \mid i \in [m]~s.t.~i \neq j\}$ (We use $(i,j)$ to denote a transposition) generates $S_m$, it is enough to show 
\[
\forall T\in\mathcal{T}_j~\text{where}~u_j\neq0,~\exists \text{ a continuous path in $\Theta^*(m)$ connecting }(u_i,\omega_i)_{i=1}^{m} \text{ and }(u_{T(i)},\omega_{T(i)})_{i=1}^{m}.
\]
For $(u_{T(j)},\omega_{T(j)})=(0,0)$, applying $T$ does not change the parameters. So, it is enough to think about the case where $(u_{T(j)},\omega_{T(j)})\neq(0,0)$
We construct a path 
    \[
    C(t)=(u_j(t), \omega_j(t))_{j=1}^{m}\in(\mathbb{R}\times \mathbb{R})^{m},~t\in[0,1]
    \]
s.t. 
\begin{align*}
    (u_i(t), \omega_i(t))=
    \left\{
\begin{array}{ll}
    (u_i, \omega_i)~~~\text{if}~~~i\neq j\text{ and }i\neq T(j),\\
    \left(u_{T(j)}|\omega_{T(j)}|\sqrt{\frac{t}{|u_{T(j)}\omega_{T(j)|}}}, \sqrt{t|u_{T(j)}\omega_{T(j)}|}s\right)~~~\text{if}~~~i=j,\\
    \left(u_{T(j)}|\omega_{T(j)}|\sqrt{\frac{1-t}{|u_{T(j)}\omega_{T(j)|}}}, \sqrt{(1-t)|u_{T(j)}\omega_{T(j)}|}s\right)~~~\text{if}~~~i=T(j)
\end{array}
\right.
\end{align*}
where $s=\operatorname{sign}(\omega_{T(j)})$. We prove the following claims.
    \begin{enumerate}
        \item $C(t)$ is well-defined and is continuous.
        \begin{proof}
            Proposition~\ref{prop:basic_result_2} and $(u_{T(j)},\omega_{T(j)})\neq(0,0)$ imply that $|u_{T(j)}\omega_{T(j)}|\neq0$. So, division by $|u_{T(j)}\omega_{T(j)}|$ is possible and well-defined. The well-definedness of $C(t)$ implies that $C(t)$ is continuous on $[0, 1]$. 
        \end{proof}
        \item $C(0)=(u_i, \omega_i)_{i=1}^m$ and $C(1)=(u_{T(i)}, \omega_{T(i)})_{i=1}^m$
        \begin{proof}
            By direct calculation.
        \end{proof}
        \item $C(t)$ is a continuous path in $\Theta^*(m)$. 
        \begin{proof}
            \begin{align*}
                &\left(Xu_{T(j)}|\omega_{T(j)}|\sqrt{\frac{t}{|u_{T(j)}\omega_{T(j)|}}}\right)_{+}\sqrt{t|u_{T(j)}\omega_{T(j)}|}s\\
                &+\left(Xu_{T(j)}|\omega_{T(j)}|\sqrt{\frac{1-t}{|u_{T(j)}\omega_{T(j)|}}}\right)_{+}\sqrt{(1-t)|u_{T(j)}\omega_{T(j)}|}s\\
                =&t\left(Xu_{T(j)}\right)_{+}\omega_{T(j)}+(1-t)\left(Xu_{T(j)}\right)_{+}\omega_{T(j)}\\
                =&\left(Xu_{T(j)}\right)_{+}\omega_{T(j)}\\
                =&\left(Xu_j\right)_{+}\omega_j+\left(Xu_{T(j)}\right)_{+}\omega_{T(j)}
            \end{align*}
            Hence, changing $(u_i, \omega_i)_{i=1}^{m}$ to $C(t)$ does not change the function $f(X)$.
            \begin{align*}
    &\left(u_{T(j)}|\omega_{T(j)}|\sqrt{\frac{t}{|u_{T(j)}\omega_{T(j)|}}}\right)^2+\left(\sqrt{t|u_{T(j)}\omega_{T(j)}|}s\right)^2\\
    &+\left(u_{T(j)}|\omega_{T(j)}|\sqrt{\frac{1-t}{|u_{T(j)}\omega_{T(j)|}}}\right)^2+\left(\sqrt{(1-t)|u_{T(j)}\omega_{T(j)}|}s\right)^2\\
    &=t|u_{T(j)}\omega_{T(j)}|+(1-t)|u_{T(j)}\omega_{T(j)}|\\
    &=|u_{T(j)}\omega_{T(j)}|\\
    &\leq u_j^2+\omega_j^2+u_{T(j)}^2+\omega_{T(j)}^2
    \end{align*}
    Hence, changing $(u_i, \omega_i)_{i=1}^{m}$ to $C(t)$ does not increase the $\ell_2$ term. By the optimality of $(u_i, \omega_i)_{i=1}^{m}$, equality holds in the last inequality.\\
    \\
    From the above arguments, changing $(u_i, \omega_i)_{i=1}^{m}$ to $C(t)$ does not change the loss value, so $\forall t\in[0,1]$, $C(t)$ is in $\Theta^*(m)$.
        \end{proof}

    \end{enumerate}

From above arguments, $C(t)$ is a continuous path in $\Theta^*(m)$ connecting $(u_i,\omega_i)_{i=1}^{m}$ and $(u_{T(i)},\omega_{T(i)})_{i=1}^{m}$.
\end{proof}

\begin{corollary}\label{col:perm}
\begin{align*}
&\forall(u_i,\omega_i)_{i=1}^{m}\in\Theta^*(m)~s.t.~\exists~k\in[m]~s.t~(u_k,v_k)=(0,0) ,~\forall\sigma\in S_m,\\
&\exists \text{a continuous path in $\Theta^*(m)$ connecting }(u_i,\omega_i)_{i=1}^{m} \text{ and }(u_{\sigma(i)},\omega_{\sigma(i)})_{i=1}^{m}.
\end{align*}
\end{corollary}
\begin{proof}
This is a corollary from the proof of Lemma~\ref{lem:cvx_conn_2}.
\end{proof}

\textbf{The second proof for Theorem~\ref{thm:connectivity}}
\begin{proof} (The second proof)\\
 \[M^*=1[\mathcal{S}_{\alpha\beta}(X_S^TY)\neq0)+1[\mathcal{S}_{\alpha\beta}(X_{S^c}^TY)\neq0]=1[|\gamma_P^*|\neq0]+1[|\gamma_N^*|\neq0].\]
(1) $M^*=0$:\\
 $M^*=0$, implies $|\gamma_P^*|=0$ and $|\gamma_N^*|=0$. For $(u_i,\omega_1)_{i=1}^{m}\in\Theta^*(m)$,
   $\sum_{i:\,u_i\ge0} u_i^{2}=0$ and $\sum_{i:\,u_i\le0} u_i^2=0$, so $\forall i\in[m],~u_i=0$. Also, $\forall i\in[m],~|u_i|=|\omega_i|$, so $\forall i\in[m],~\omega_i=0$.\\
\\
(2) $m<M^*$:\\
 $\Theta^*(m)=\emptyset$ because we need at least $M^*$ non-zero neurons for $\sum_{i:\,u_i\ge0} u_i^{2}=|\gamma_P^*|$ and $\sum_{i:\,u_i\le0} u_i^{2}=|\gamma_N^*|$ to hold.\\
 \\
 (3) $m=M^*$:\\
 As a set of permutations of a finite number of neurons is a finite set, Lemma~\ref{lem:cvx_conn_1} implies that $\Theta^*_{min}(m)$ is a finite set for all $m$. By this and Lemma~\ref{lem:cvx_fin_2}, $\Theta^*(M^*)$ is a finite set.\\
 \\
 (4) $m>M^*$:\\
 By Lemma~\ref{lem:cvx_conn_1} and Lemma~\ref{lem:cvx_conn_2}, all the elements in $\Theta_{min}^*(m)$ are connected in $\Theta^*(m)$. By this and Lemmma~\ref{lem:cvx_conn_3}, all elemets in $\Theta^*(m)$ are connected, so $\Theta^*(m)$ is connected.
\end{proof}
\section{Visualizations of global minima}\label{app:visualizations}
The Figure~\ref{fig:varphi_3} (a) (b) show samples from $\varphi^*(3)$ and (c) (d) show all the points in $\Theta^*(3)$. We can visually see that adding $\ell_2$-regularization limits the set of globally optimal parameters. We also provide visualizations for global minima of other data sets and the regularization coefficient (Figure~\ref{fig:theta_2D}, \ref{fig:theta_inverse}, \ref{fig:theta_2m}).
\begin{figure}[t]
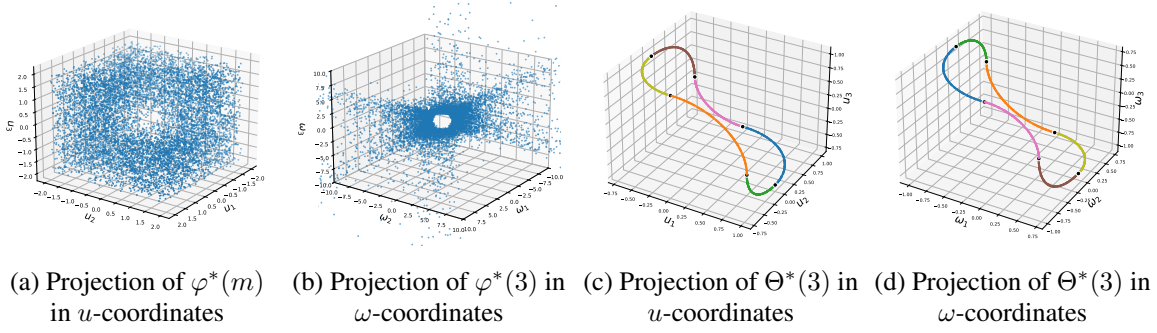

\centering
\begin{minipage}[b]{0.24\textwidth}
    \centering
    \includegraphics[width=\linewidth]{Visualization/varphi_u.pdf}

    \small (a) Projection of $\varphi^*(m)$ in $u$-coordinates
\end{minipage}\hfill
\begin{minipage}[b]{0.24\textwidth}
    \centering
    \includegraphics[width=\linewidth]{Visualization/varphi_w.pdf}

    \small (b) Projection of $\varphi^*(3)$ in $\omega$-coordinates
\end{minipage}\hfill
\begin{minipage}[b]{0.24\textwidth}
    \centering
    \includegraphics[width=\linewidth]{Visualization/theta_u24.pdf}

    \small (c) Projection of $\Theta^*(3)$ in $u$-coordinates
\end{minipage}\hfill
\begin{minipage}[b]{0.24\textwidth}
    \centering
    \includegraphics[width=\linewidth]{Visualization/theta_w24.pdf}

    \small (d) Projection of $\Theta^*(3)$ in $\omega$-coordinates
\end{minipage}

\caption{(Figure~\ref{fig:main_varphi_3} in main) (a) (b) \textbf{Projections of samples from $\varphi^*(3)$ onto $u$ and $\omega$ coordinates} (Theorem~\ref{prop:sqrd_loss}). The blue points are samples from $\varphi^*(3)$ for $X=[1,-1]^{T}$, $Y=[-\frac{2}{3},\frac12]^{T}$, $\alpha=3^1$. $u$ values are sampled with constraint that its $\ell_{\infty}$ norm is in [$10^{-100}, 2.0$]. For each sample of $u$, $\omega$ that satisfies the constraints is sampled from the $\omega$ coordinate. (c) (d) \textbf{Projections of $\Theta^*(3)$ onto $u$ and $\omega$ coordinates} (Theorem~\ref{thm:global_optim_explicit_main}). The dataset and scaling is the same $X=[1,-1]^{T}$, $Y=[-\frac{2}{3},\frac12]^{T}$, $\alpha=3^1$, and we additionally have weight decay $\beta=3^{-2}$ ($\gamma_P^*=-1$, $\gamma_N^*=-0.5$). The black dots correspond to Minimal Optimal Solutions defined in Eq.~\ref{eq:theta-min}. Different colors for $u$ represent different sign patterns. For each $u$, corresponding $\omega$ has the same color in $\omega$-coordinates.}
\label{fig:varphi_3}
\end{figure}

\begin{figure}
\centering
\includegraphics[width=0.8\linewidth]{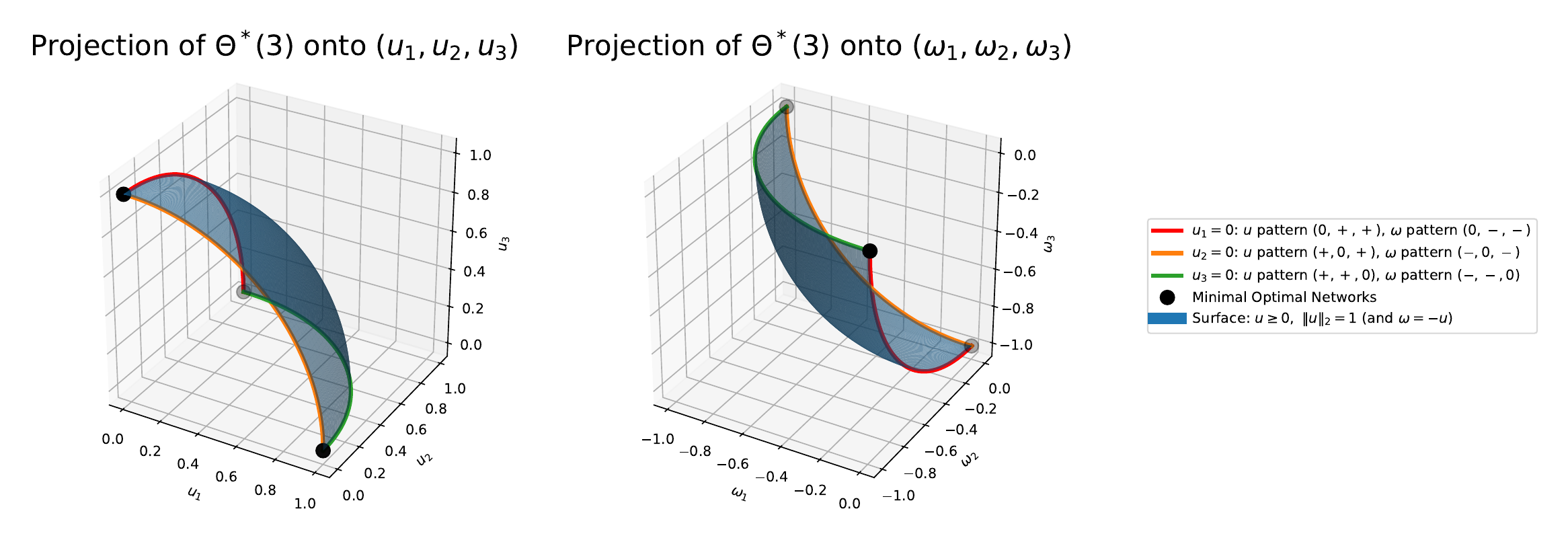}
\caption{The projections of $\Theta^*(3)$ onto $u$ and $\omega$ coordinates for $X=[1,-1]^{T}$, $Y=[-\frac{4}{9},\frac{1}{18}]^{T}$, $\alpha=3^1$, $\beta=3^{-3}$ ($\gamma_P^*=-1$, $\gamma_N^*=0$). The black dots correspond to Minimal Optimal Solutions defined in \eqref{eq:theta-min}. Both the surface and the colored boundaries are globally optimal parameters. For each $u$, corresponding $\omega$ has the same color in $\omega$ coordinate.}
\label{fig:theta_2D}
\end{figure}

\begin{figure}
\centering
\includegraphics[width=0.8\linewidth]{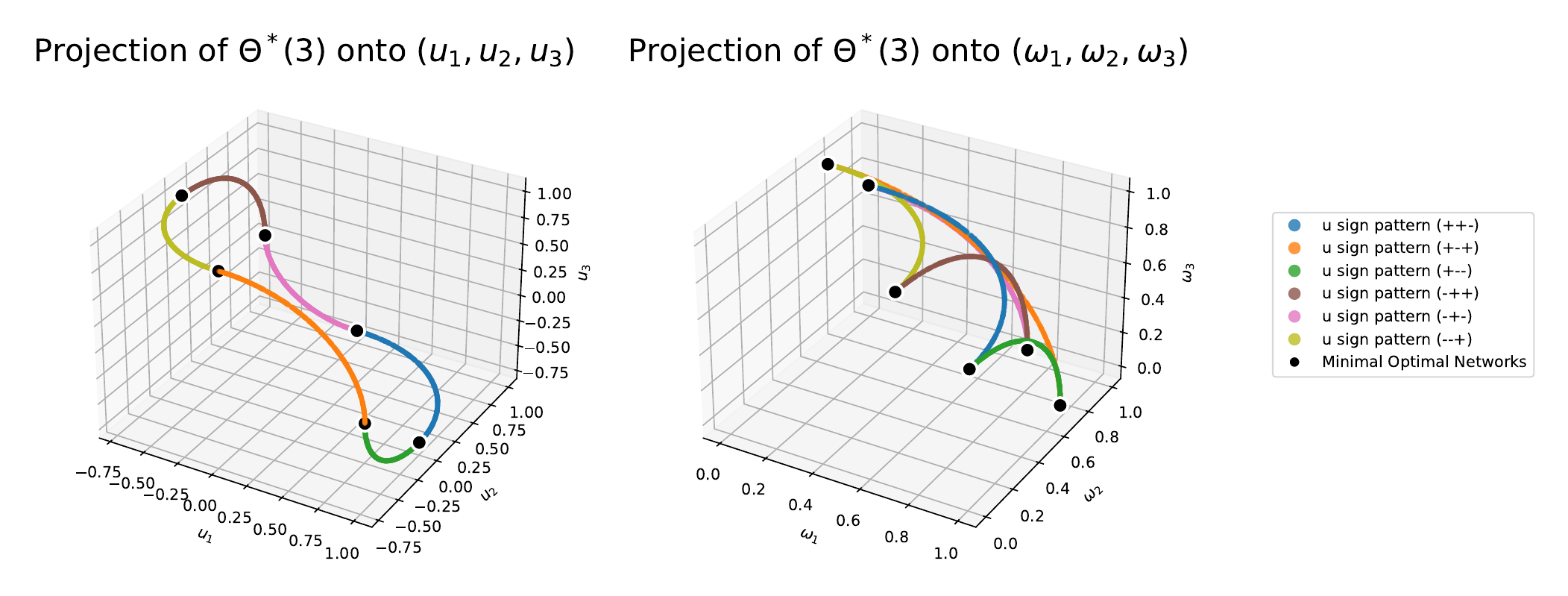}
\caption{The projections of $\Theta^*(3)$ onto $u$ and $\omega$ coordinates for $X=[1,-1]^{T}$, $Y=[\frac{2}{3},\frac{1}{2}]^{T}$, $\alpha=3^1$, $\beta=3^{-2}$ ($\gamma_P^*=1$, $\gamma_N^*=-0.5$). The black dots correspond to Minimal Optimal Solutions defined in \eqref{eq:theta-min}. For each $u$, corresponding $\omega$ has the same color in $\omega$ coordinate.}
\label{fig:theta_inverse}
\end{figure}

\begin{figure}
\centering
\includegraphics[width=0.6\linewidth]{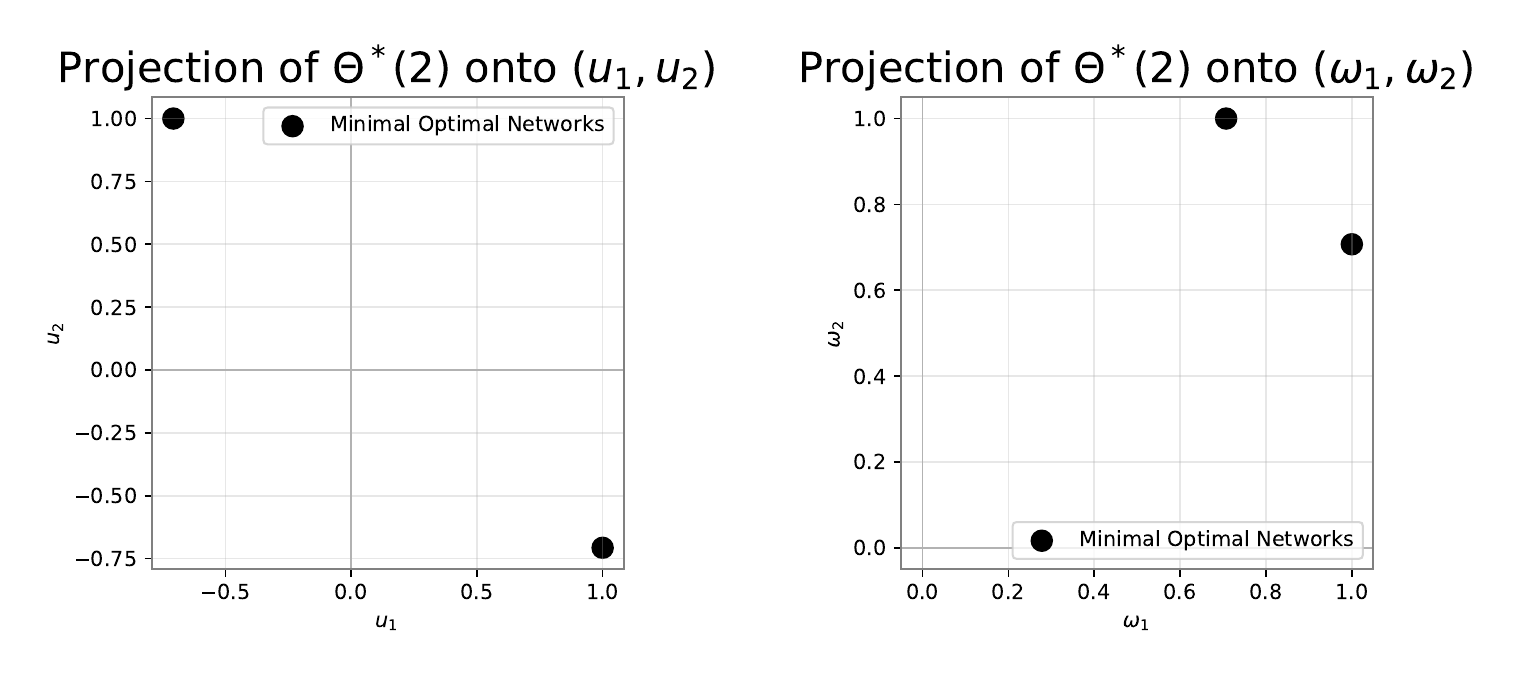}
\caption{The projections of $\Theta^*(2)$ onto $u$ and $\omega$ coordinates for $X=[1,-1]^{T}$, $Y=[1,\frac{4}{3}]^{T}$, $\alpha=2^1$, $\beta=2^{-2}$ ($\gamma_P^*=1$, $\gamma_N^*=-0.5$). The black dots correspond to Minimal Optimal Solutions defined in \eqref{eq:theta-min}. As $m=M^*$, $\Theta^*(m)$ is a finite set.}
\label{fig:theta_2m}
\end{figure}

\section{More results from Numerical Experiments}\label{app:ne}
We train two-layer ReLU neural networks \eqref{eq:g_2ReLu} with different widths ($m$=64, 128, 256, 512, 1024, 2048) using the Yacht Hydrodynamics data \cite{yacht_hydrodynamics_243} (squared loss), and MNIST \cite{MNIST} (cross entropy loss). For Yacht Hydrodynamics data \cite{yacht_hydrodynamics_243} (6-dimensional input), the network is trained with squared loss. For MNIST data \cite{MNIST} (784-dimensional input), the network is trained with the cross-entropy loss. We train them with stochastic gradient descent (SGD), and AdamW \cite{loshchilov2018decoupled} using $\beta$ as weight decay coefficient. For SGD, adding $\beta$ as a weight decay coefficient is equivalent to adding the explicit $\ell_2$-regularization. Whereas, for AdamW, we use squared loss or cross entropy as the minimization object and use $\beta$ as a weight decay to push learned parameters towards zero. We use 0.9 as the first moment decay and 0.999 as the second moment decay for AdamW. For all the experiments, the learning rate is set to be 0.01.\\
 \\
All the plots show final test loss, which does not include the $\ell_2$-regularization term, (above) and final $\ell_2$-norm of parameters (bottom, $y$-values taken in logarithmic scale base 2). The plots show the mean of values obtained from three independent experiments.\\
\\
We show the results when we trained the network with SGD (left) and AdamW (right). Across all the experimental settings, training with AdamW prevents the learned parameter from collapsing to zero. 
\subsection{Yacht Hydrodynamics data}
\begin{itemize}
    \item Figures~\ref{fig:yacht_widths_page1} and Figure \ref{fig:yacht_widths_page2} show results for networks with scaling and initialization in the neural tangent kernel setting \citep{jacot_neural_2020, arora2019exact, chizat_lazy_2020, luo_phase_2020} ($a=0.5$, $b_1=b_2=0$).
    \item Figures~\ref{fig:yacht_MF_page1} and Figure \ref{fig:yacht_MF_page2} show results for networks with scaling and initialization in the mean field setting \citep{mei2018mean, ChizatB18, luo_phase_2020, sirignano2020mean} ($a=0.5$, $b_1=b_2=0$).
    \item Figures~\ref{fig:yacht_sv1_NTK} and Figure \ref{fig:yacht_sv1_MF} show results for small initialization ($b_1=b_2=0.5$) for different scalings of the network ($a=0.5$ for Figure\ref{fig:yacht_sv1_NTK} and $a=1$ for Figure \ref{fig:yacht_sv1_MF}).
\end{itemize}

\begin{figure*}[p]
\centering
\setlength{\tabcolsep}{2pt}
\renewcommand{\arraystretch}{0.4}

\begin{tabular}{cc}
\includegraphics[width=0.46\textwidth]{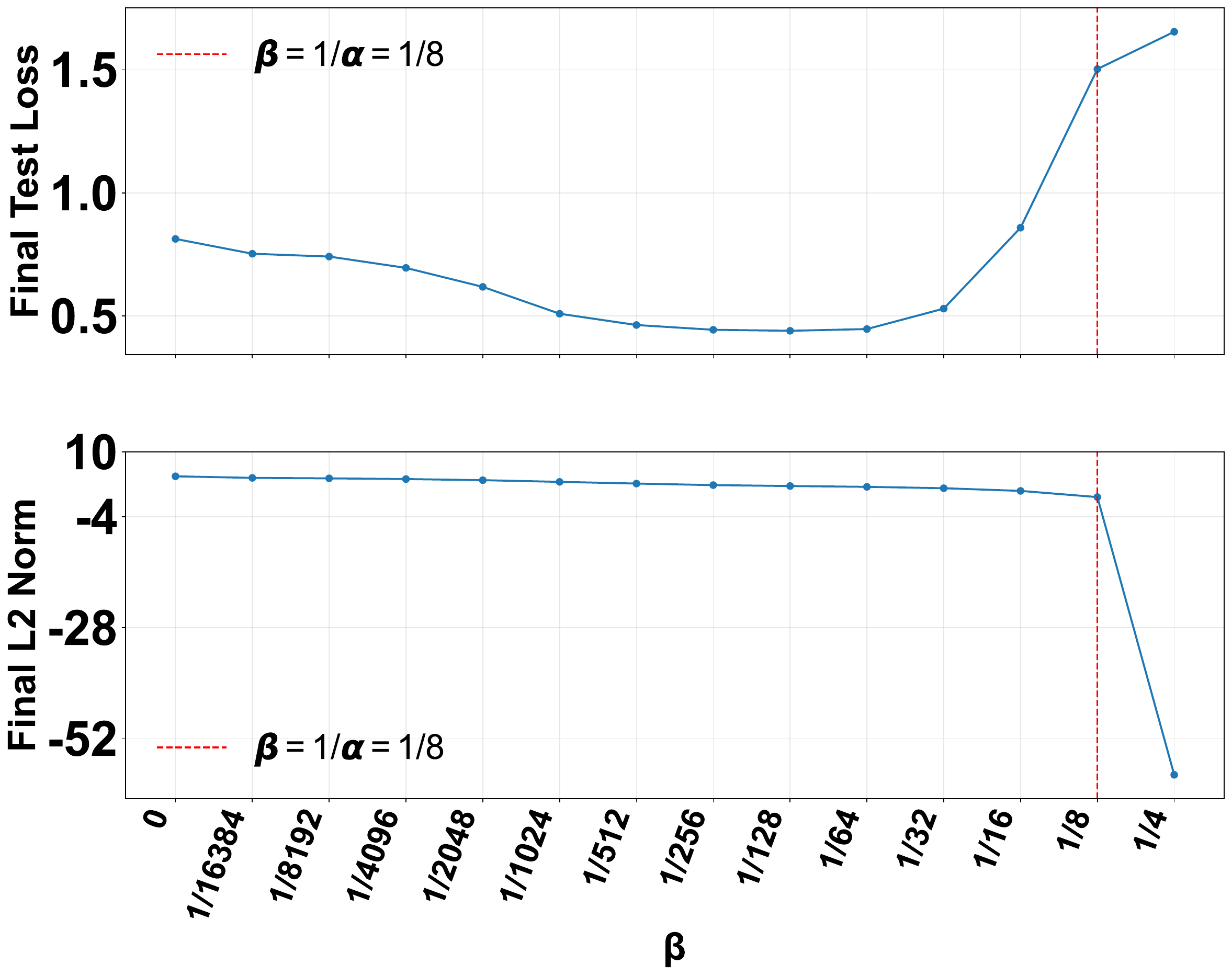} &
\includegraphics[width=0.46\textwidth]{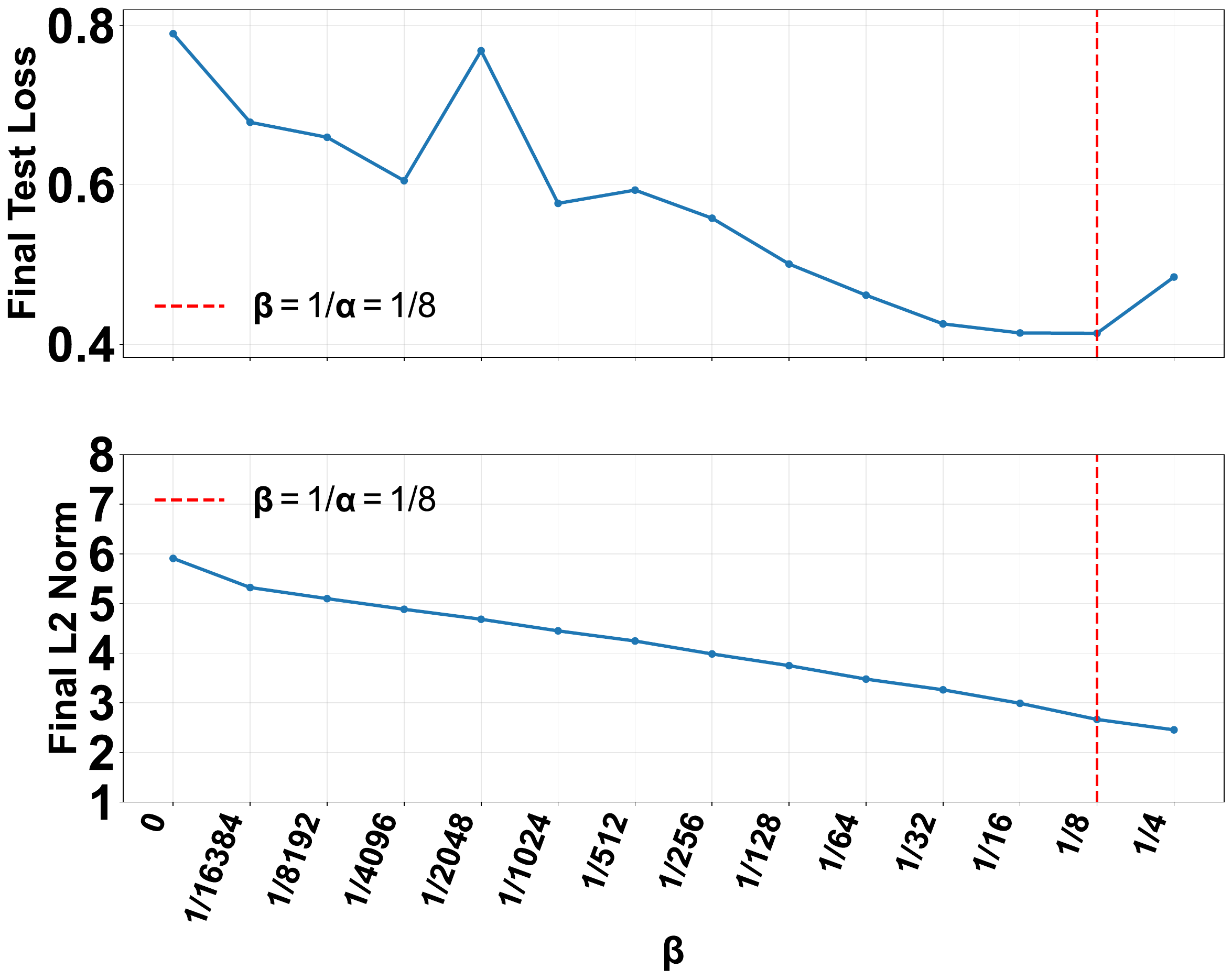} \\[-0.3em]

\includegraphics[width=0.46\textwidth]{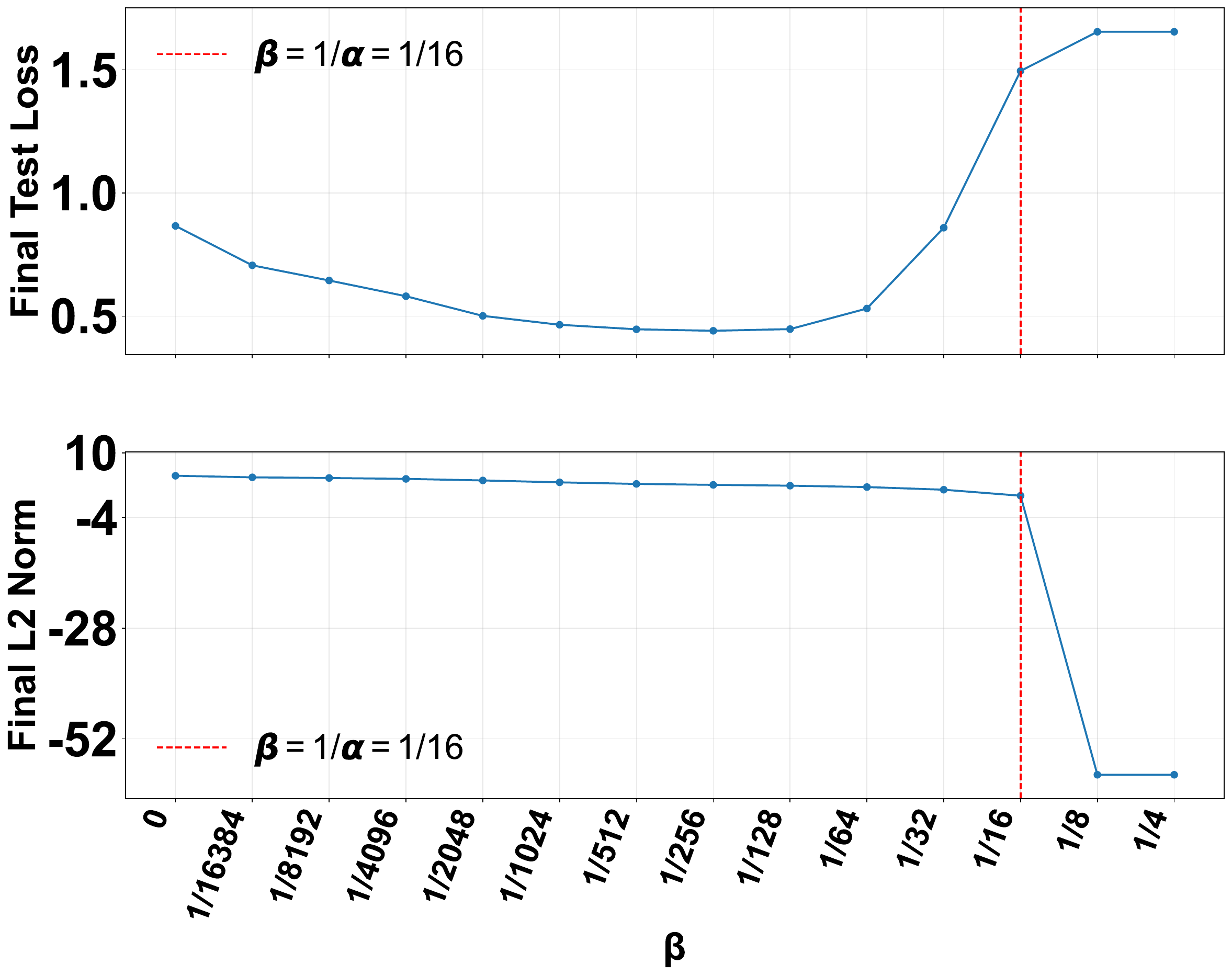} &
\includegraphics[width=0.46\textwidth]{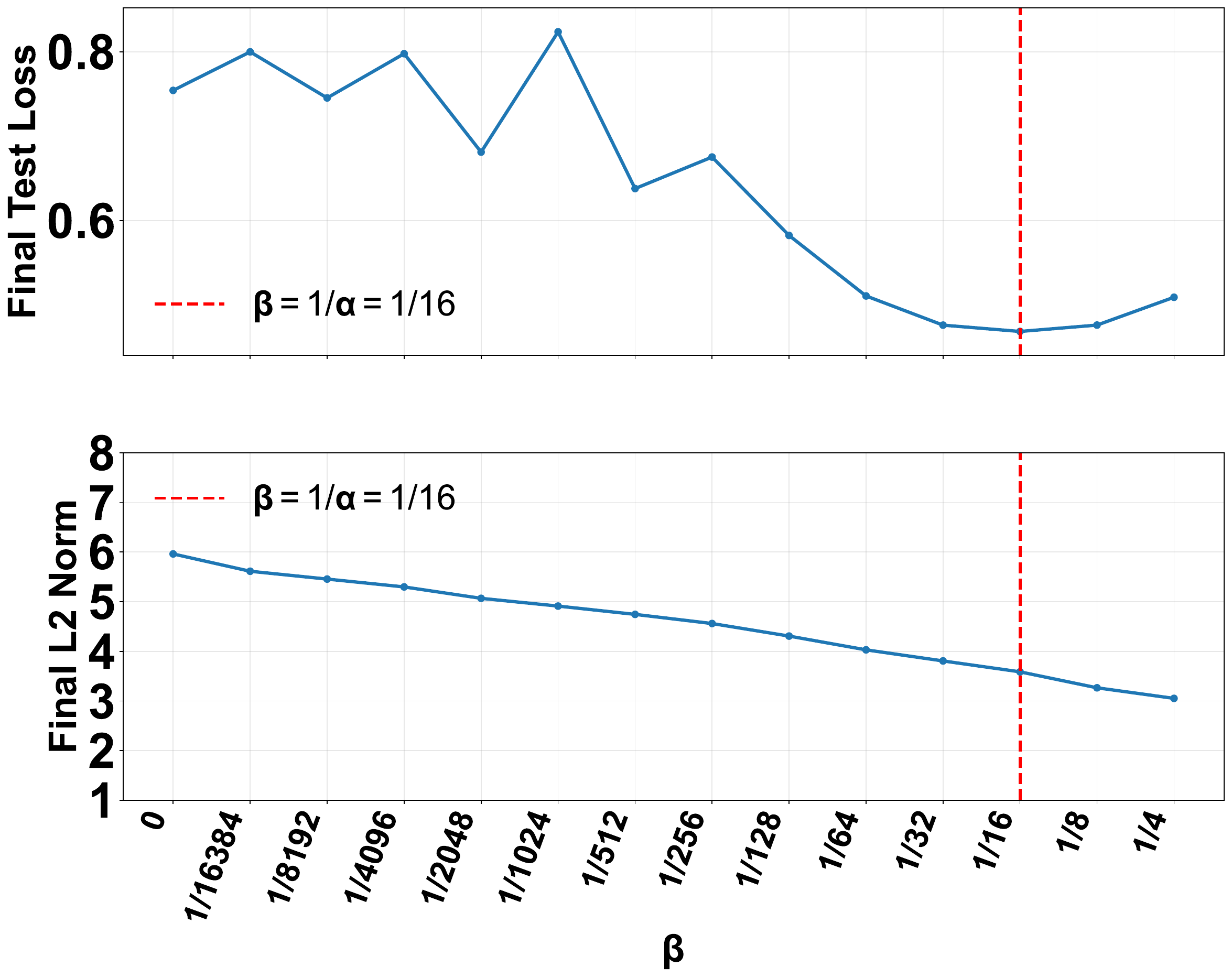} \\[-0.3em]

\includegraphics[width=0.46\textwidth]{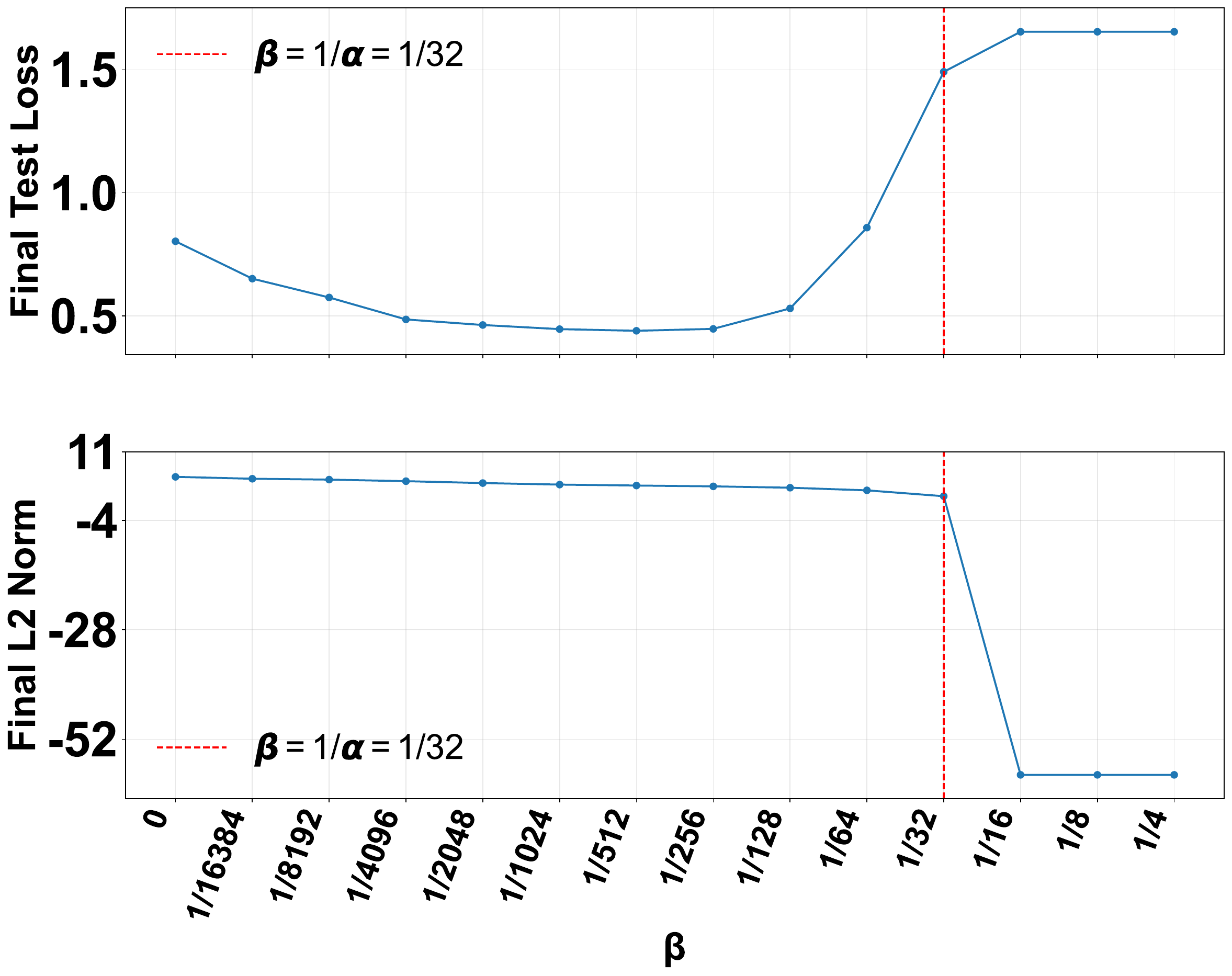} &
\includegraphics[width=0.46\textwidth]{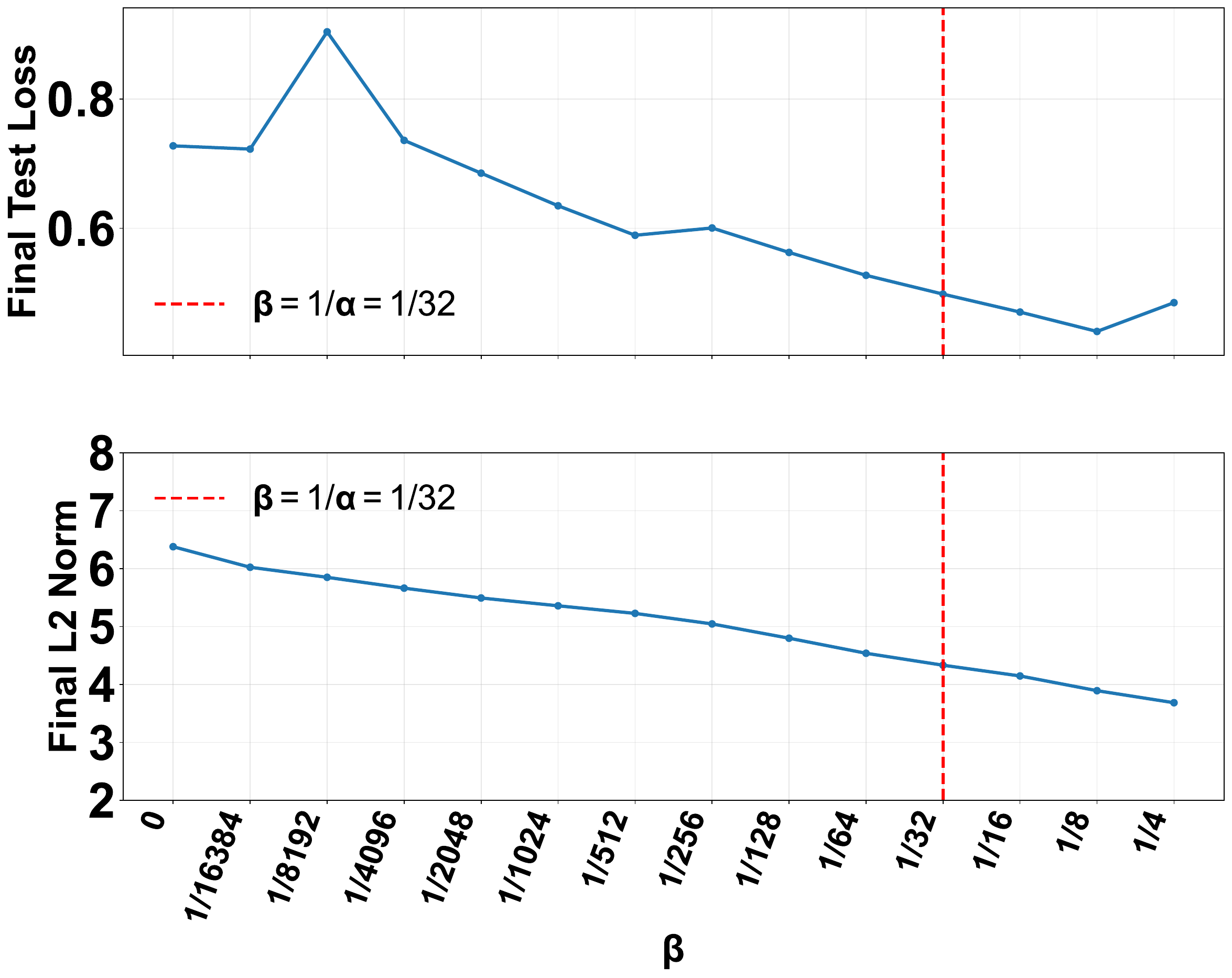}
\end{tabular}

\caption{Two-layer ReLU networks initialized by $\tau_1=\tau_2=0.01$ ($b_1=b_2=0$), scaled by $1/\alpha=1/\sqrt{m}~(a=0.5)$ (NTK regime) and trained for 40000 epochs on Yacht Hydrodynamics. Left column: SGD. Right column: AdamW. Rows show widths 64, 256, and 1024.}
\label{fig:yacht_widths_page1}
\end{figure*}

\begin{figure*}[p]
\centering
\setlength{\tabcolsep}{2pt}
\renewcommand{\arraystretch}{0.4}

\begin{tabular}{cc}
\includegraphics[width=0.46\textwidth]{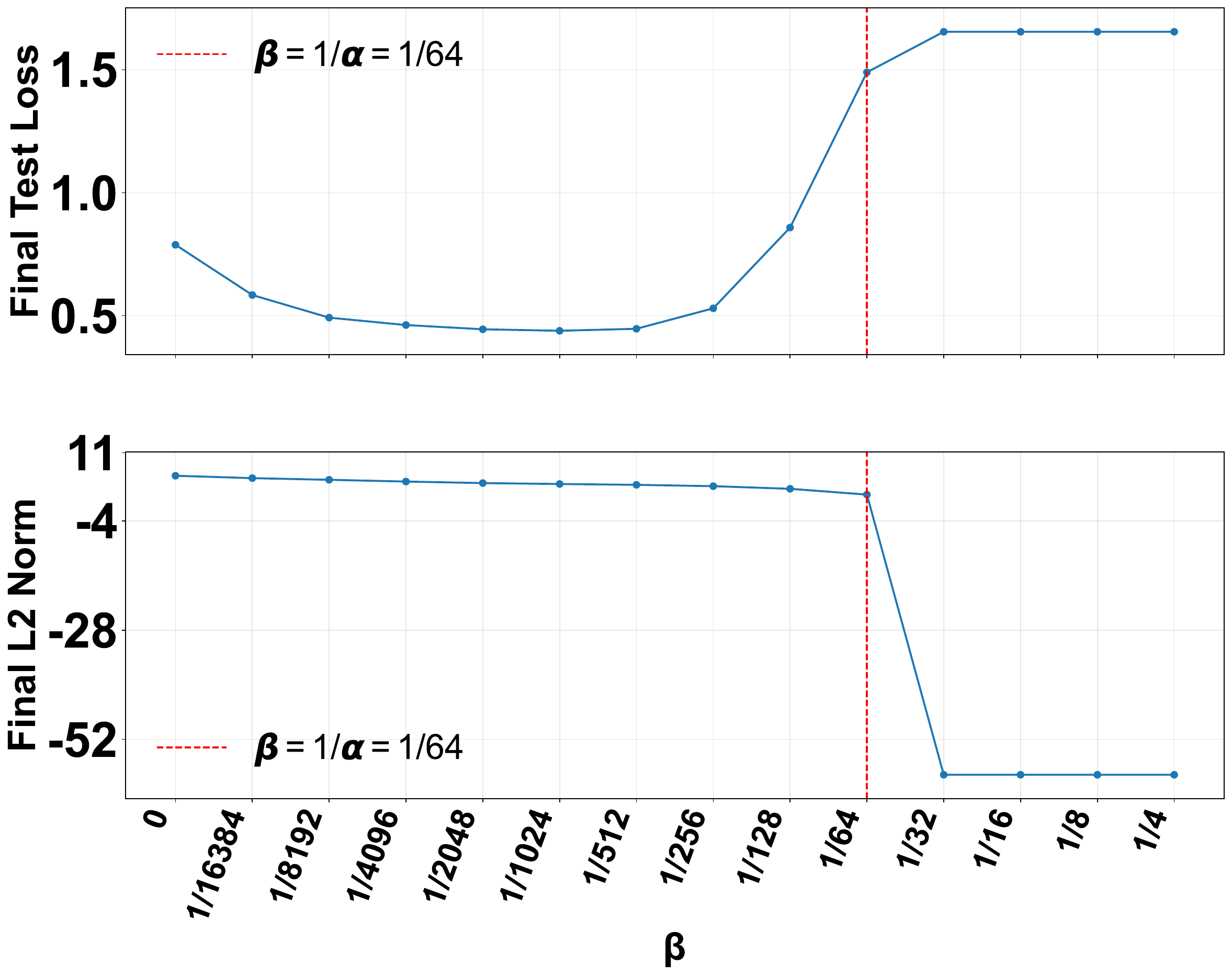} &
\includegraphics[width=0.46\textwidth]{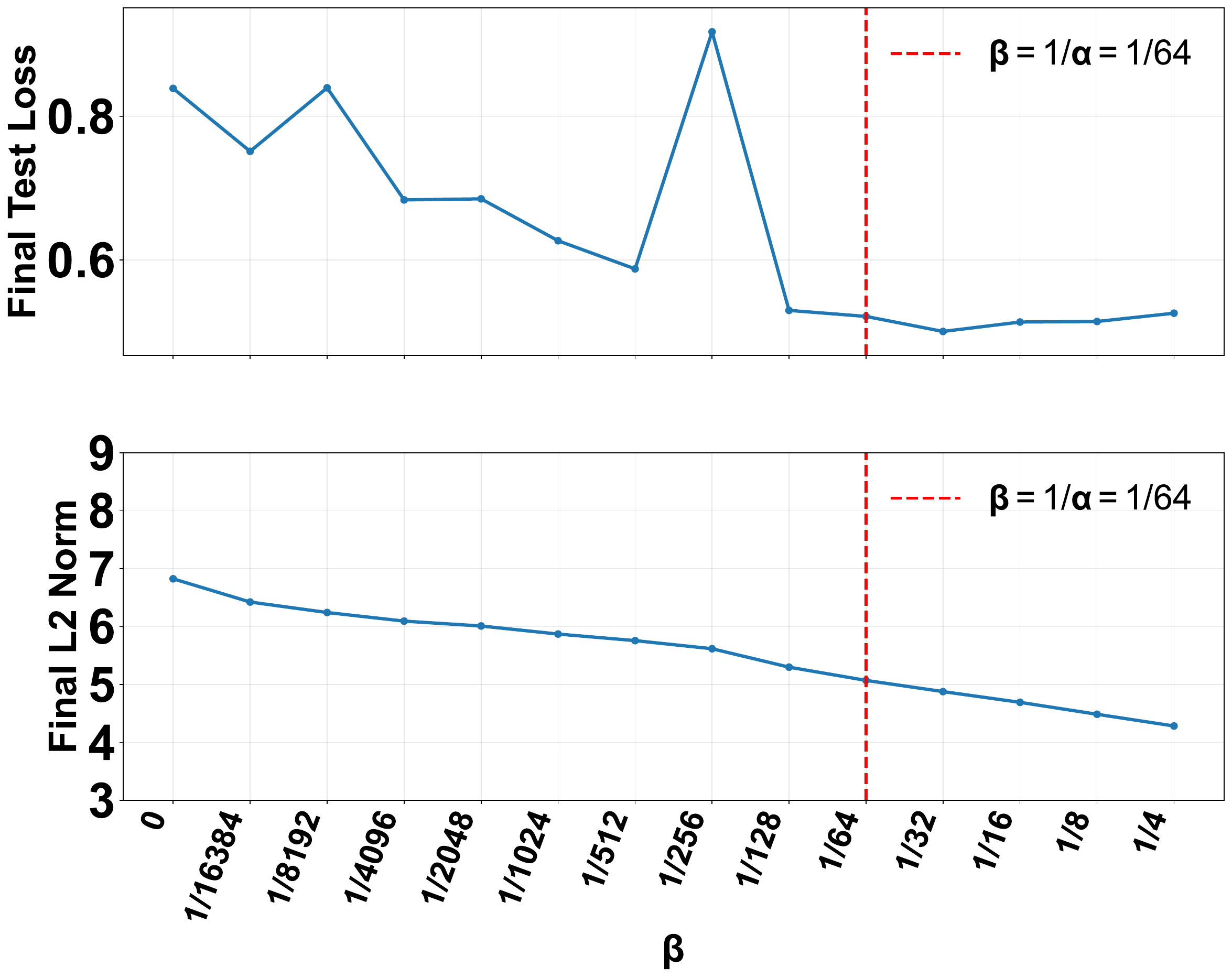} \\[-0.3em]

\includegraphics[width=0.46\textwidth]{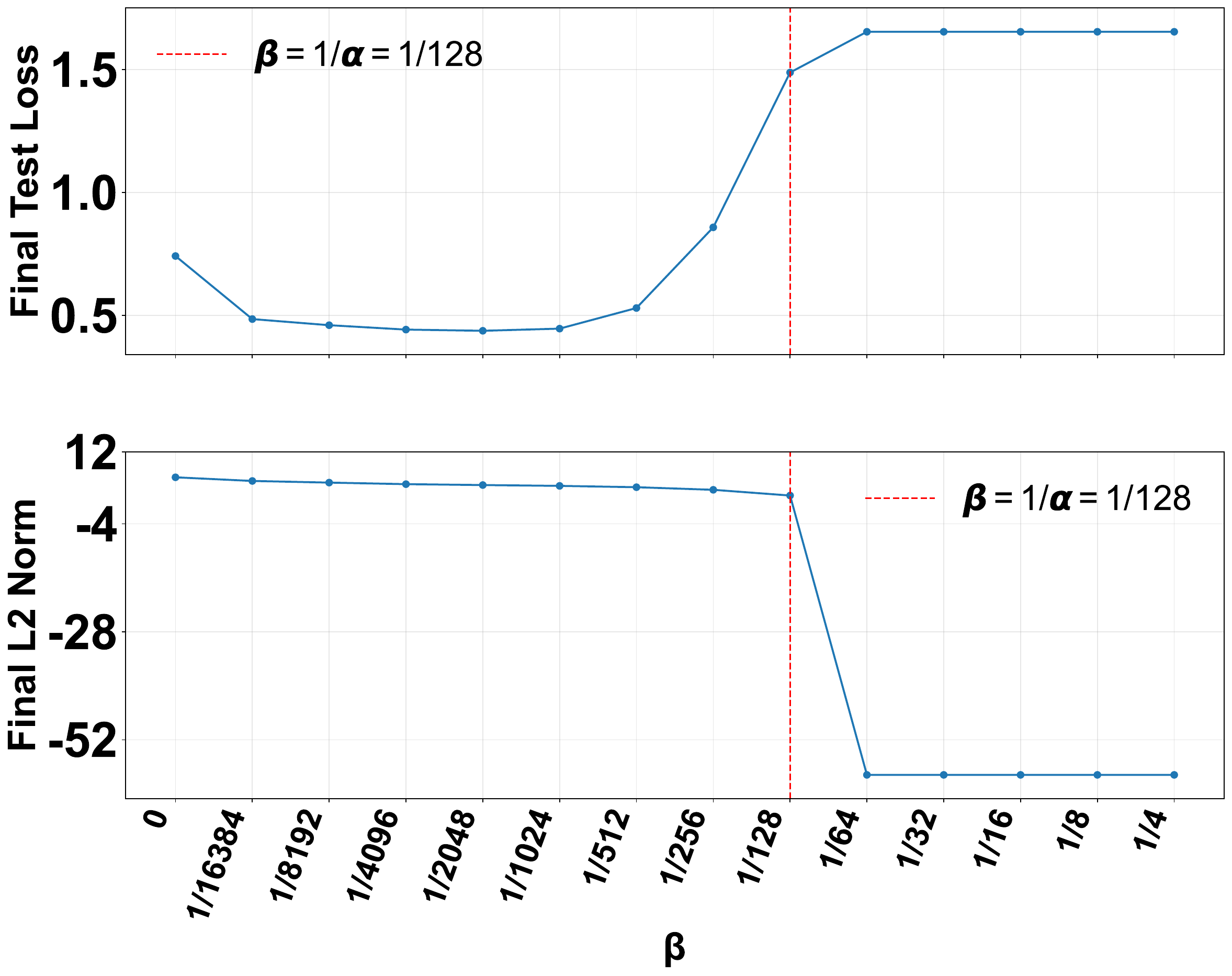} &
\includegraphics[width=0.46\textwidth]{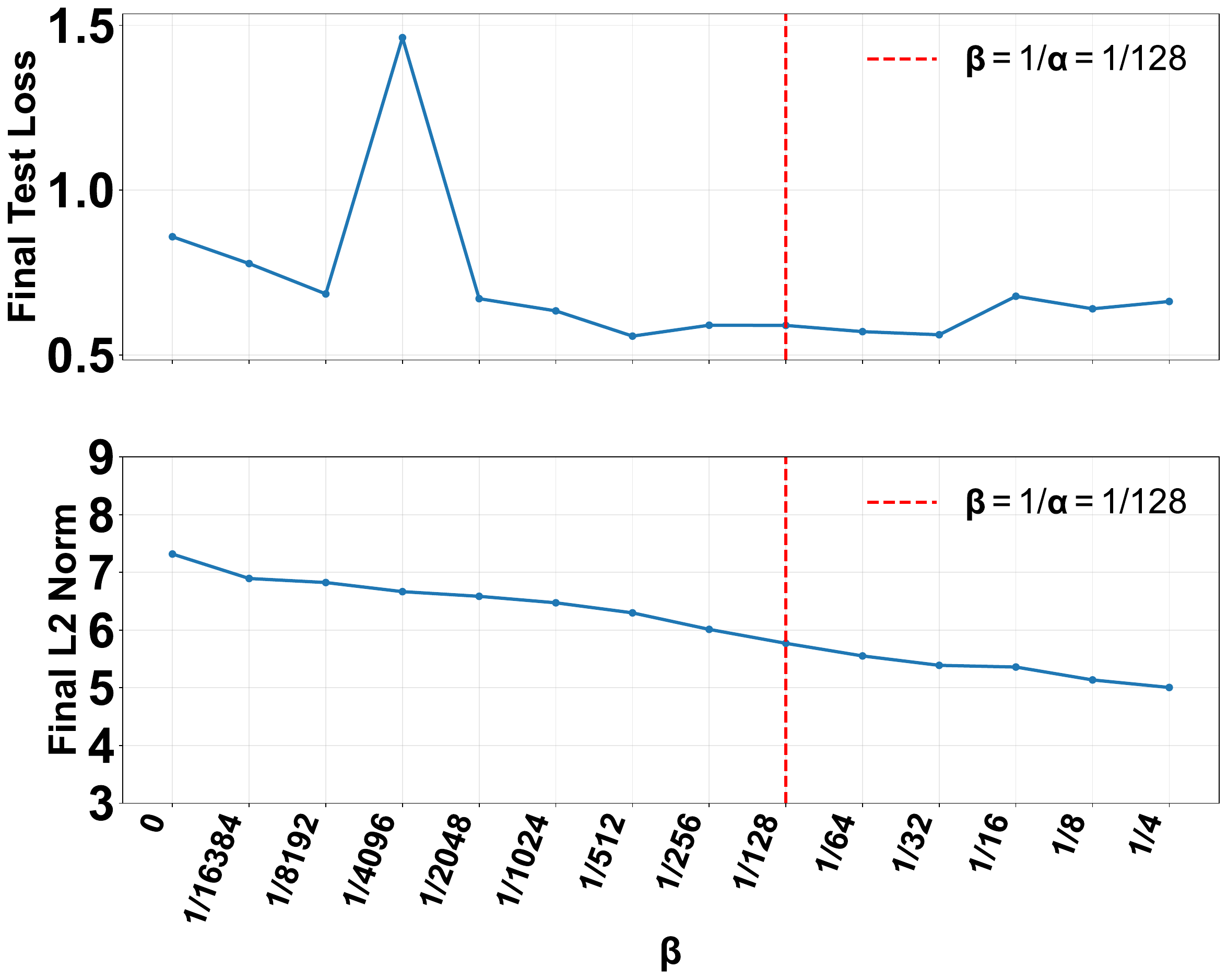}
\end{tabular}

\caption{Two-layer ReLU networks initialized by $\tau_1=\tau_2=0.01$ ($b_1=b_2=0$), scaled by $1/\alpha=1/\sqrt{m}~(a=0.5)$ (NTK regime) and trained for 40000 epochs on Yacht Hydrodynamics. Left column: SGD. Right column: AdamW. Rows show widths 4096 and 16384.}
\label{fig:yacht_widths_page2}
\end{figure*}

\begin{figure*}[p]
\centering
\setlength{\tabcolsep}{2pt}
\renewcommand{\arraystretch}{0.4}

\begin{tabular}{cc}
\includegraphics[width=0.46\textwidth]{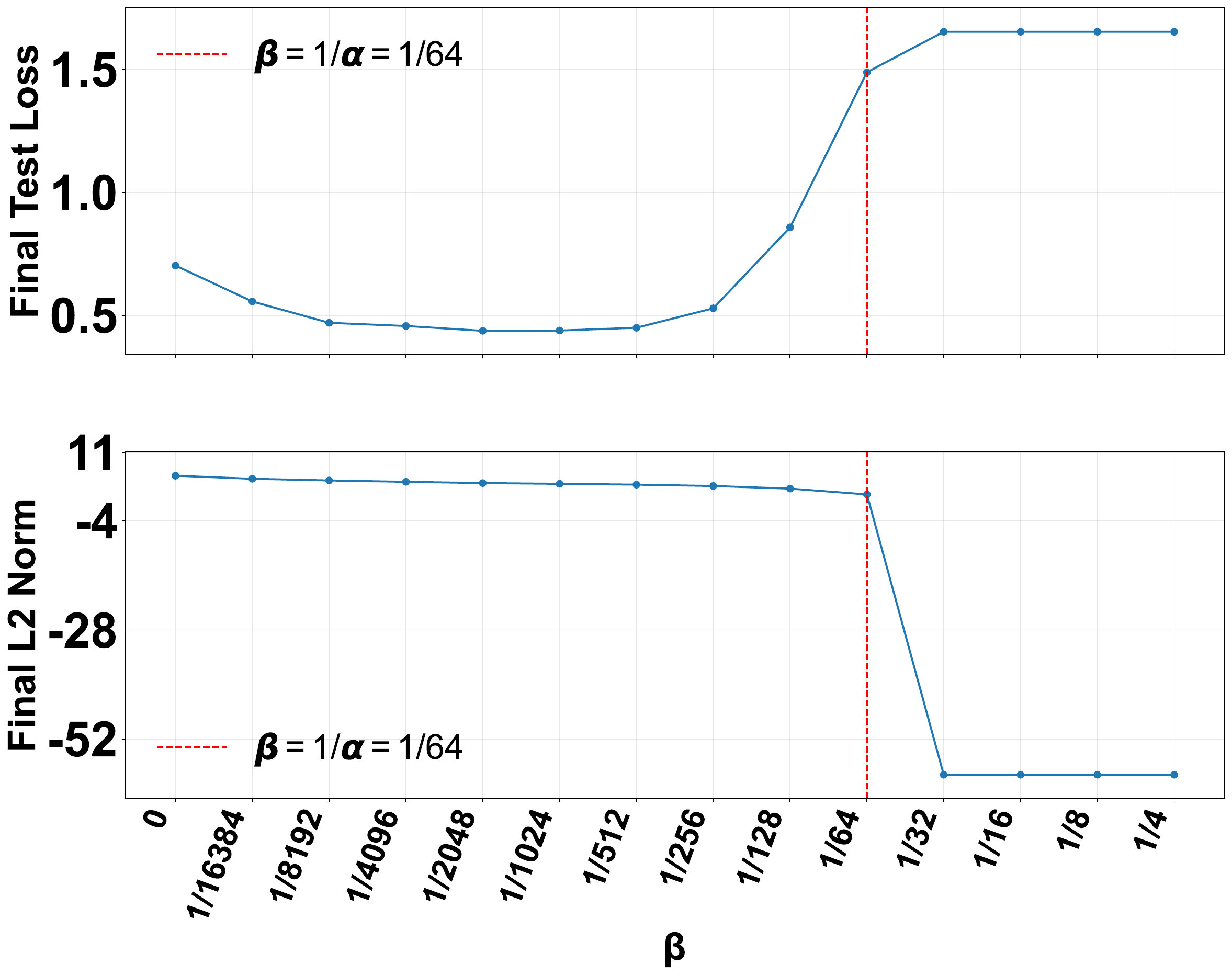} &
\includegraphics[width=0.46\textwidth]{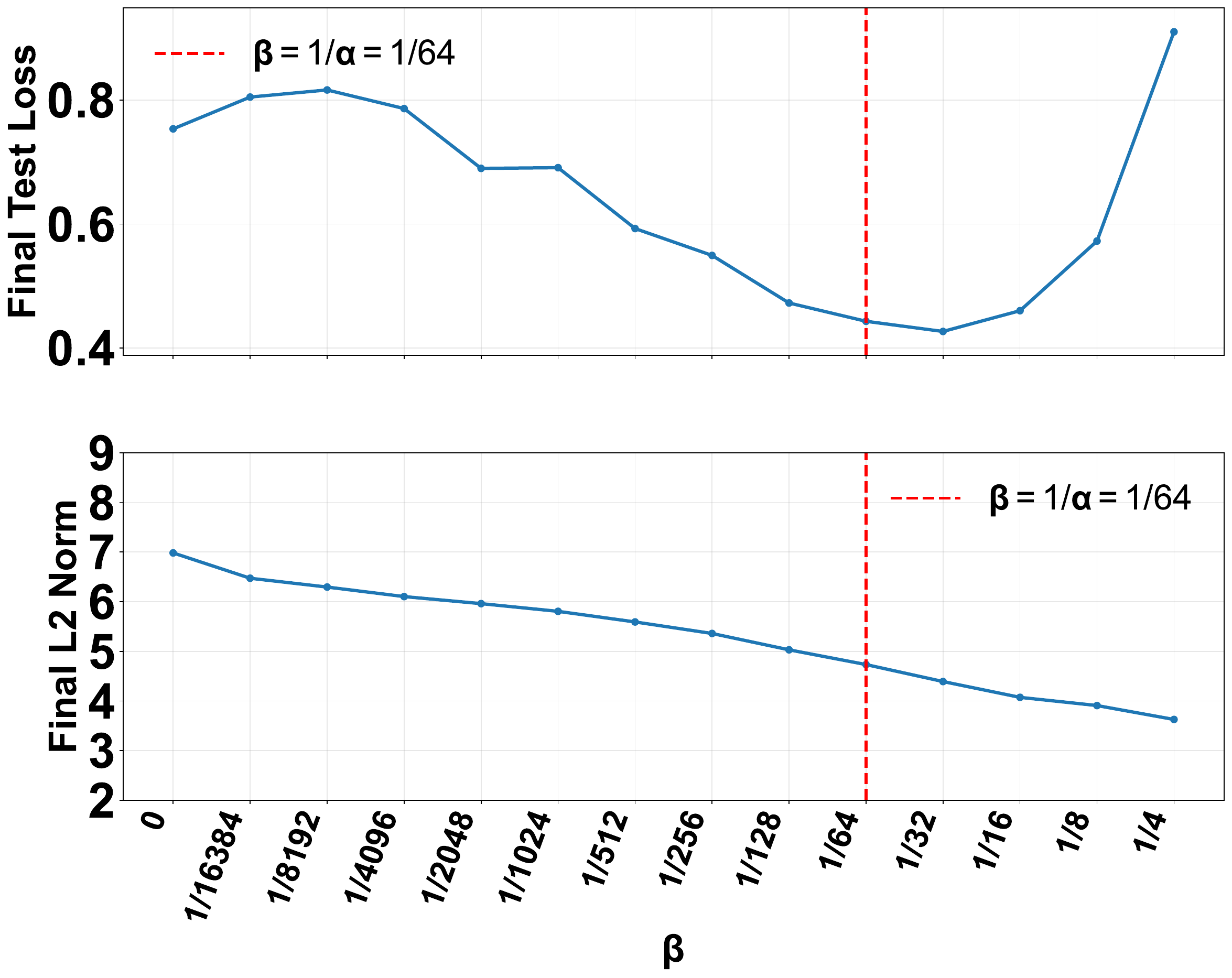} \\[-0.3em]

\includegraphics[width=0.46\textwidth]{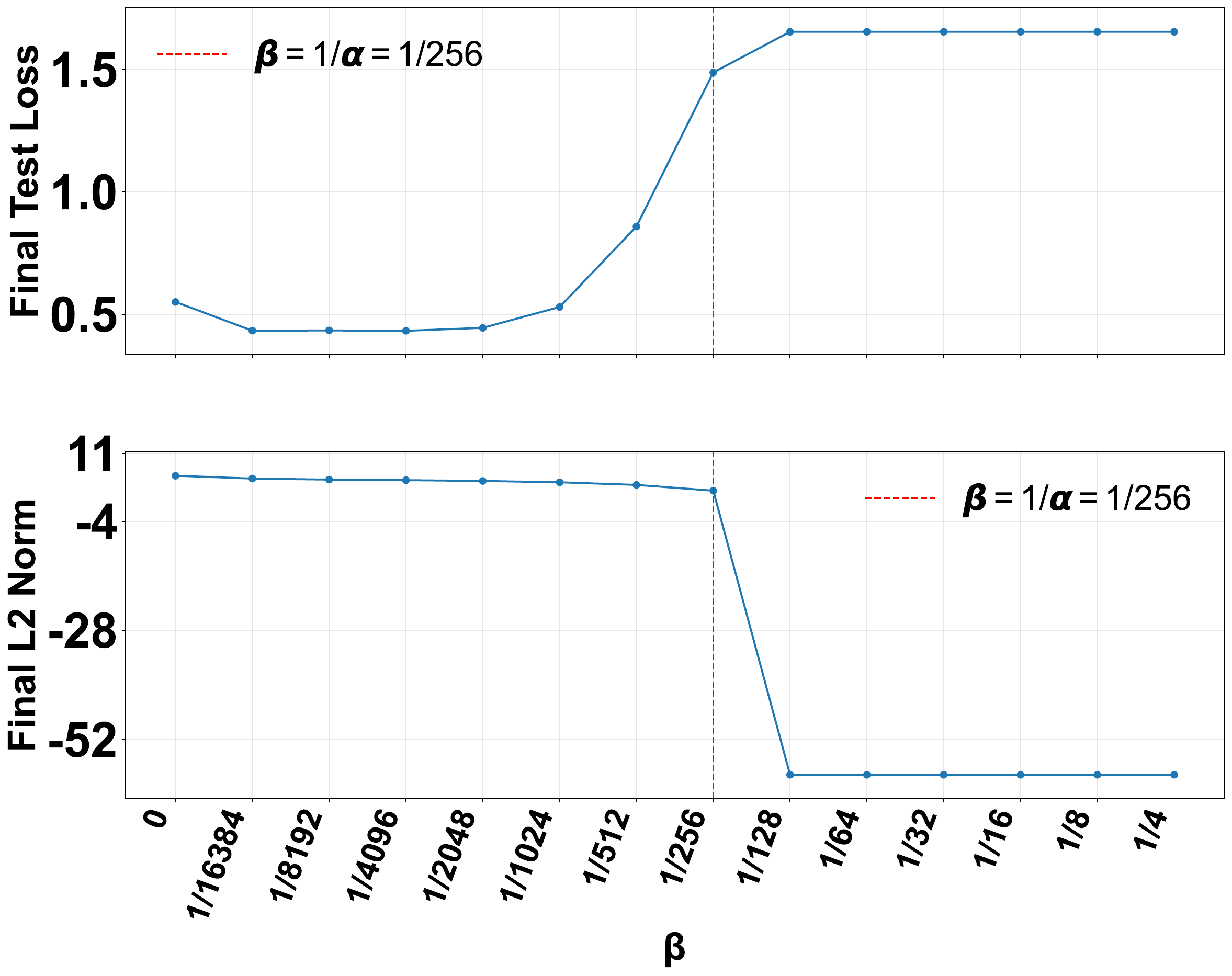} &
\includegraphics[width=0.46\textwidth]{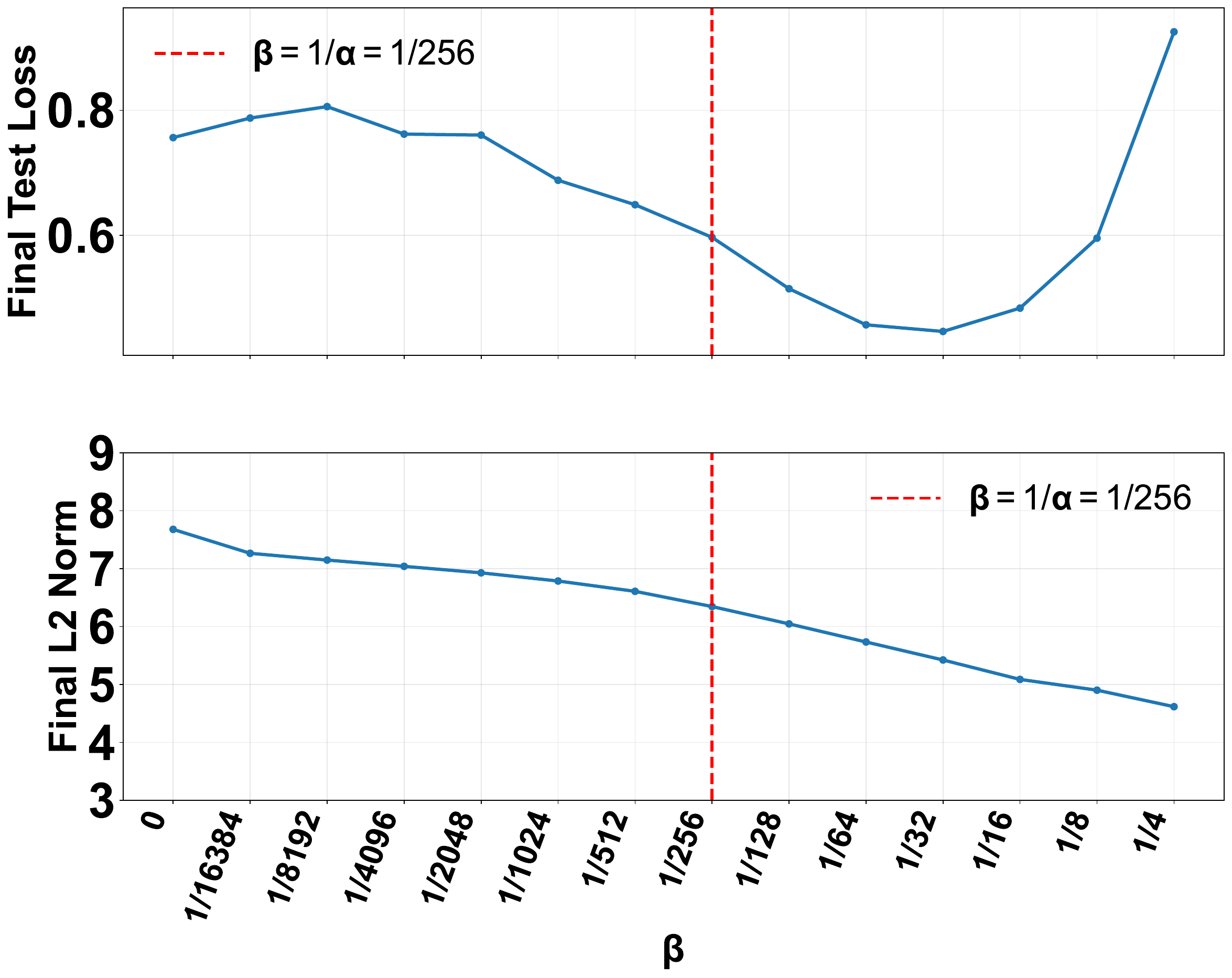} \\[-0.3em]

\includegraphics[width=0.46\textwidth]{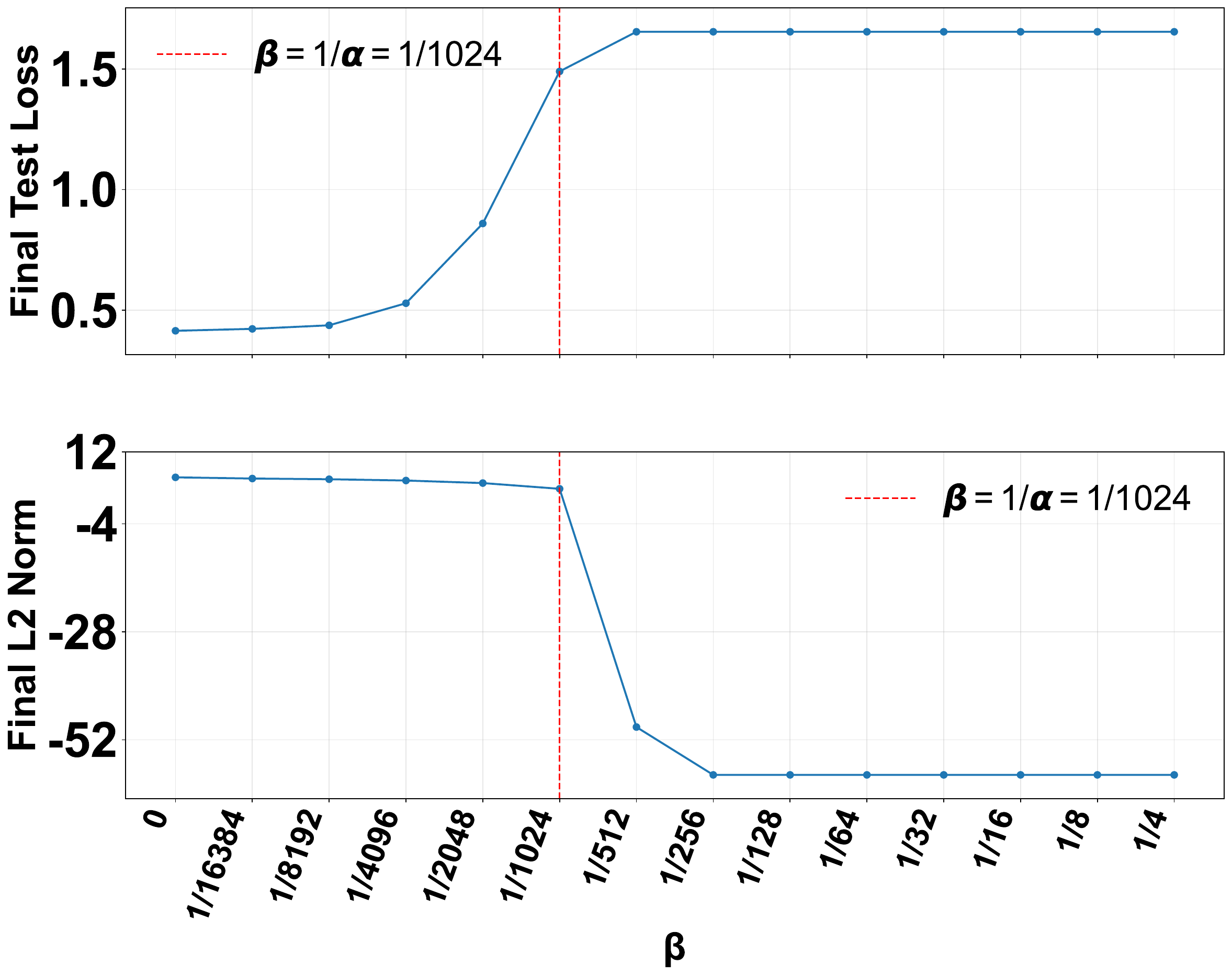} &
\includegraphics[width=0.46\textwidth]{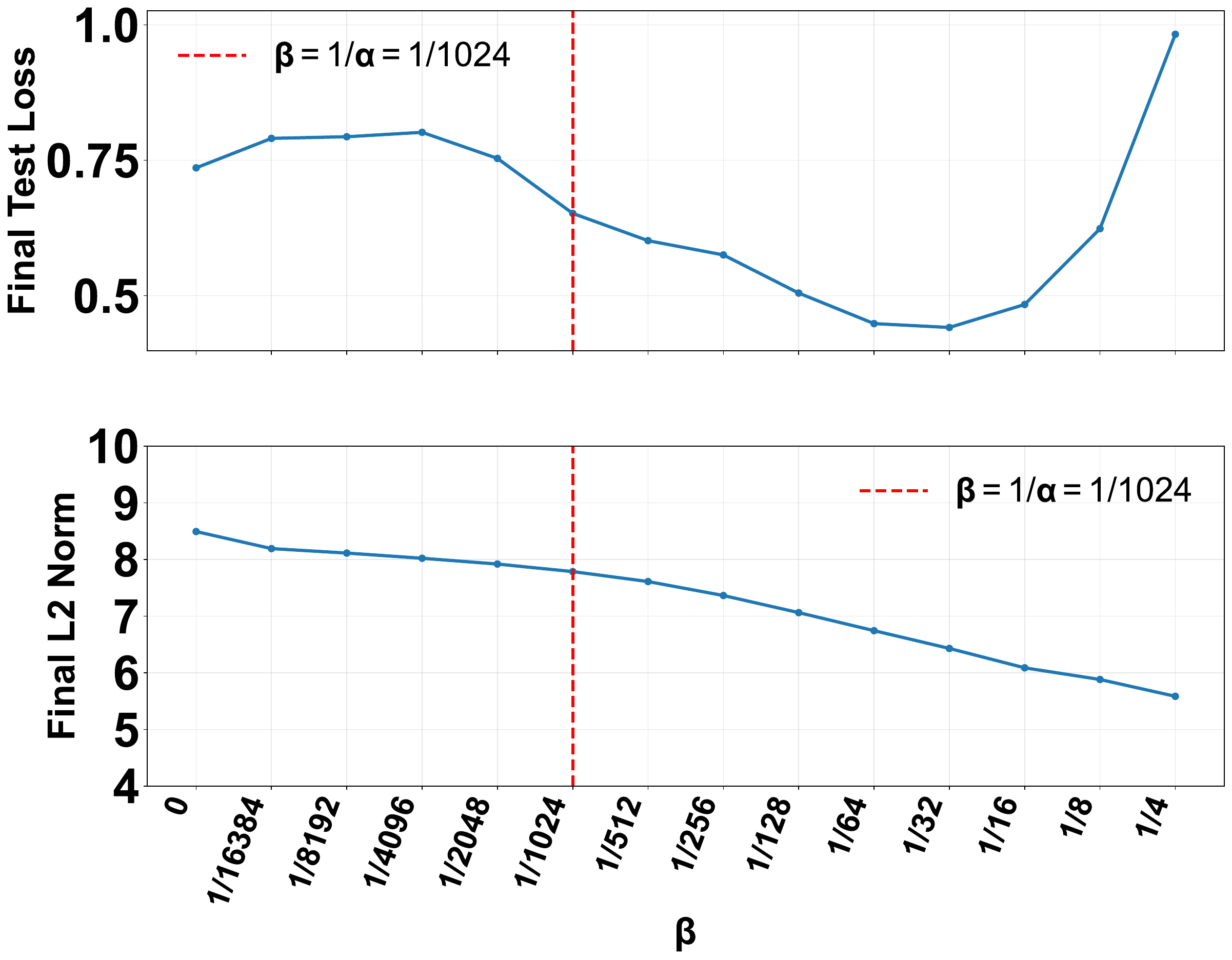}
\end{tabular}

\caption{Two-layer ReLU networks initialized by $\tau_1=\tau_2=0.01$ ($b_1=b_2=0$), scaled by $1/\alpha=1/m~(a=1)$ (mean field regime) and trained for 40000 epochs on Yacht Hydrodynamics. Left column: SGD. Right column: AdamW. Rows show widths 64, 256, and 1024.}
\label{fig:yacht_MF_page1}
\end{figure*}

\begin{figure*}[p]
\centering
\setlength{\tabcolsep}{2pt}
\renewcommand{\arraystretch}{0.4}

\begin{tabular}{cc}
\includegraphics[width=0.46\textwidth]{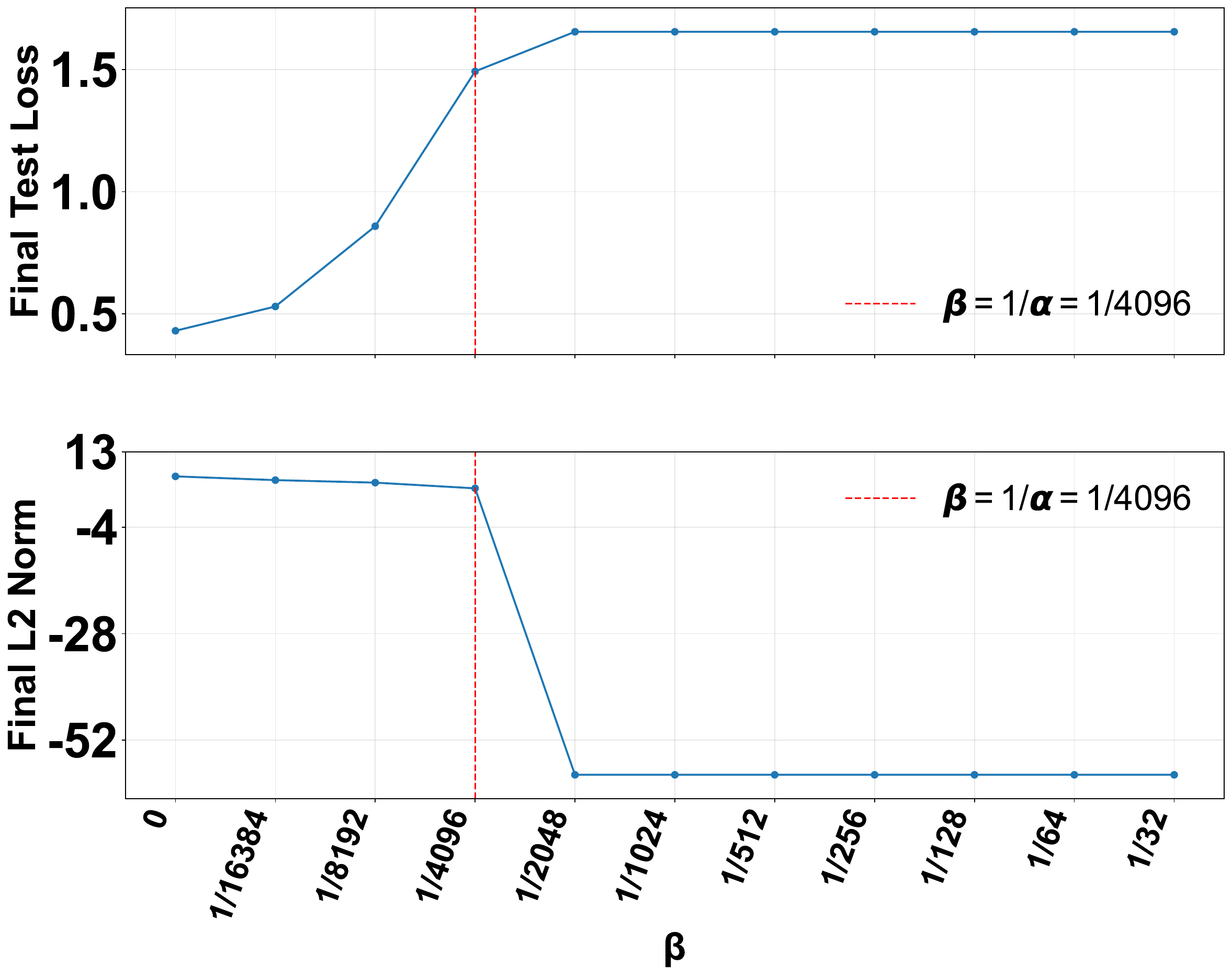} &
\includegraphics[width=0.46\textwidth]{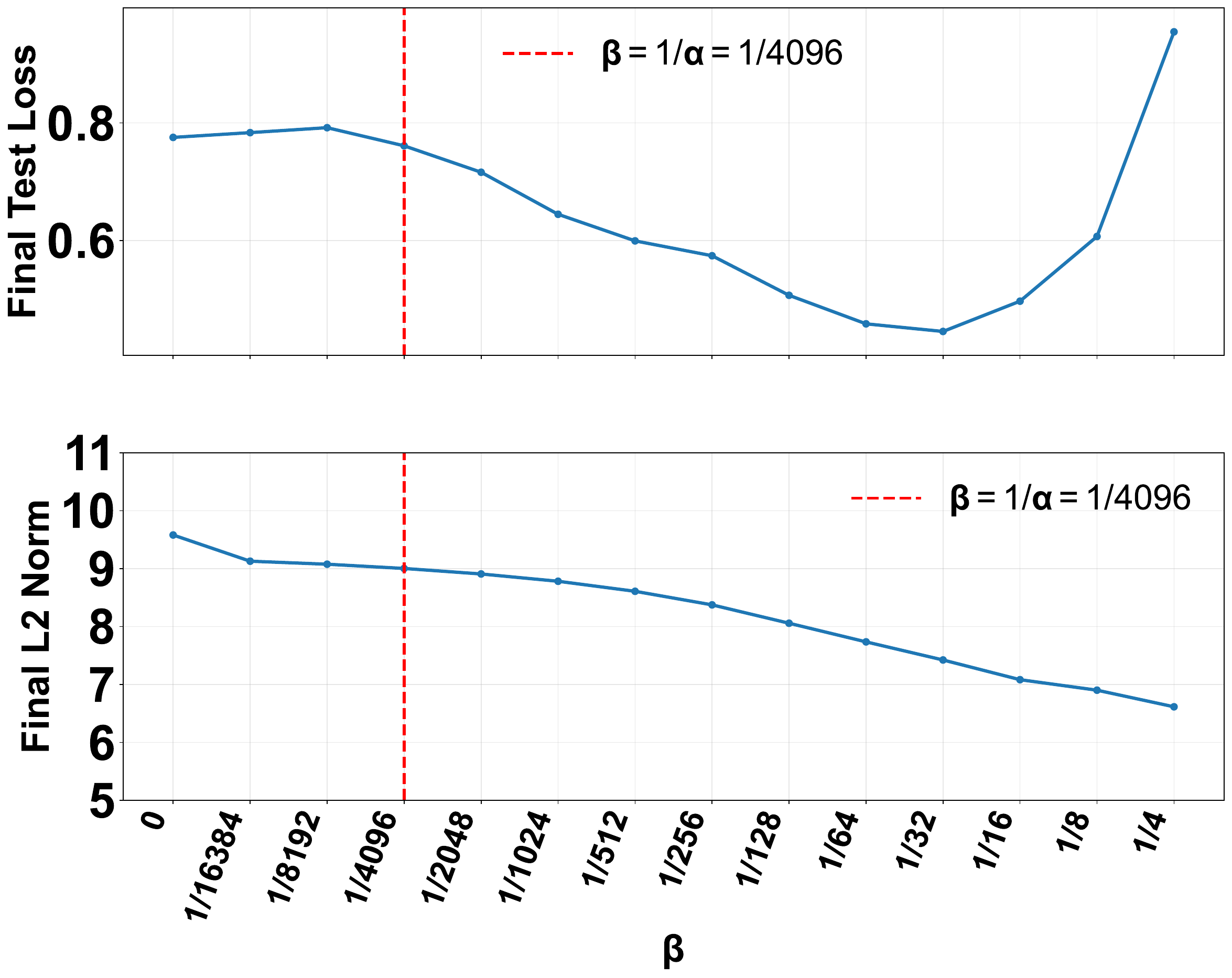} \\[-0.3em]

\includegraphics[width=0.46\textwidth]{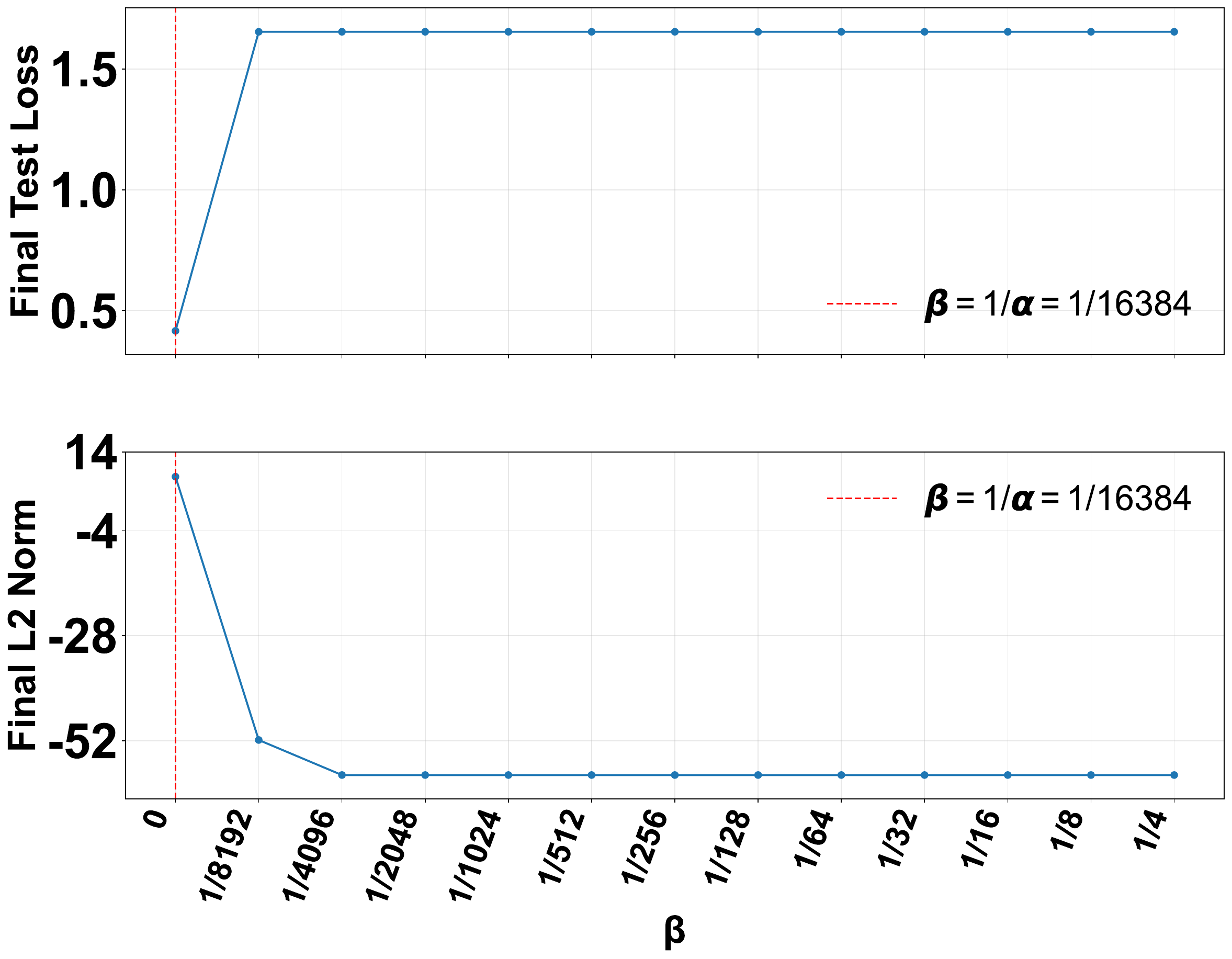} &
\includegraphics[width=0.46\textwidth]{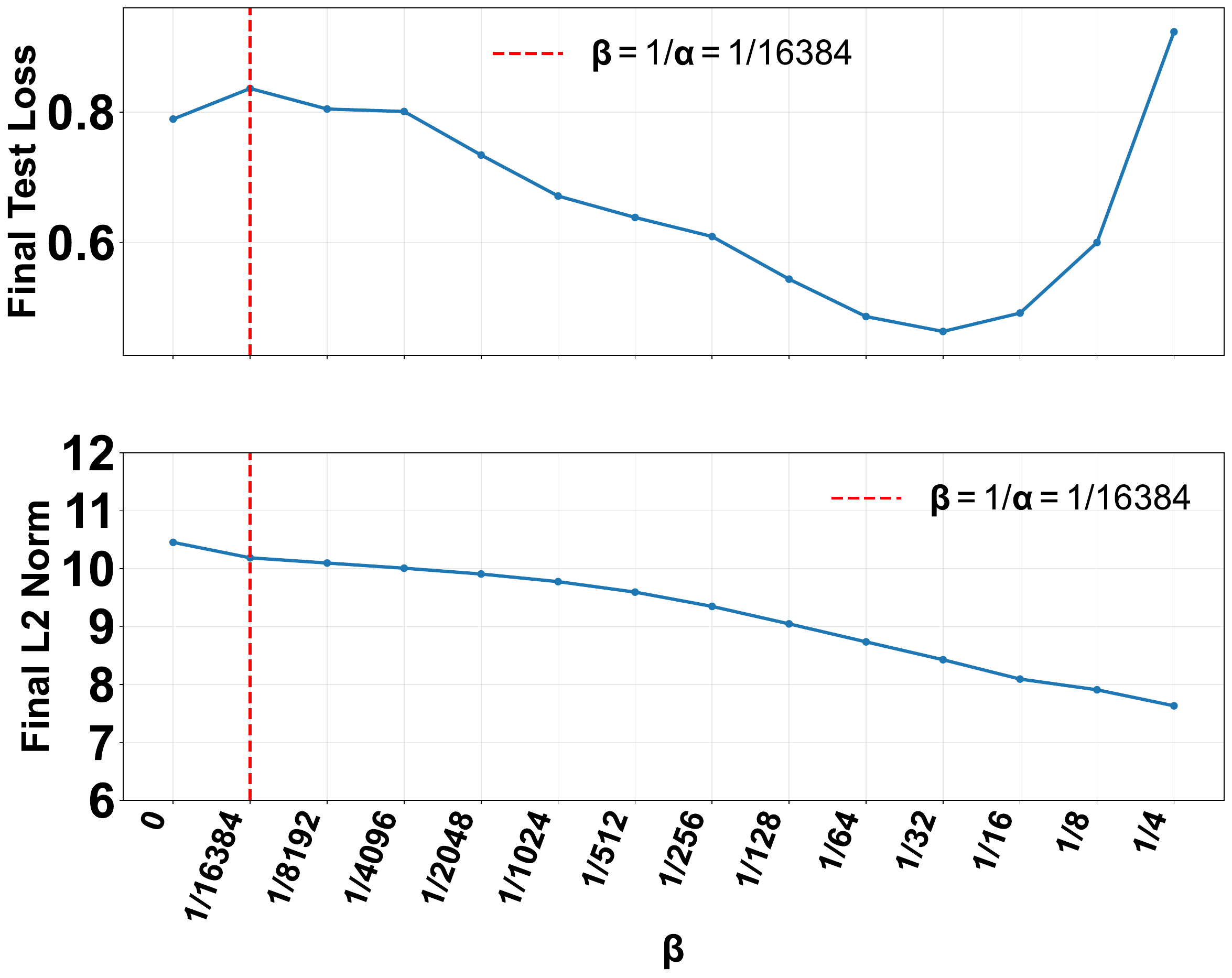}
\end{tabular}

\caption{Two-layer ReLU networks initialized by $\tau_1=\tau_2=0.01$ ($b_1=b_2=0$), scaled by $1/\alpha=1/m~(a=1)$ (mean field regime) and trained on Yacht Hydrodynamics. Left column: SGD. Right column: AdamW. Rows show widths 4096 (240000 epochs) and 16384 (700000 epochs).}
\label{fig:yacht_MF_page2}
\end{figure*}

\begin{figure*}[p]
\centering
\setlength{\tabcolsep}{2pt}
\renewcommand{\arraystretch}{0.4}

\begin{tabular}{cc}
\includegraphics[width=0.46\textwidth]{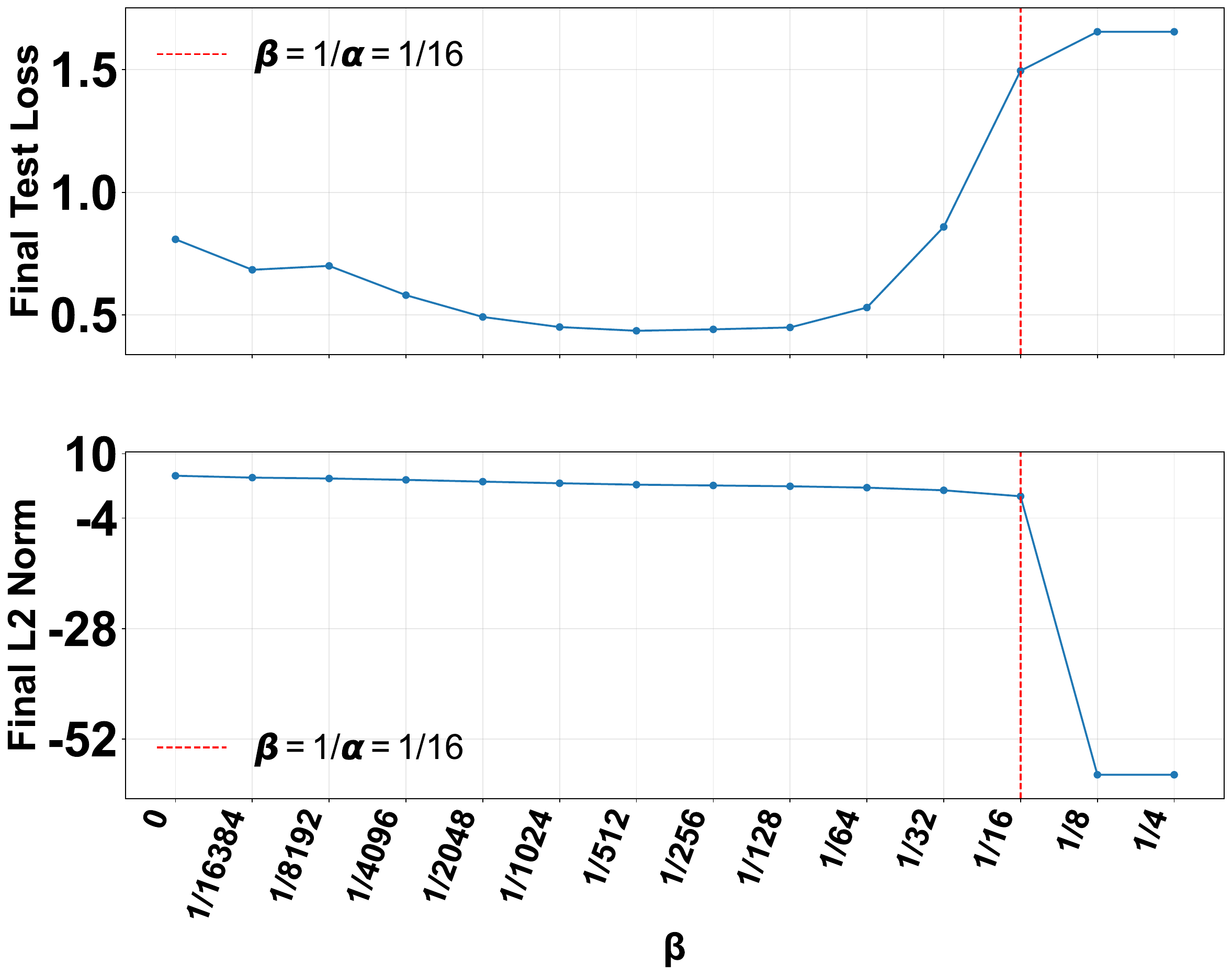} &
\includegraphics[width=0.46\textwidth]{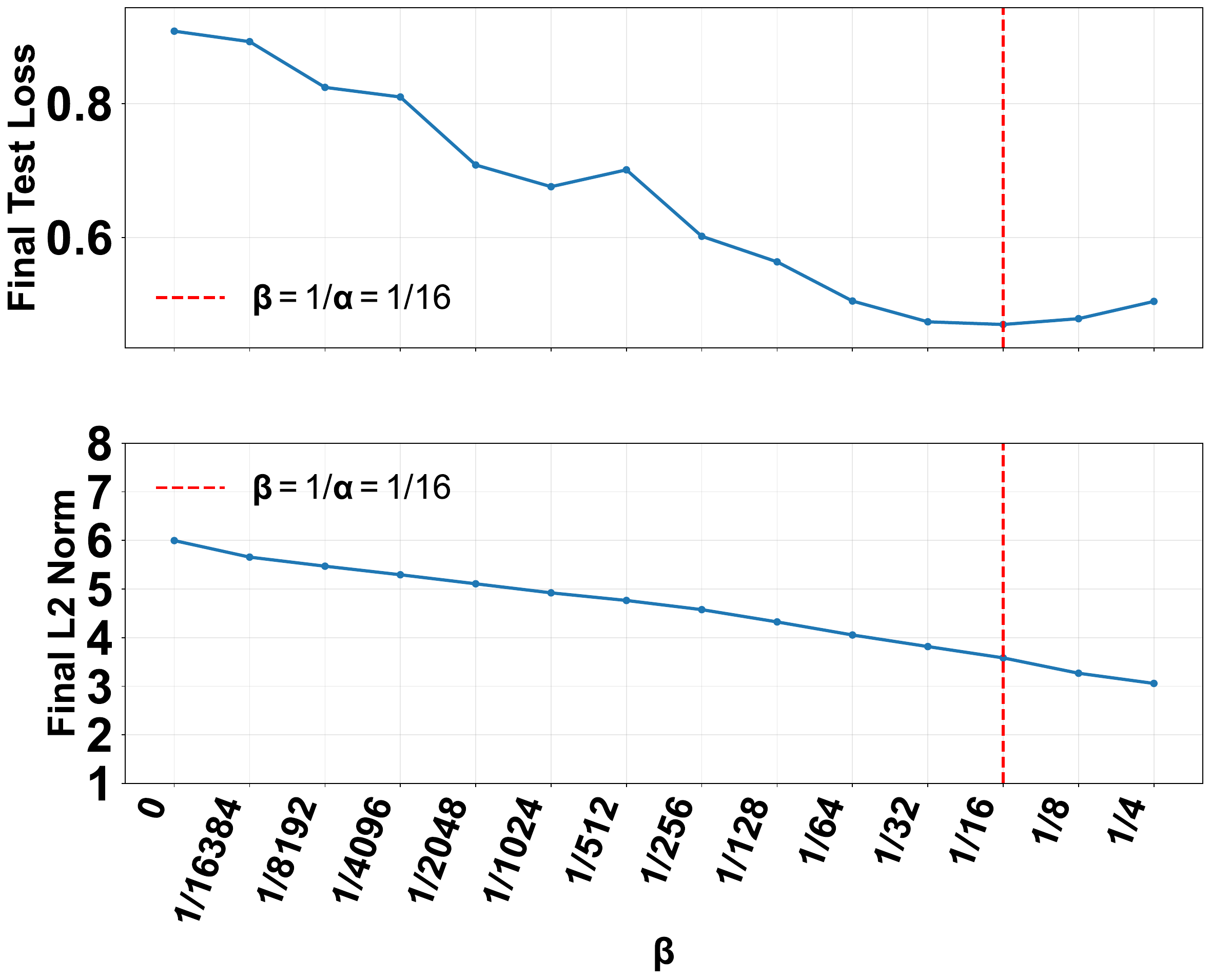} \\[-0.3em]

\includegraphics[width=0.46\textwidth]{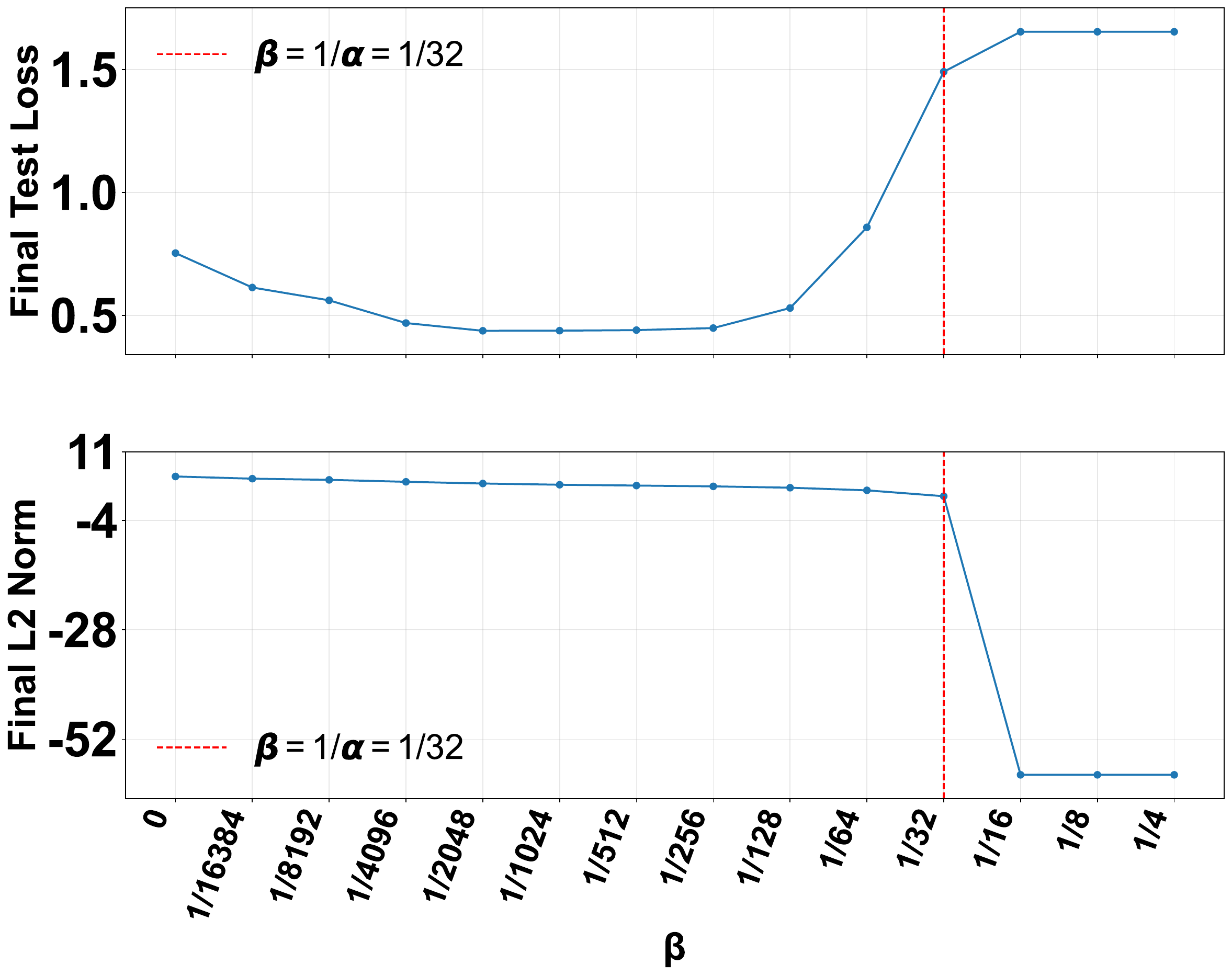} &
\includegraphics[width=0.46\textwidth]{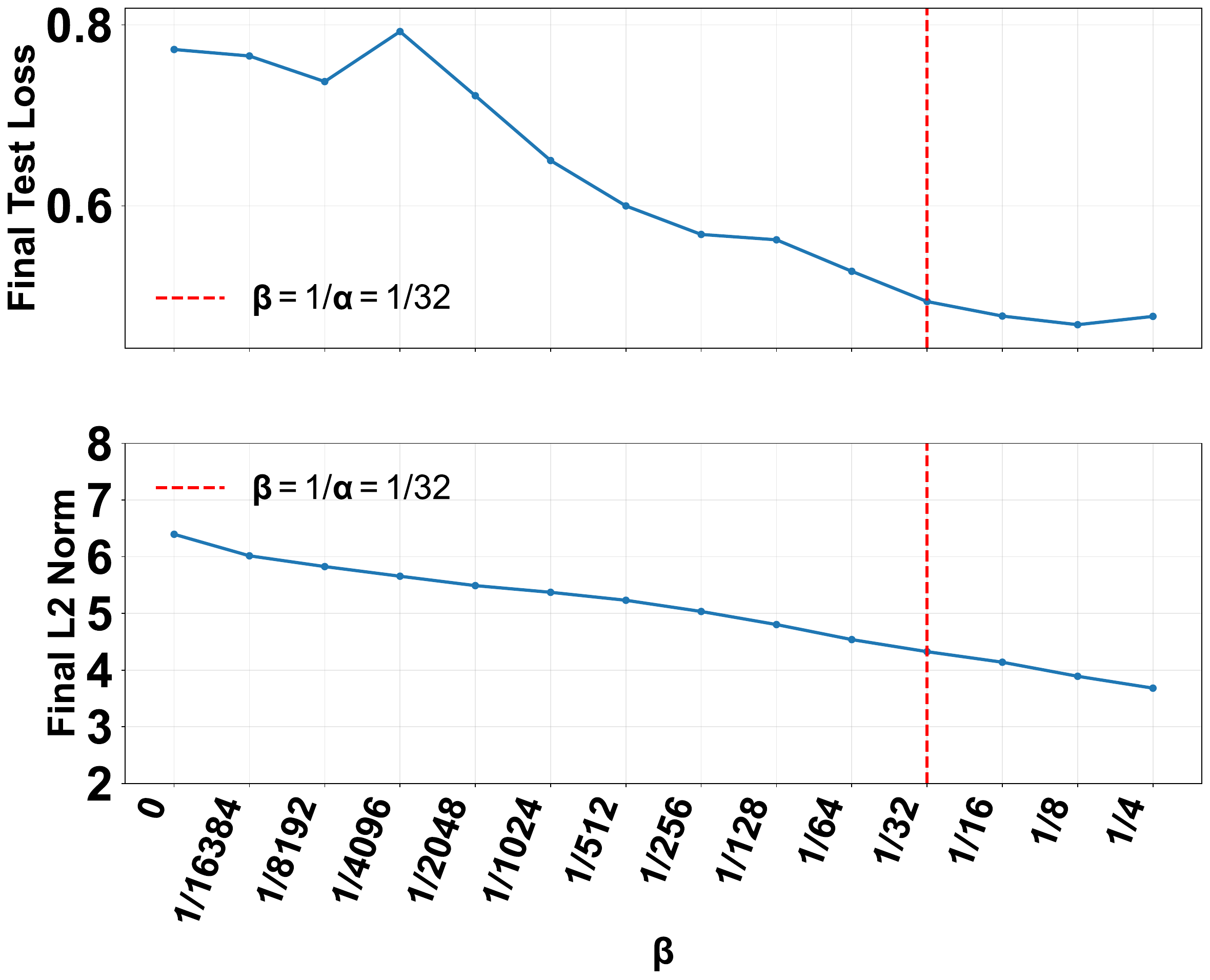}
\end{tabular}

\caption{Two-layer ReLU networks initialized by $\tau_1=\tau_2=0.01/\sqrt{m}$ ($b_1=b_2=0.5$), scaled by $1/\alpha=1/\sqrt{m}~(a=0.5)$ and trained for 40000 epochs on Yacht Hydrodynamics. Left column: SGD. Right column: AdamW. Rows show widths 256 and 1024.}
\label{fig:yacht_sv1_NTK}
\end{figure*}

\begin{figure*}[p]
\centering
\setlength{\tabcolsep}{2pt}
\renewcommand{\arraystretch}{0.4}

\begin{tabular}{cc}
\includegraphics[width=0.46\textwidth]{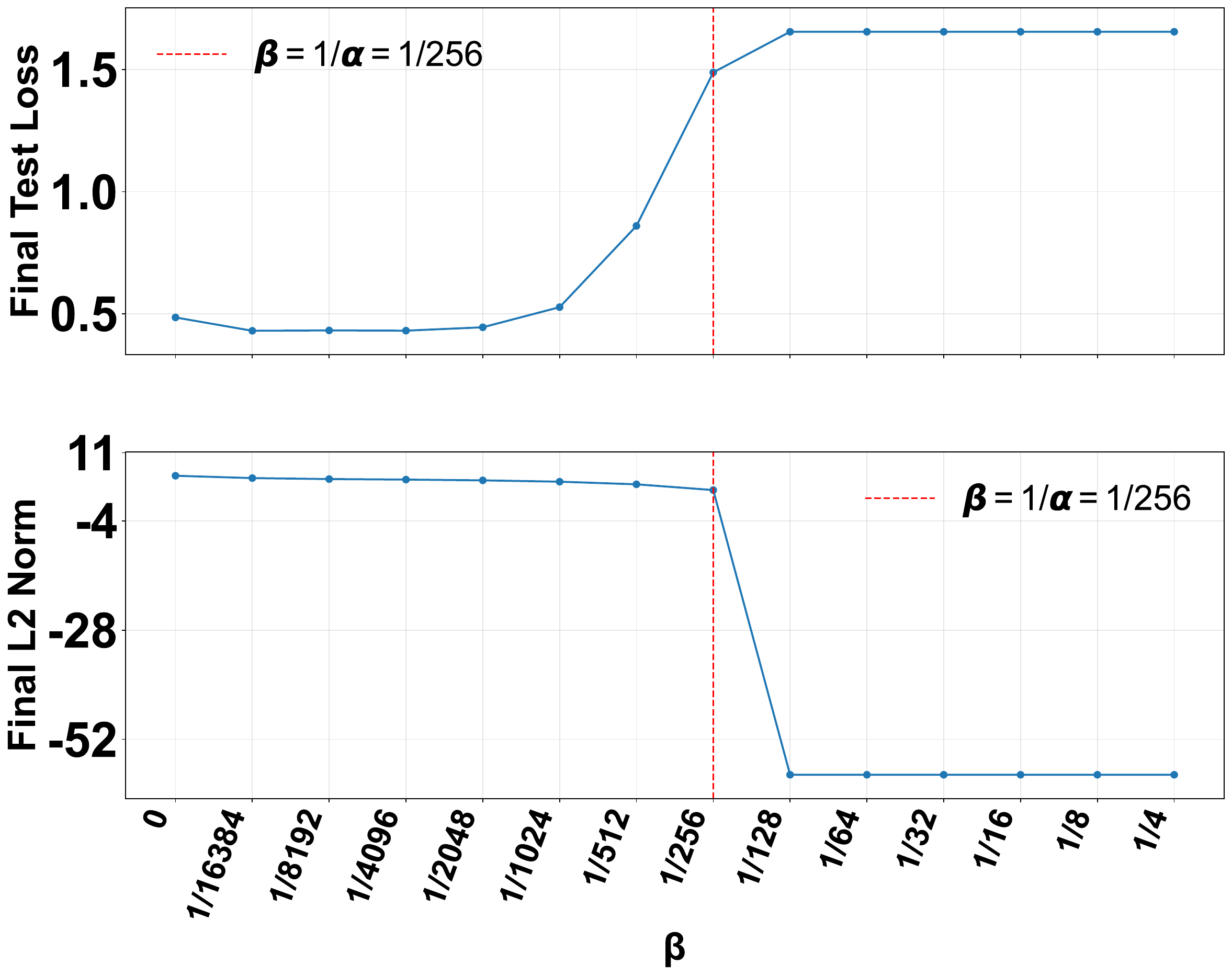} &
\includegraphics[width=0.46\textwidth]{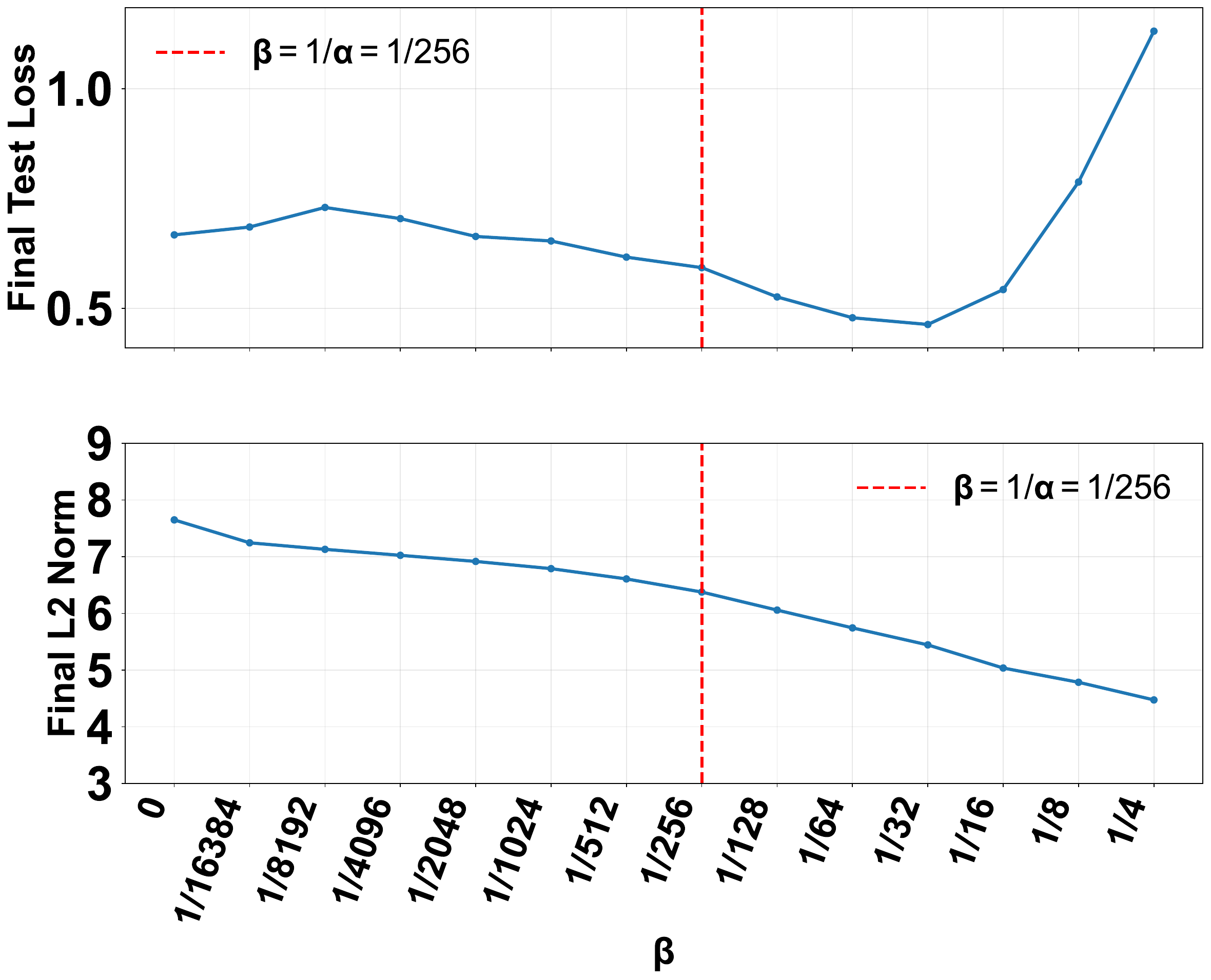} \\[-0.3em]

\includegraphics[width=0.46\textwidth]{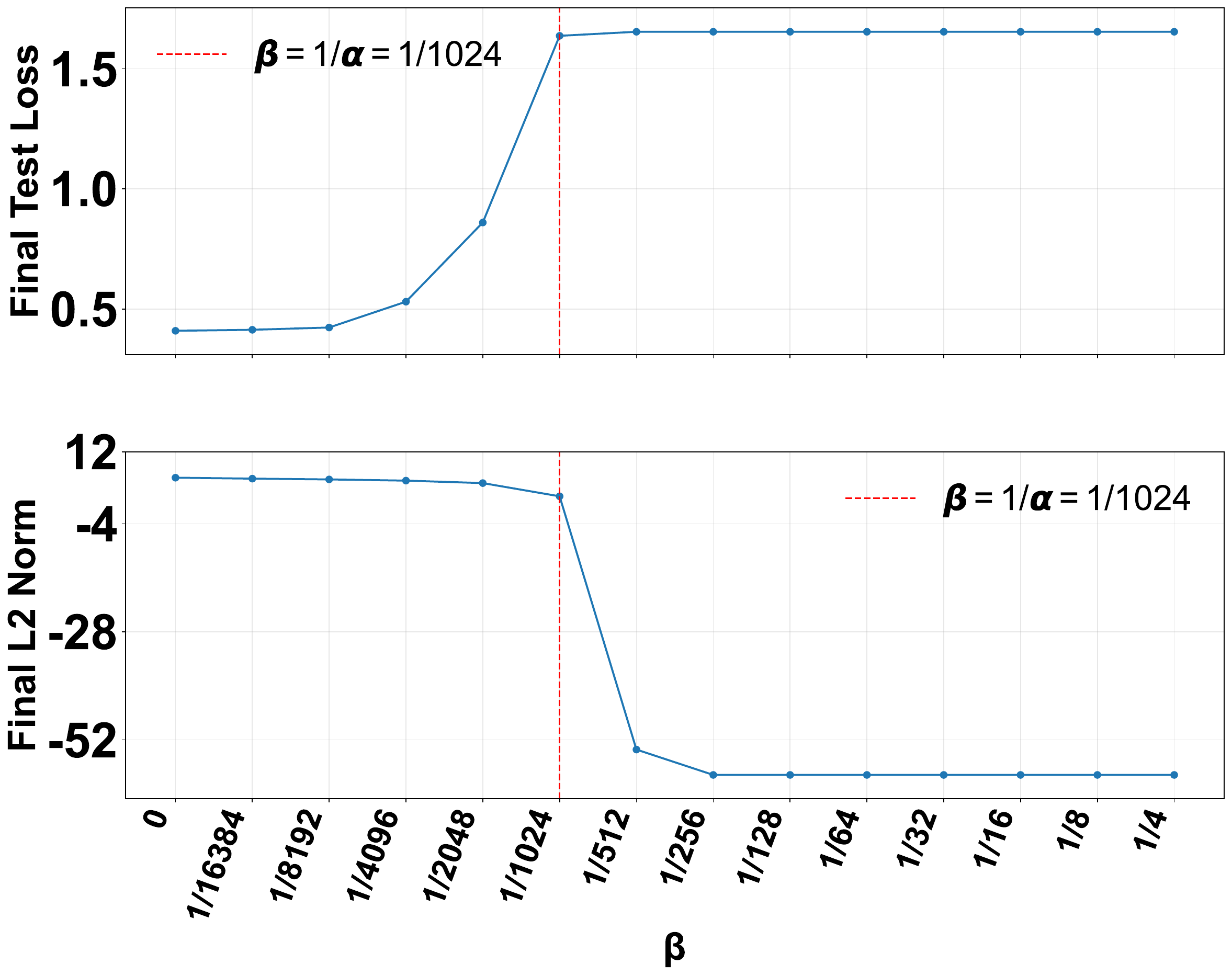} &
\includegraphics[width=0.46\textwidth]{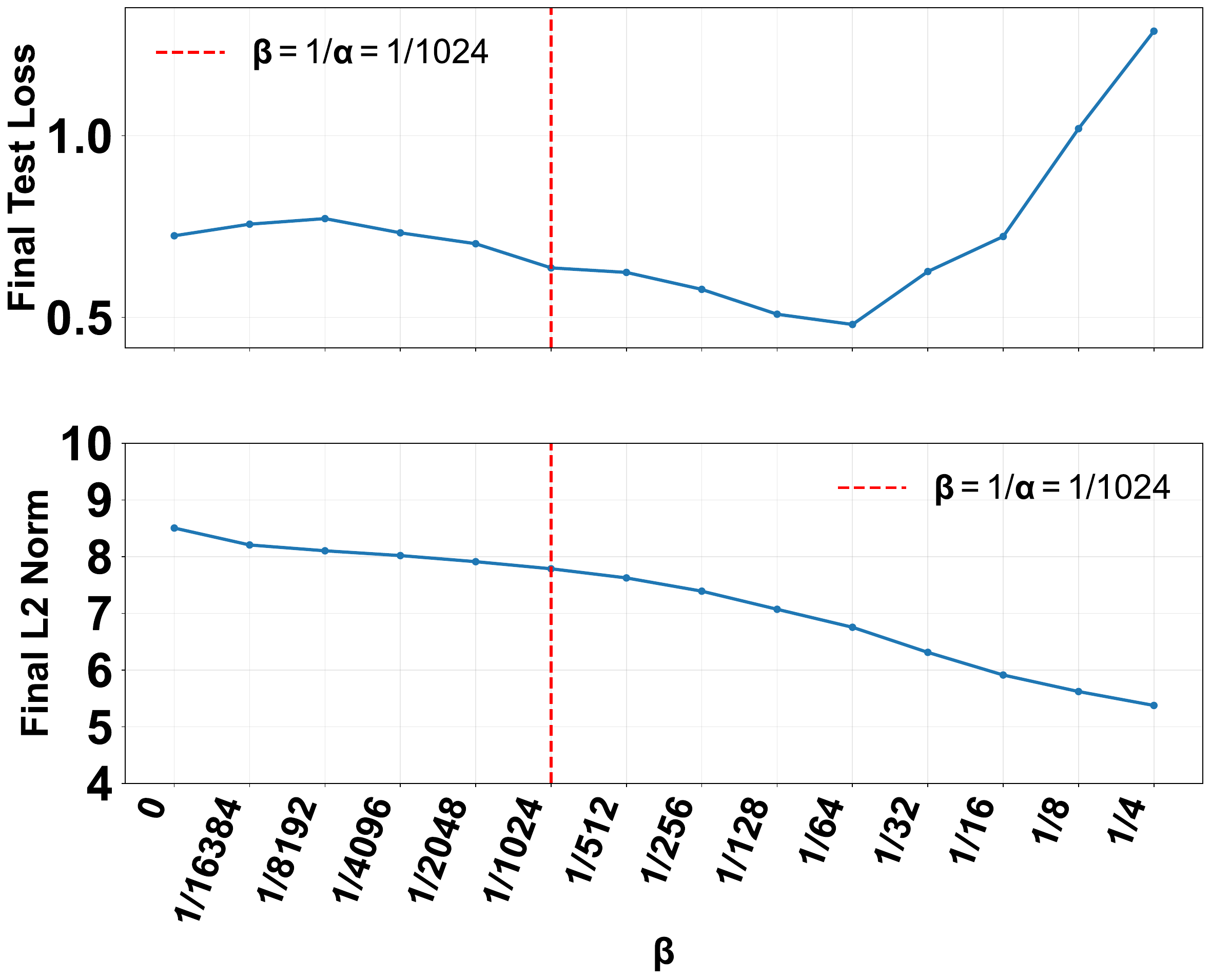}
\end{tabular}

\caption{Two-layer ReLU networks initialized by $\tau_1=\tau_2=0.01/\sqrt{m}$ ($b_1=b_2=0.5$), scaled by $1/\alpha=1/m~(a=1)$ and trained for 40000 epochs on Yacht Hydrodynamics. Left column: SGD. Right column: AdamW. Rows show widths 256 and 1024.}
\label{fig:yacht_sv1_MF}
\end{figure*}

\subsection{MNIST}
\begin{itemize}
    \item Figures~\ref{fig:MNIST_NTK_page1} and Figure \ref{fig:MNIST_NTK_page2} show results for networks with scaling and initialization in the neural tangent kernel setting \citep{jacot_neural_2020, arora2019exact, chizat_lazy_2020, luo_phase_2020} ($a=0.5$, $b_1=b_2=0$).
    \item Figures~\ref{fig:MNIST_MF_page1} and Figure \ref{fig:MNIST_MF_page2} show results for networks with scaling and initialization in the mean field setting \citep{mei2018mean, ChizatB18, luo_phase_2020, sirignano2020mean} ($a=0.5$, $b_1=b_2=0$).
\end{itemize}

\begin{figure*}[p]
\centering
\setlength{\tabcolsep}{2pt}
\renewcommand{\arraystretch}{0.4}

\begin{tabular}{cc}
\includegraphics[width=0.46\textwidth]{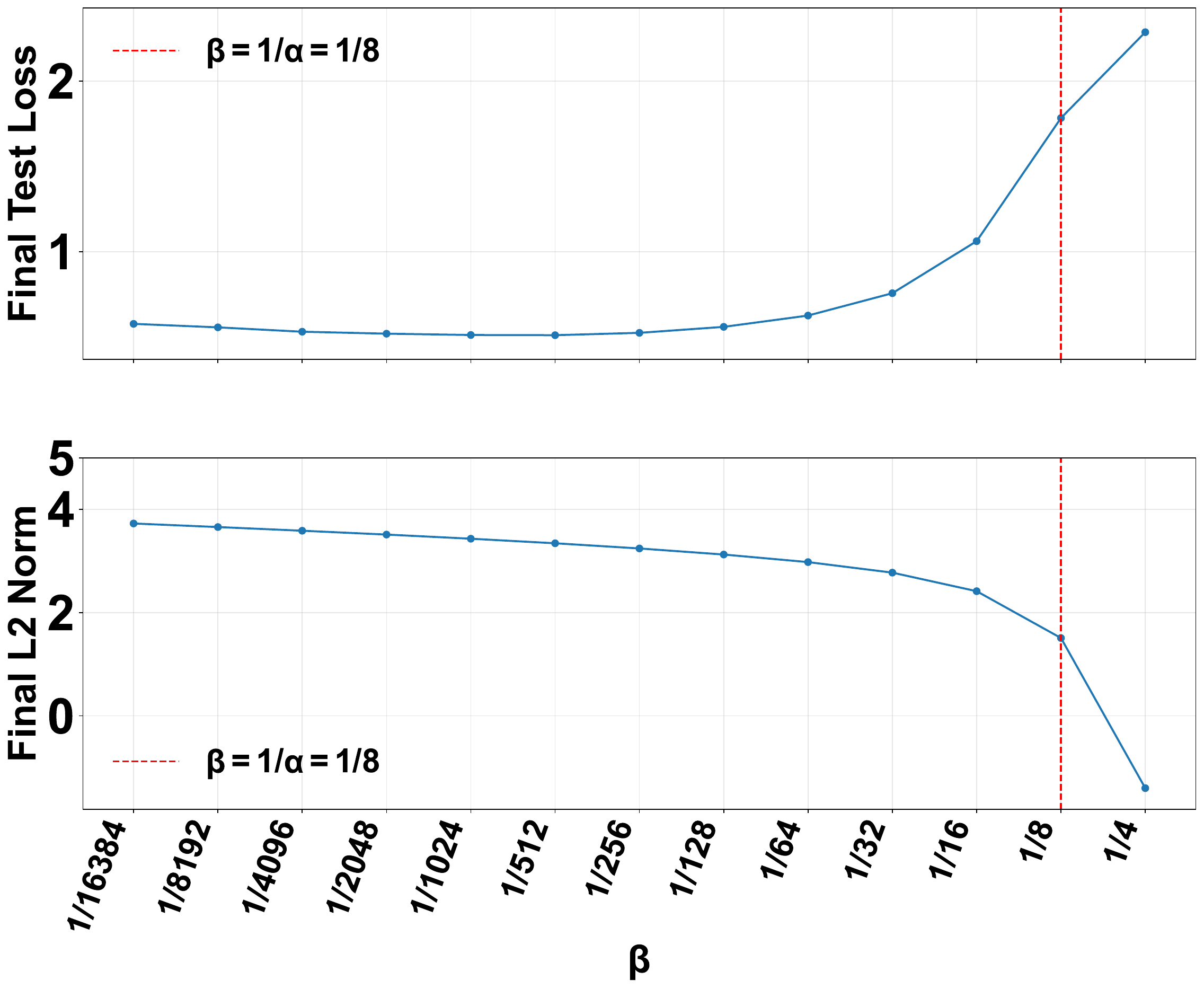} &
\includegraphics[width=0.46\textwidth]{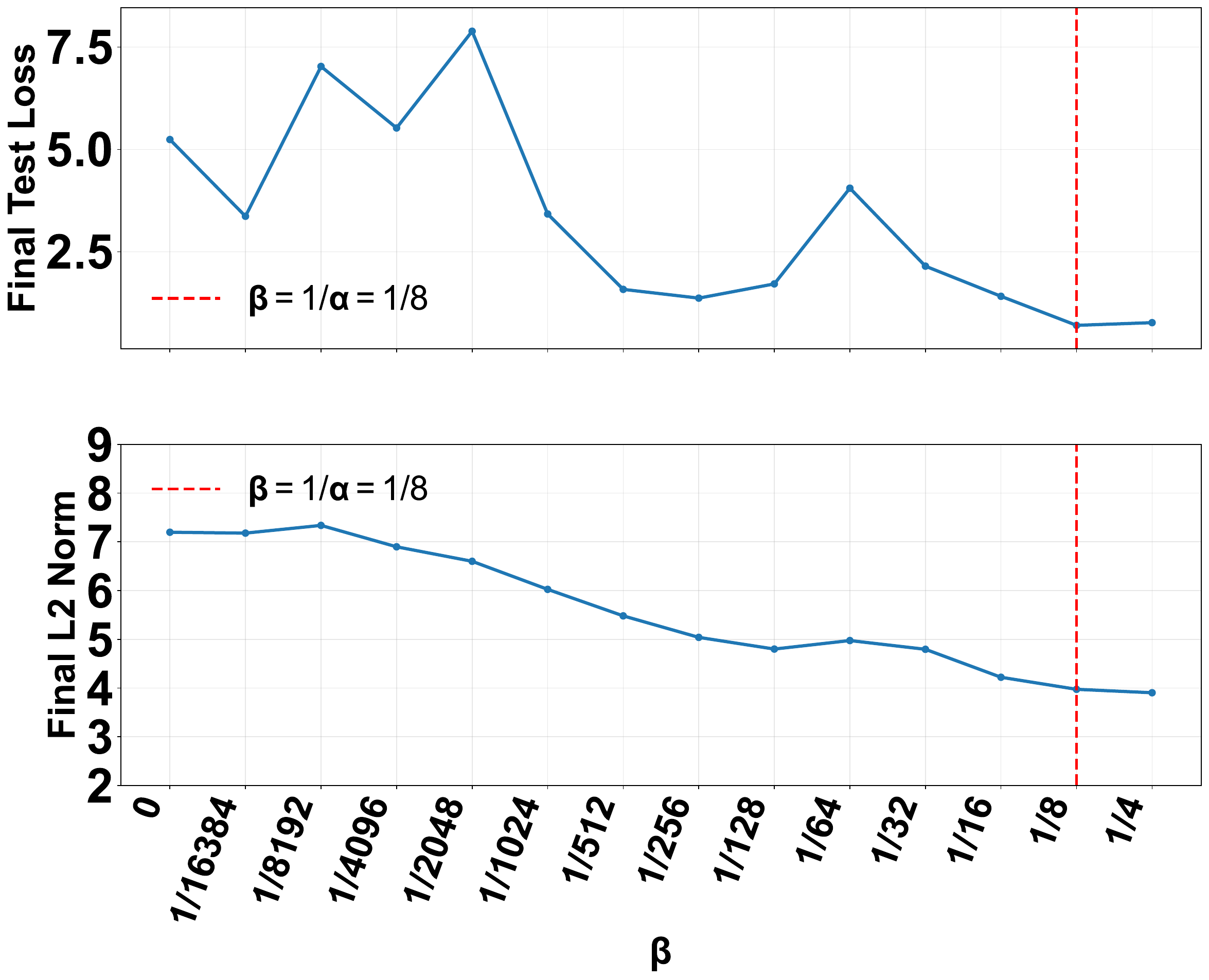} \\[-0.3em]

\includegraphics[width=0.46\textwidth]{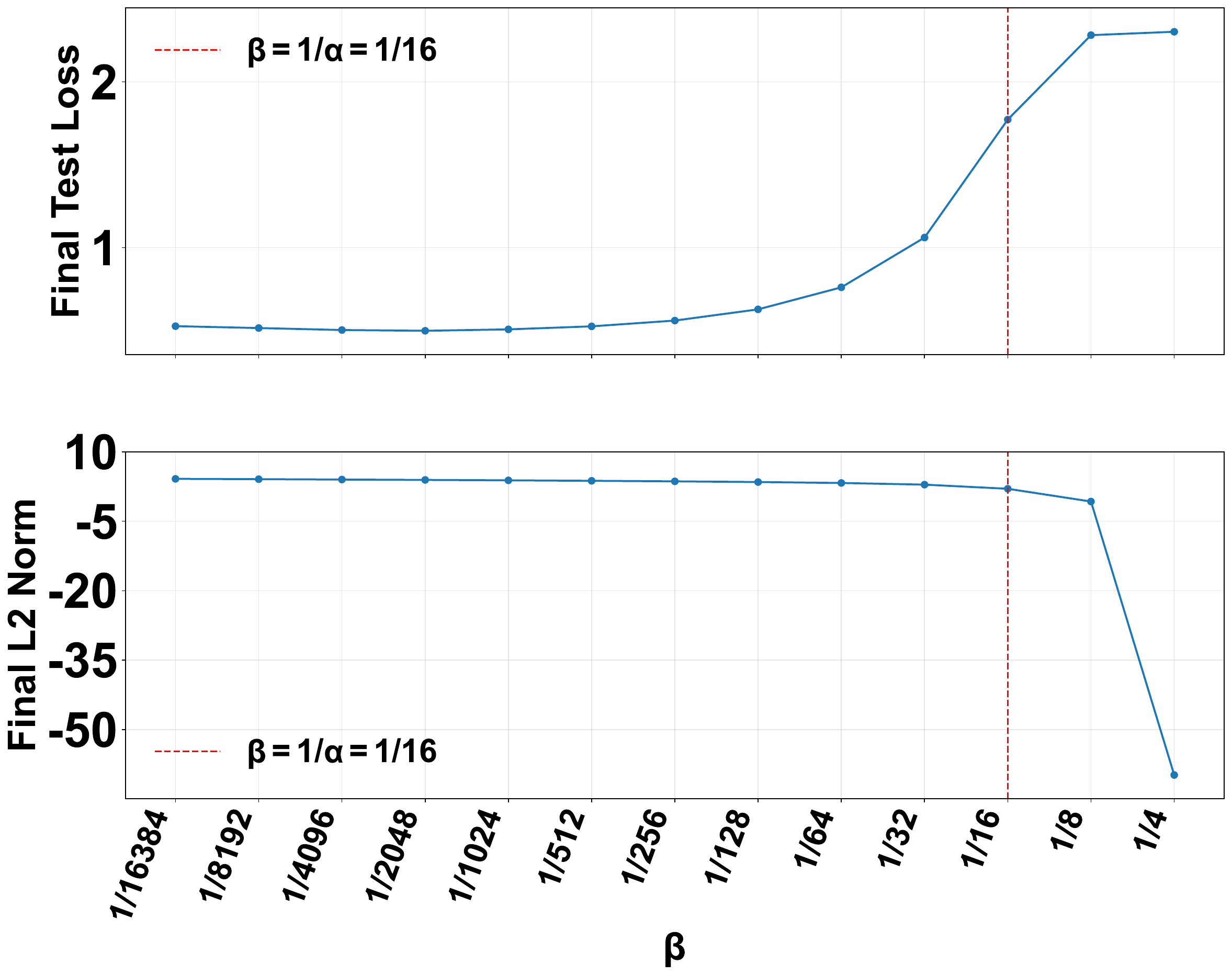} &
\includegraphics[width=0.46\textwidth]{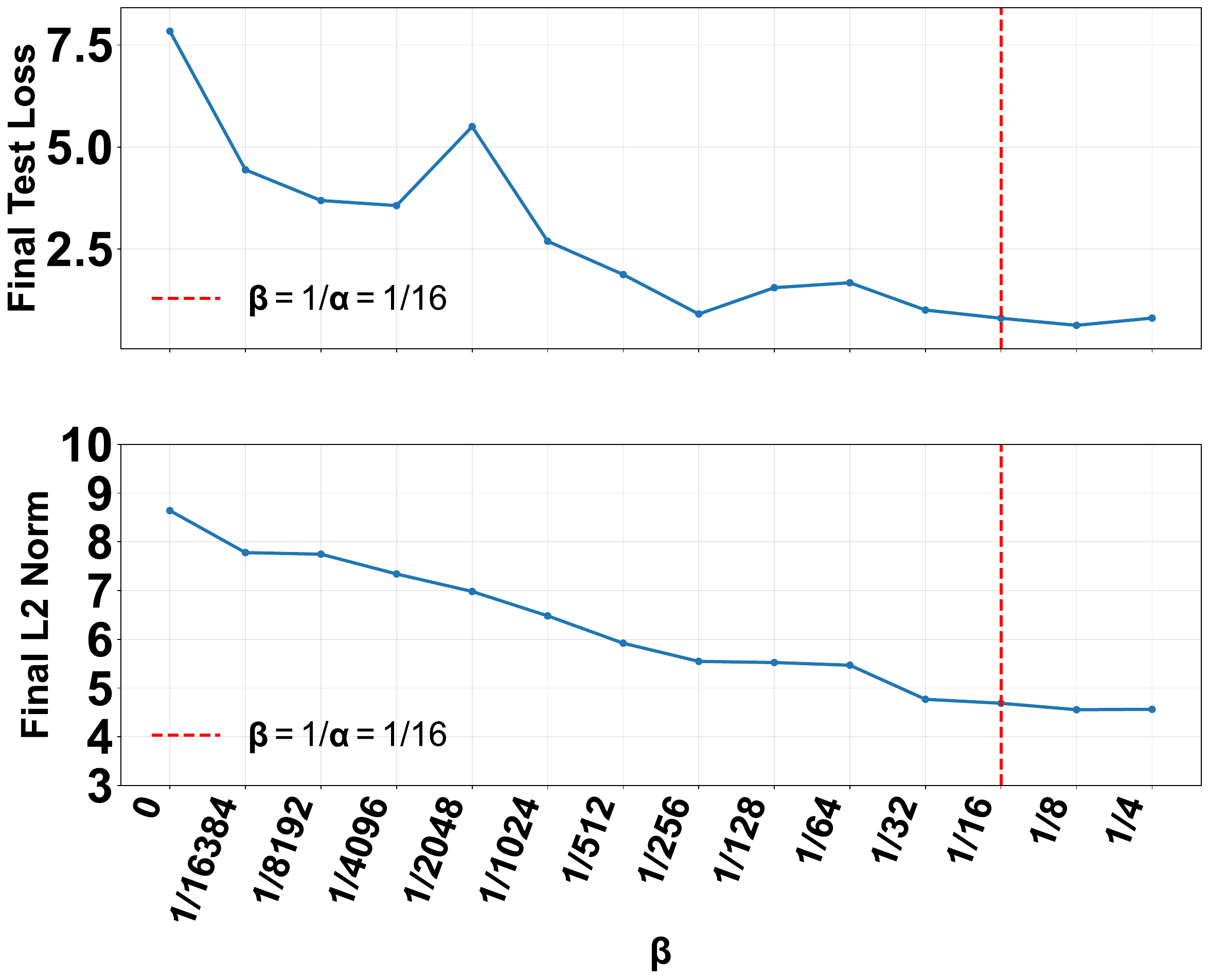} \\[-0.3em]

\includegraphics[width=0.46\textwidth]{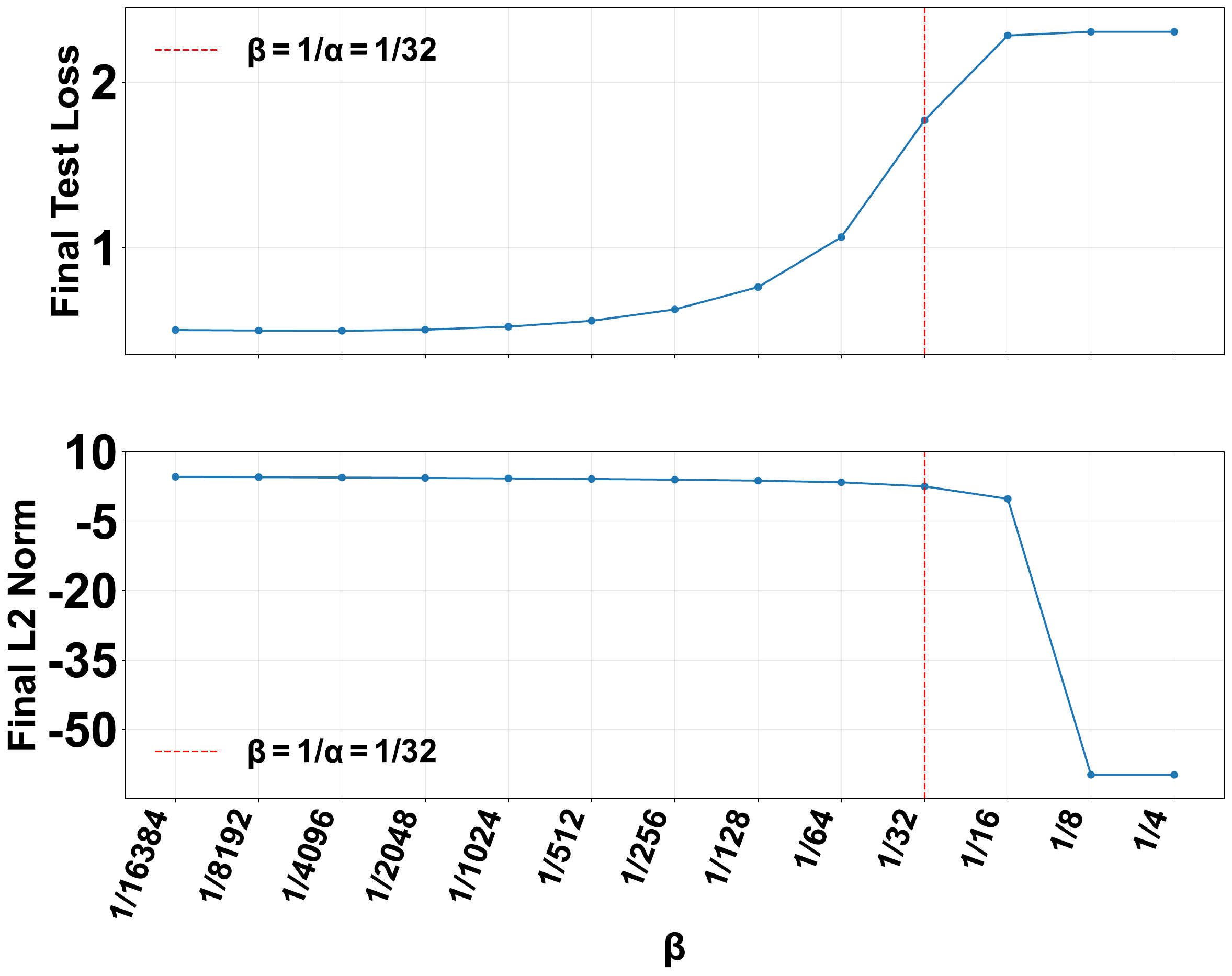} &
\includegraphics[width=0.46\textwidth]{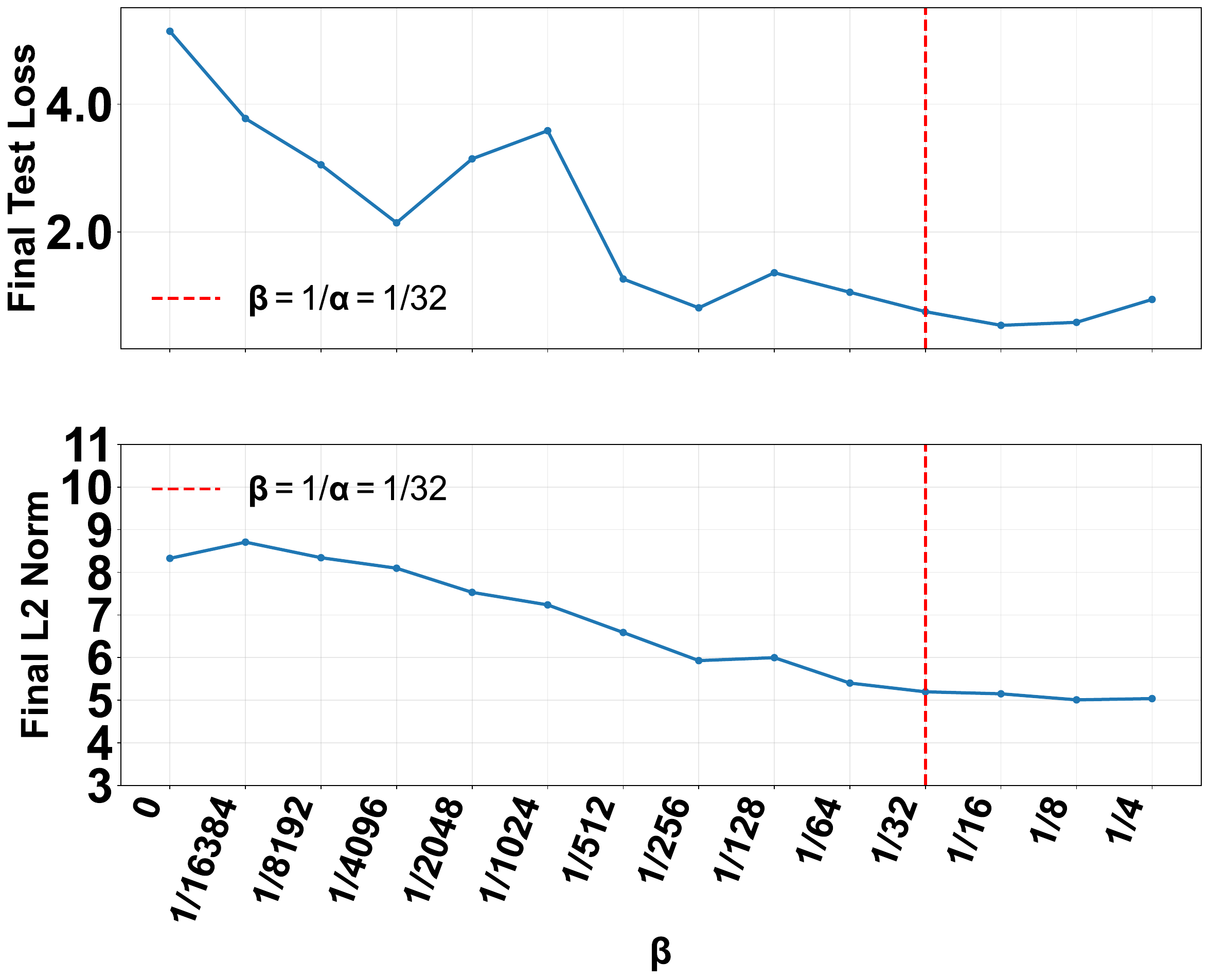}
\end{tabular}

\caption{Two-layer ReLU networks initialized by $\tau_1=\tau_2=0.01$ ($b_1=b_2=0$), scaled by $1/\alpha=1/\sqrt{m}~(a=0.5)$ (NTK regime) and trained for 100000 epochs on MNIST. Left column: SGD. Right column: AdamW. Rows show widths 64, 256, and 1024.}
\label{fig:MNIST_NTK_page1}
\end{figure*}

\begin{figure*}[p]
\centering
\setlength{\tabcolsep}{2pt}
\renewcommand{\arraystretch}{0.4}

\begin{tabular}{cc}
\includegraphics[width=0.46\textwidth]{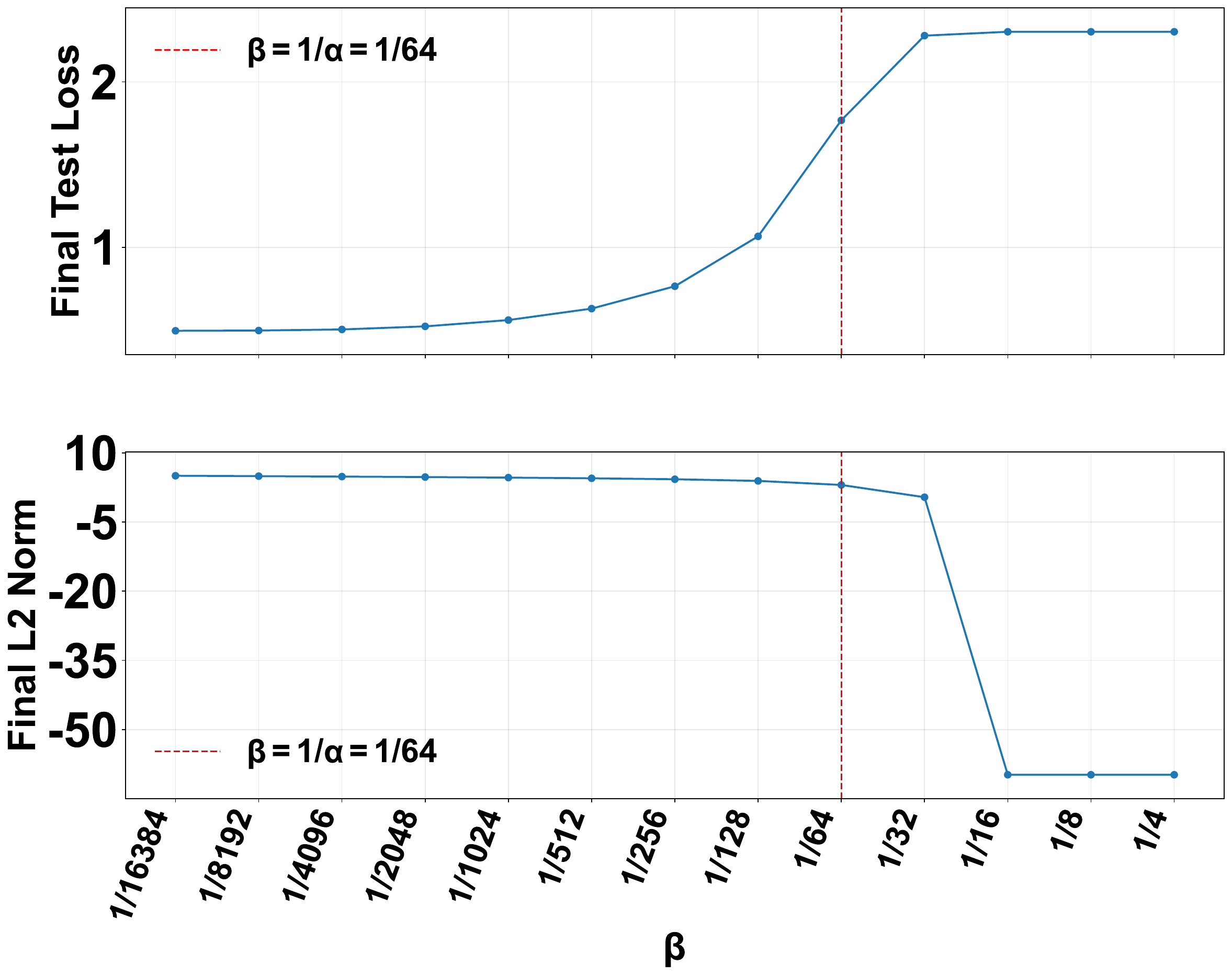} &
\includegraphics[width=0.46\textwidth]{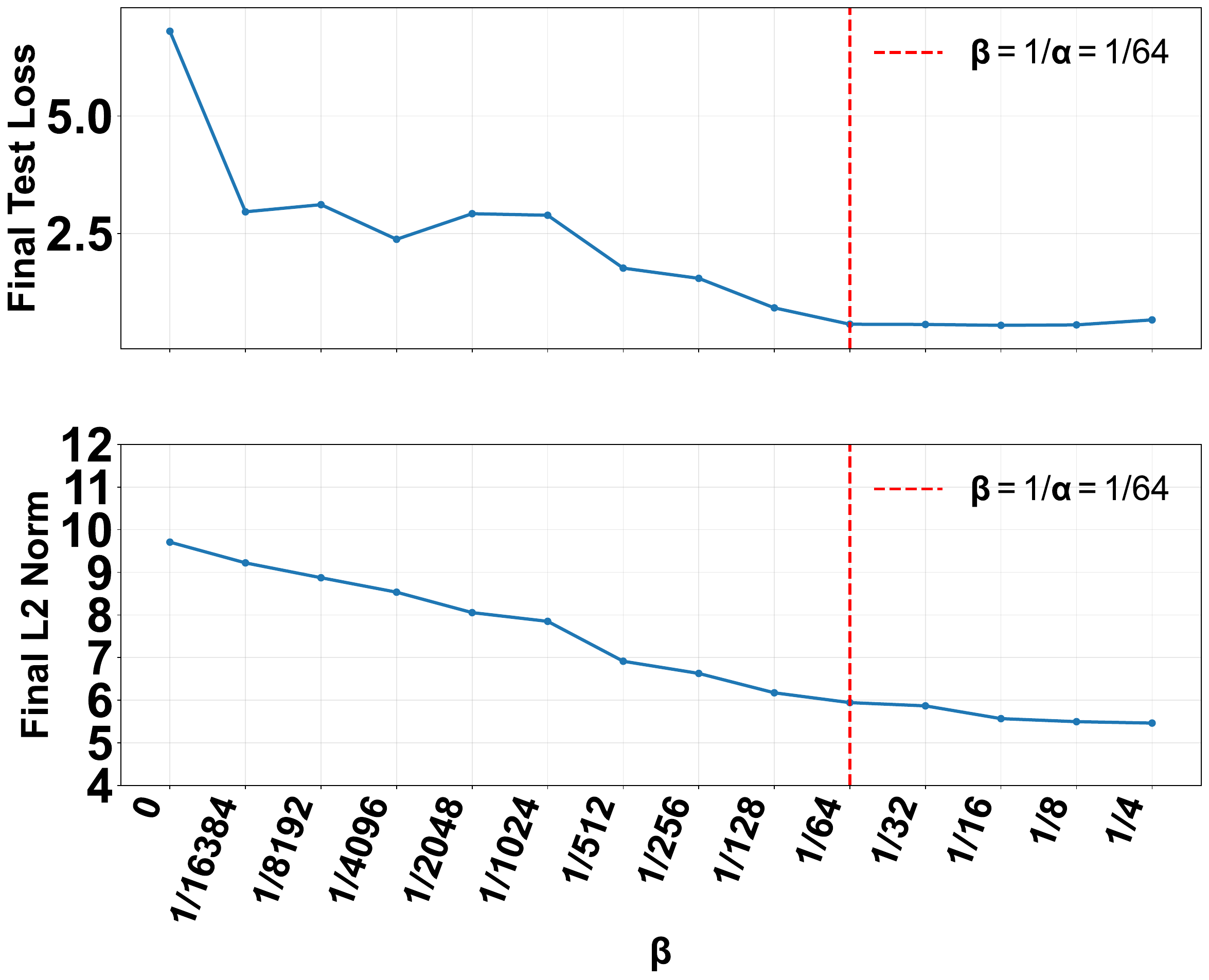} \\[-0.3em]

\includegraphics[width=0.46\textwidth]{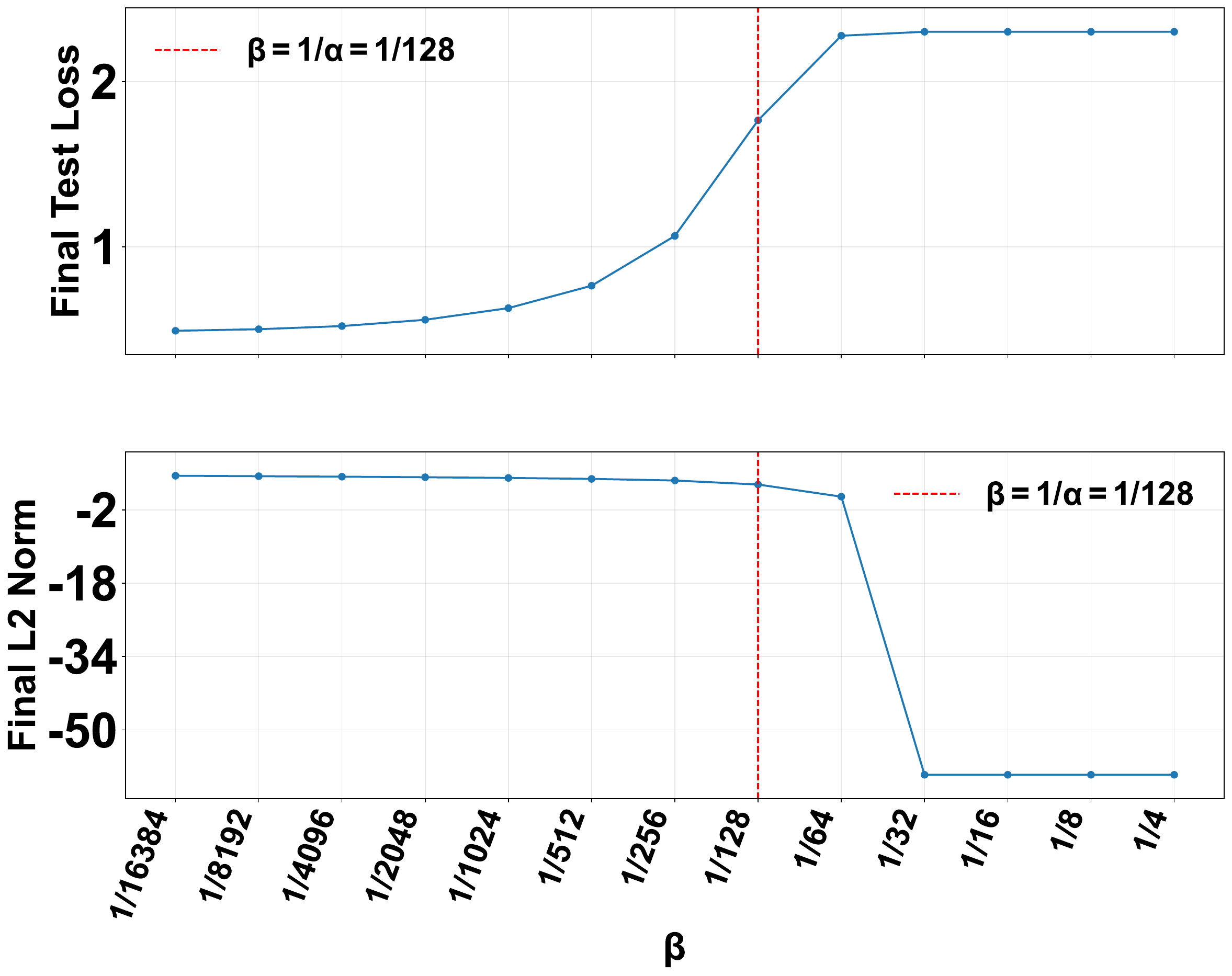} &
\includegraphics[width=0.46\textwidth]{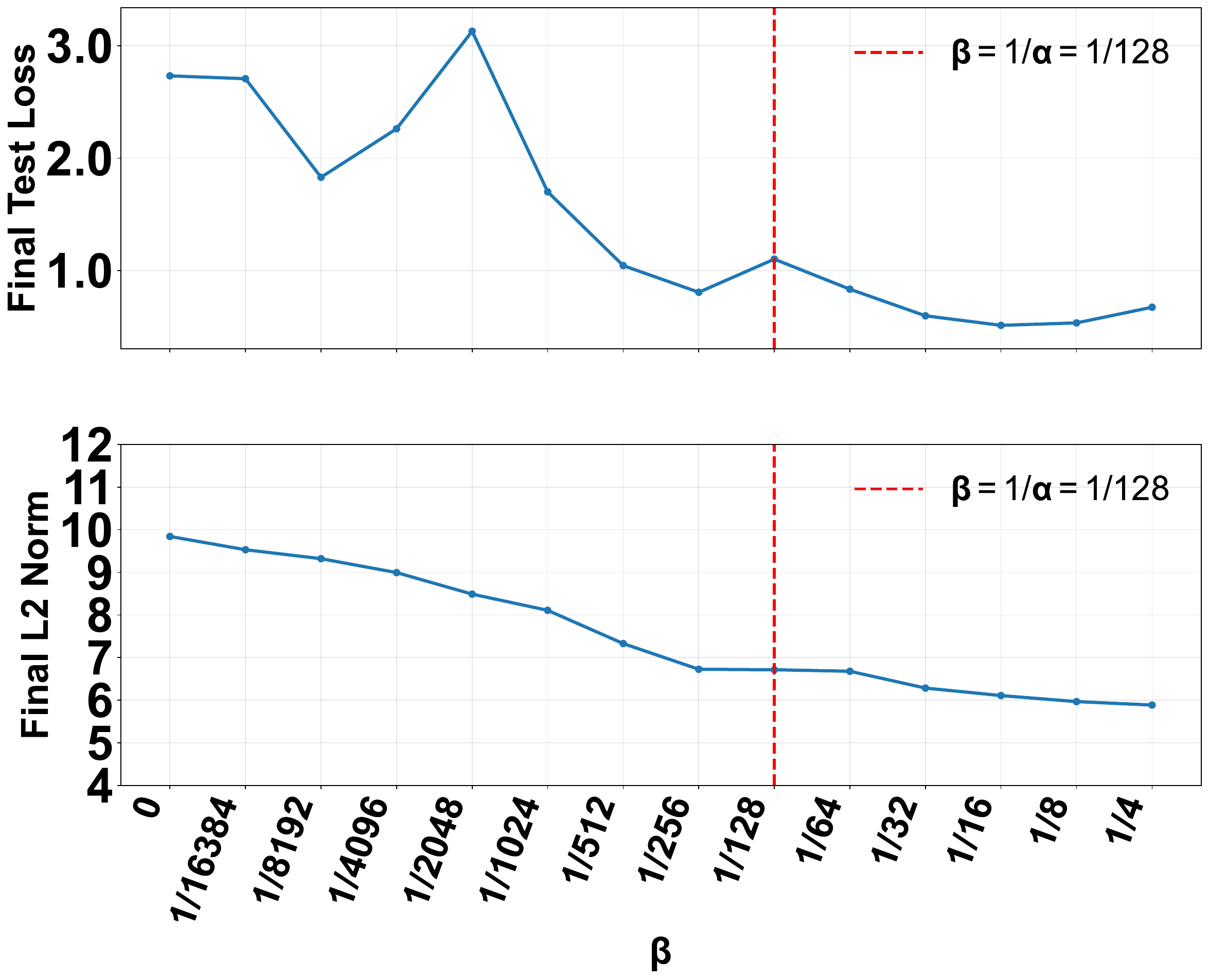}
\end{tabular}

\caption{Two-layer ReLU networks initialized by $\tau_1=\tau_2=0.01$ ($b_1=b_2=0$), scaled by $1/\alpha=1/\sqrt{m}~(a=0.5)$ (NTK regime) and trained for 100000 epochs on MNIST. Left column: SGD. Right column: AdamW. Rows show widths 4096 and 16384.}
\label{fig:MNIST_NTK_page2}
\end{figure*}

\begin{figure*}[p]
\centering
\setlength{\tabcolsep}{2pt}
\renewcommand{\arraystretch}{0.4}

\begin{tabular}{cc}
\includegraphics[width=0.46\textwidth]{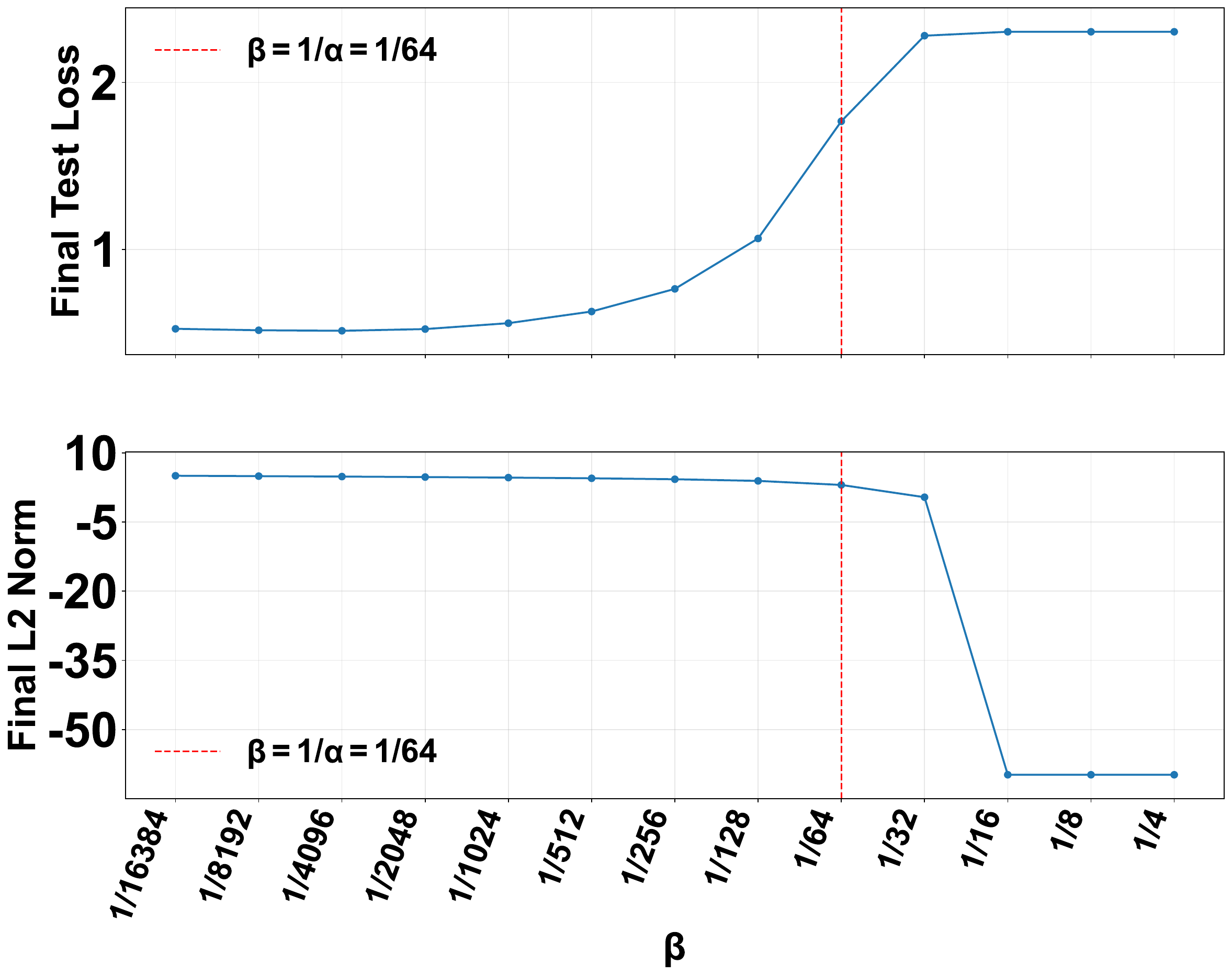} &
\includegraphics[width=0.46\textwidth]{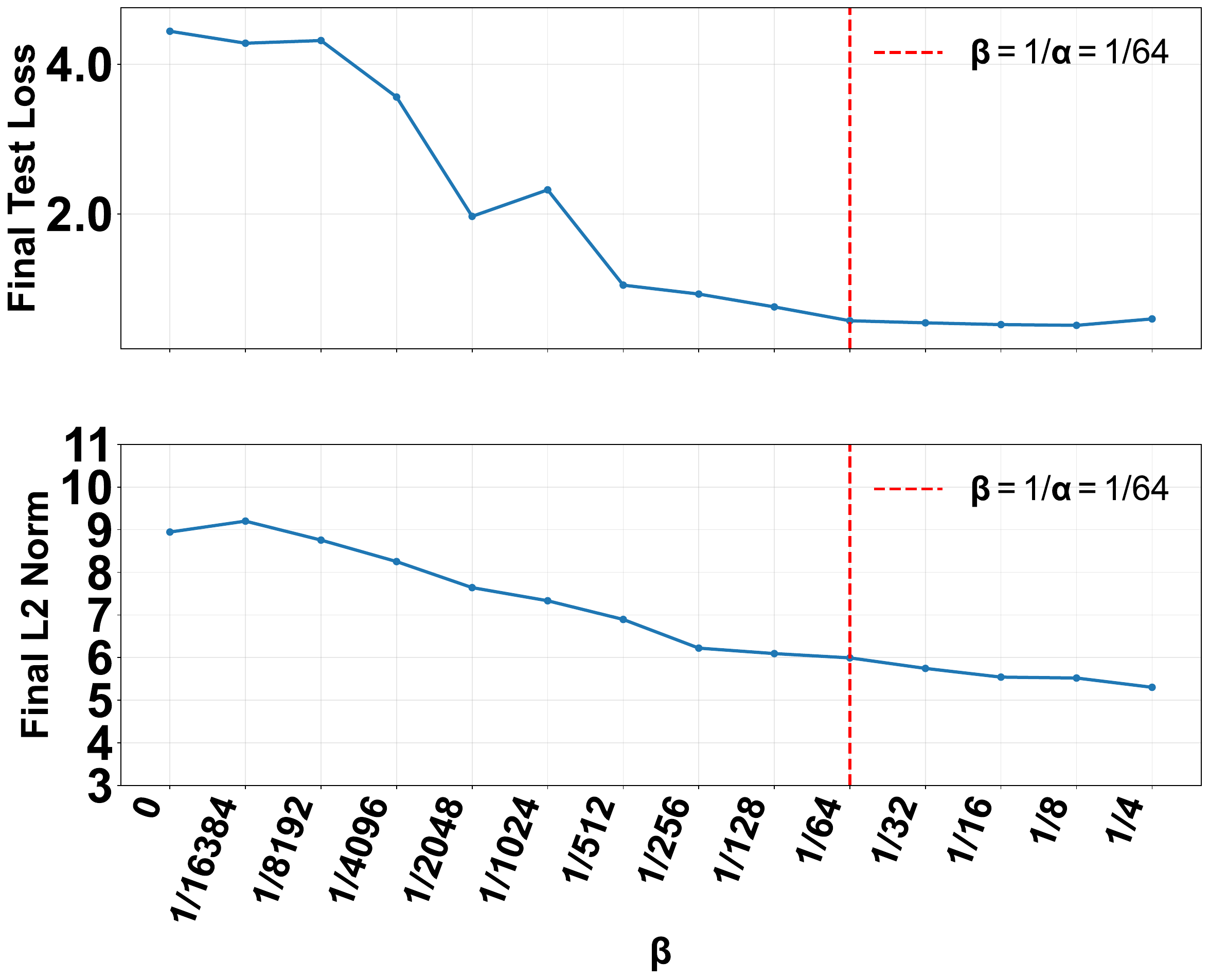} \\[-0.3em]

\includegraphics[width=0.46\textwidth]{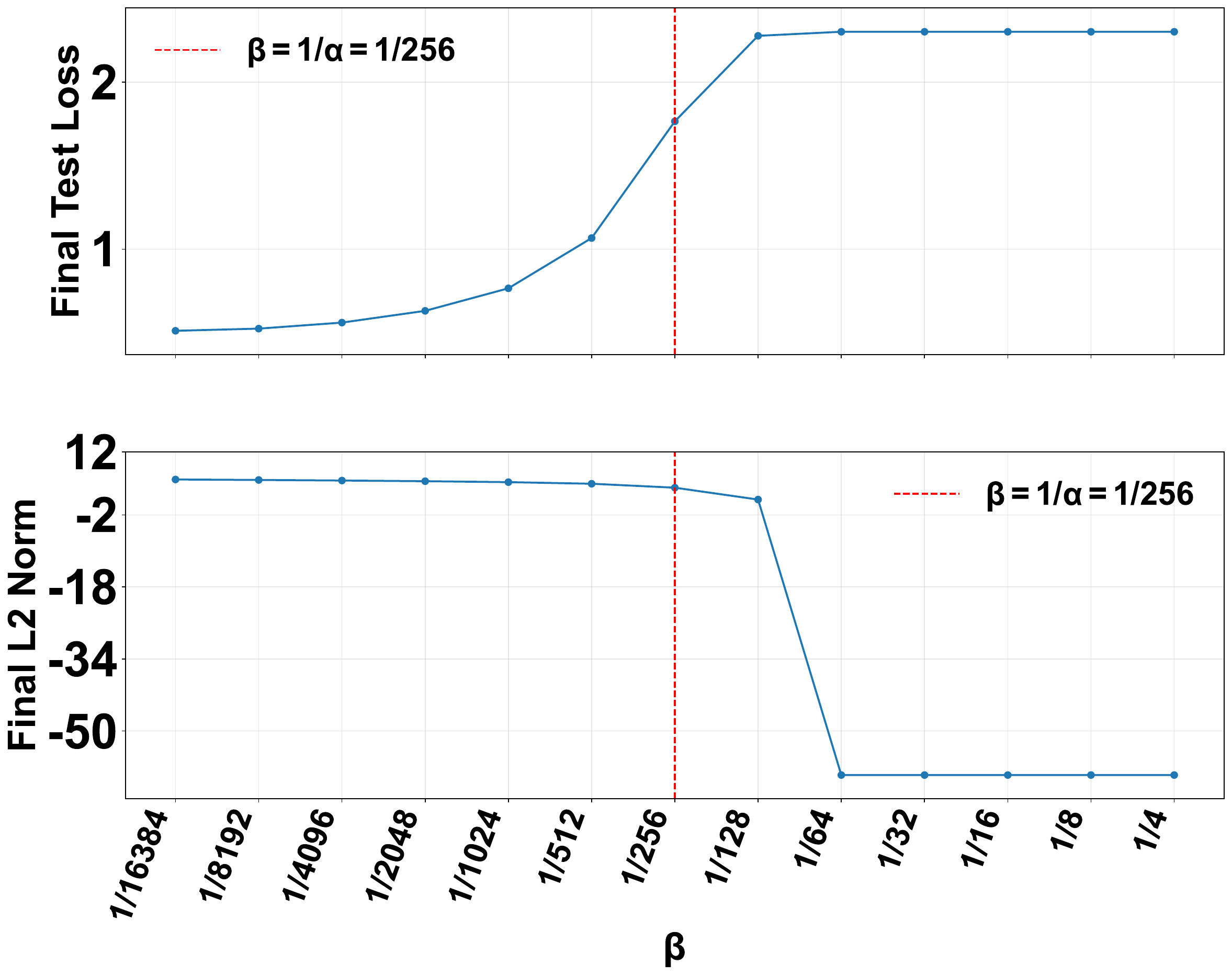} &
\includegraphics[width=0.46\textwidth]{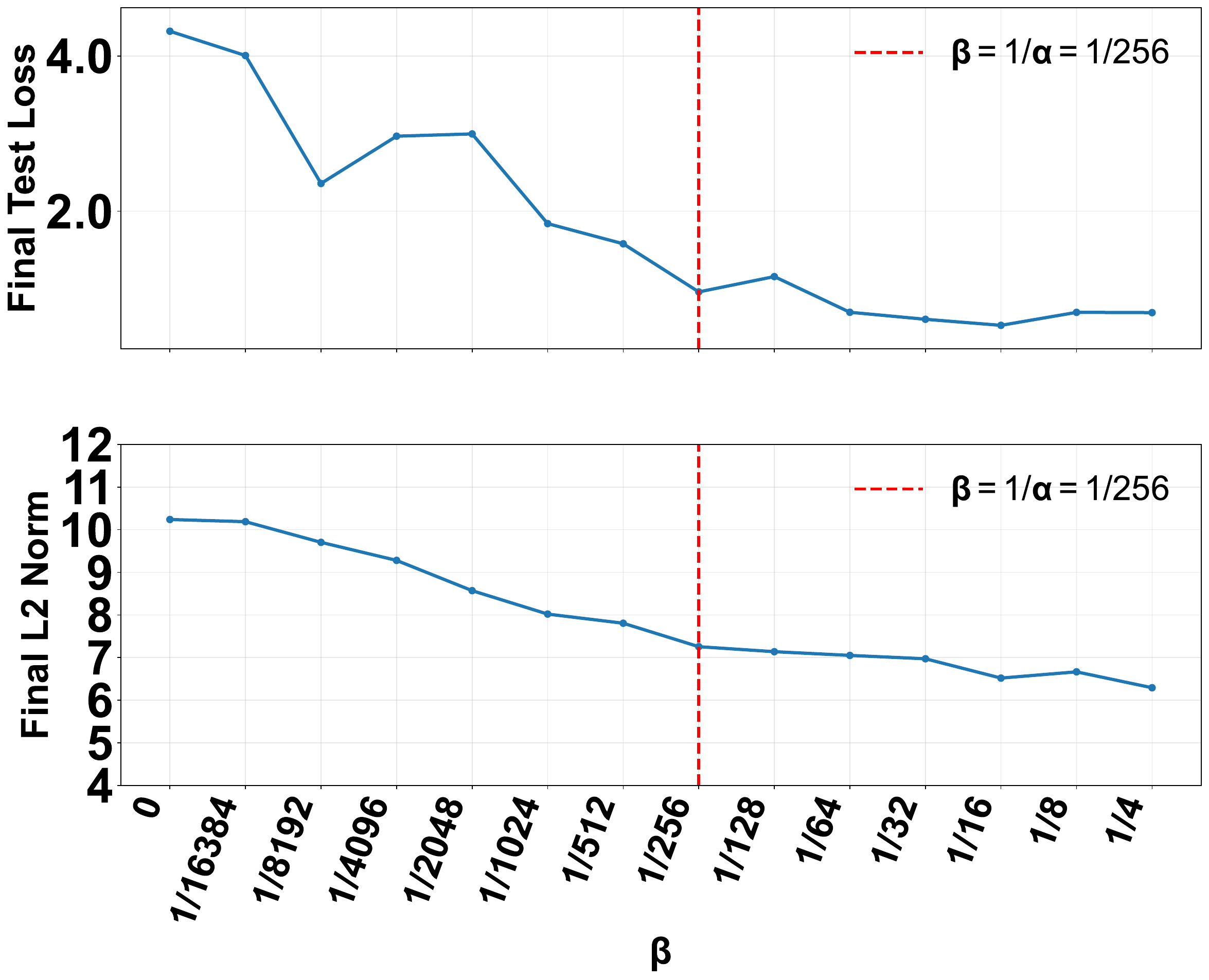} \\[-0.3em]

\includegraphics[width=0.46\textwidth]{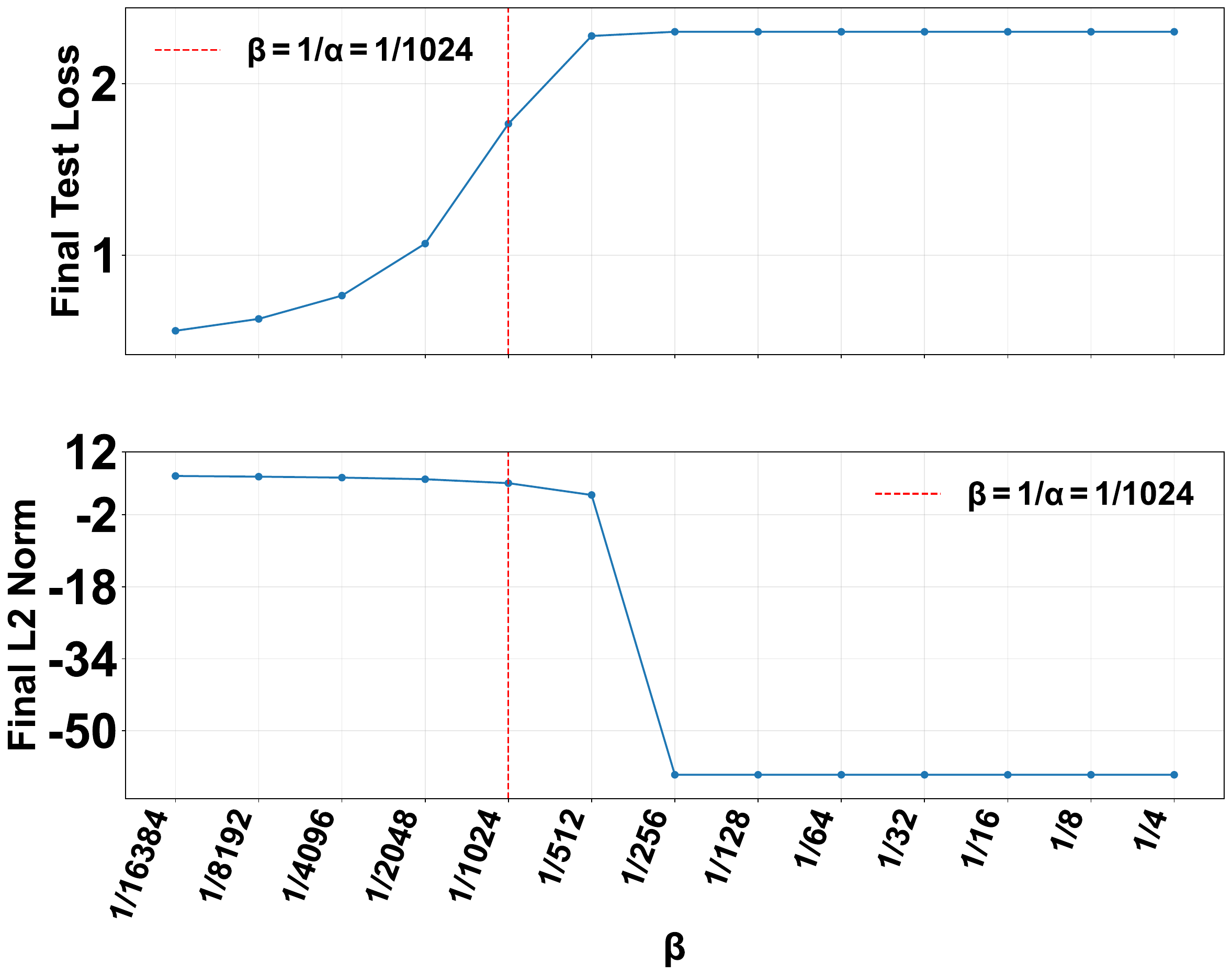} &
\includegraphics[width=0.46\textwidth]{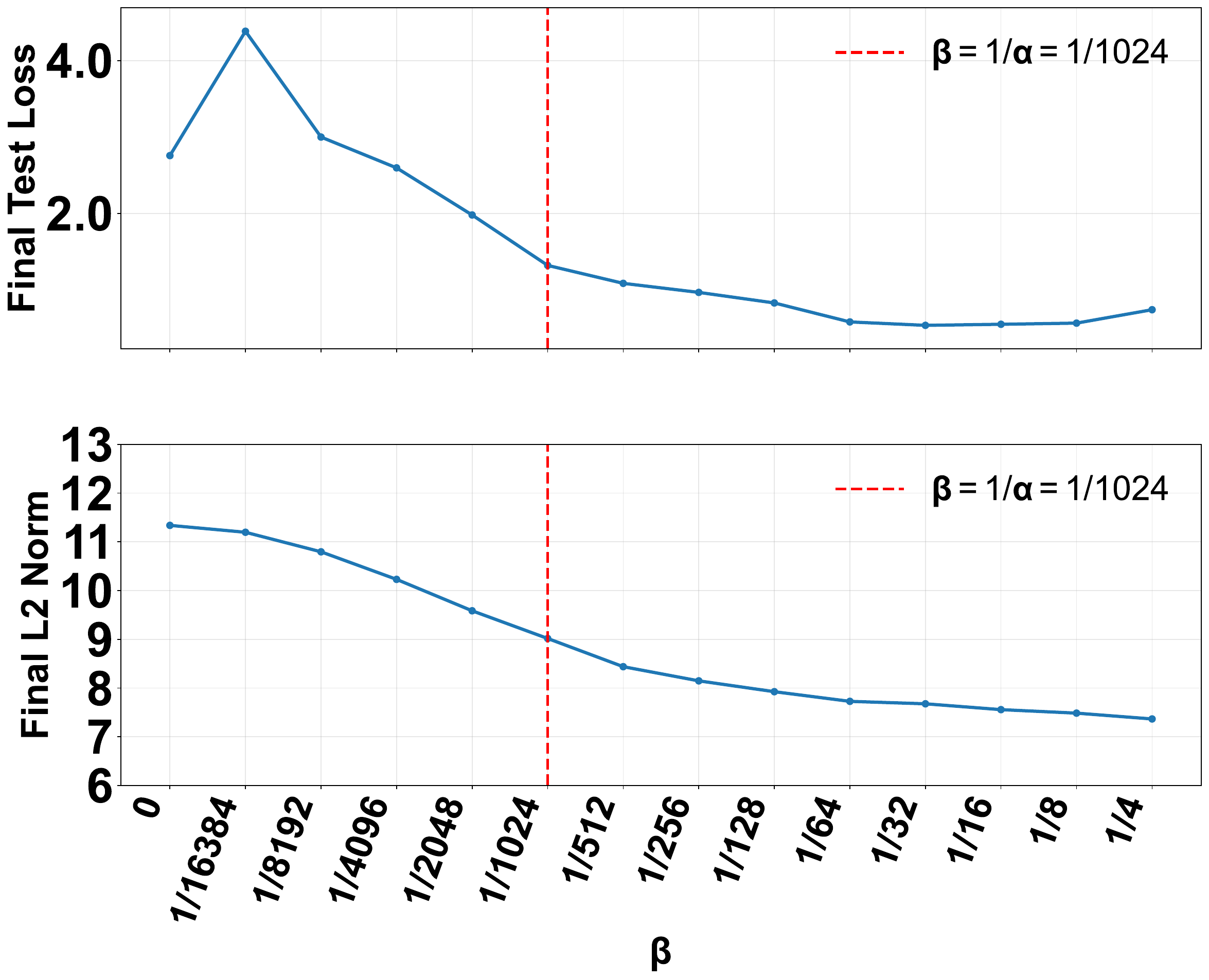}
\end{tabular}

\caption{Two-layer ReLU networks initialized by $\tau_1=\tau_2=0.01$ ($b_1=b_2=0$), scaled by $1/\alpha=1/m~(a=1)$ (mean field regime) and trained for 100000 epochs on MNIST. Left column: SGD. Right column: AdamW. Rows show widths 64, 256, and 1024.}
\label{fig:MNIST_MF_page1}
\end{figure*}

\begin{figure*}[p]
\centering
\setlength{\tabcolsep}{2pt}
\renewcommand{\arraystretch}{0.4}

\begin{tabular}{cc}
\includegraphics[width=0.46\textwidth]{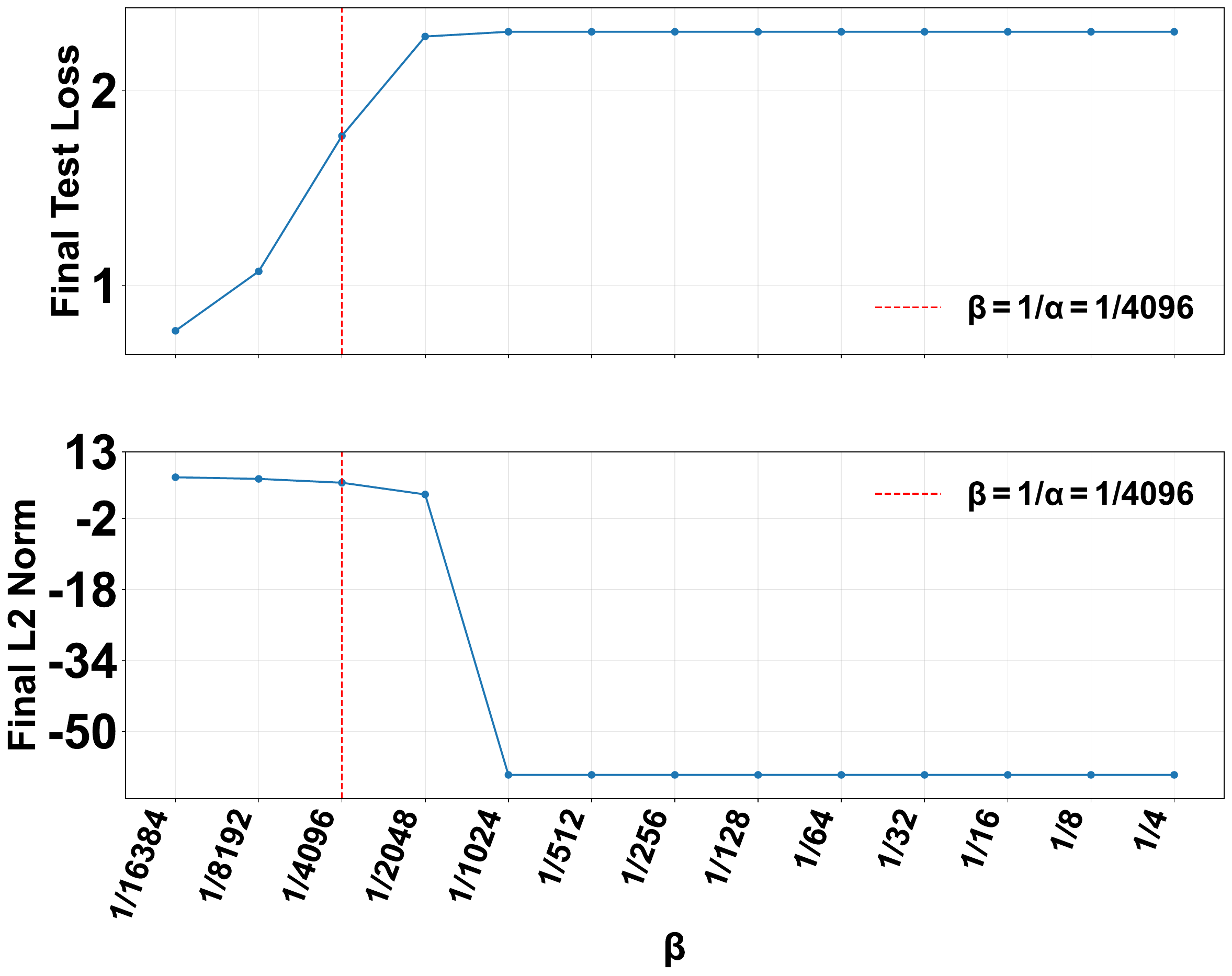} &
\includegraphics[width=0.46\textwidth]{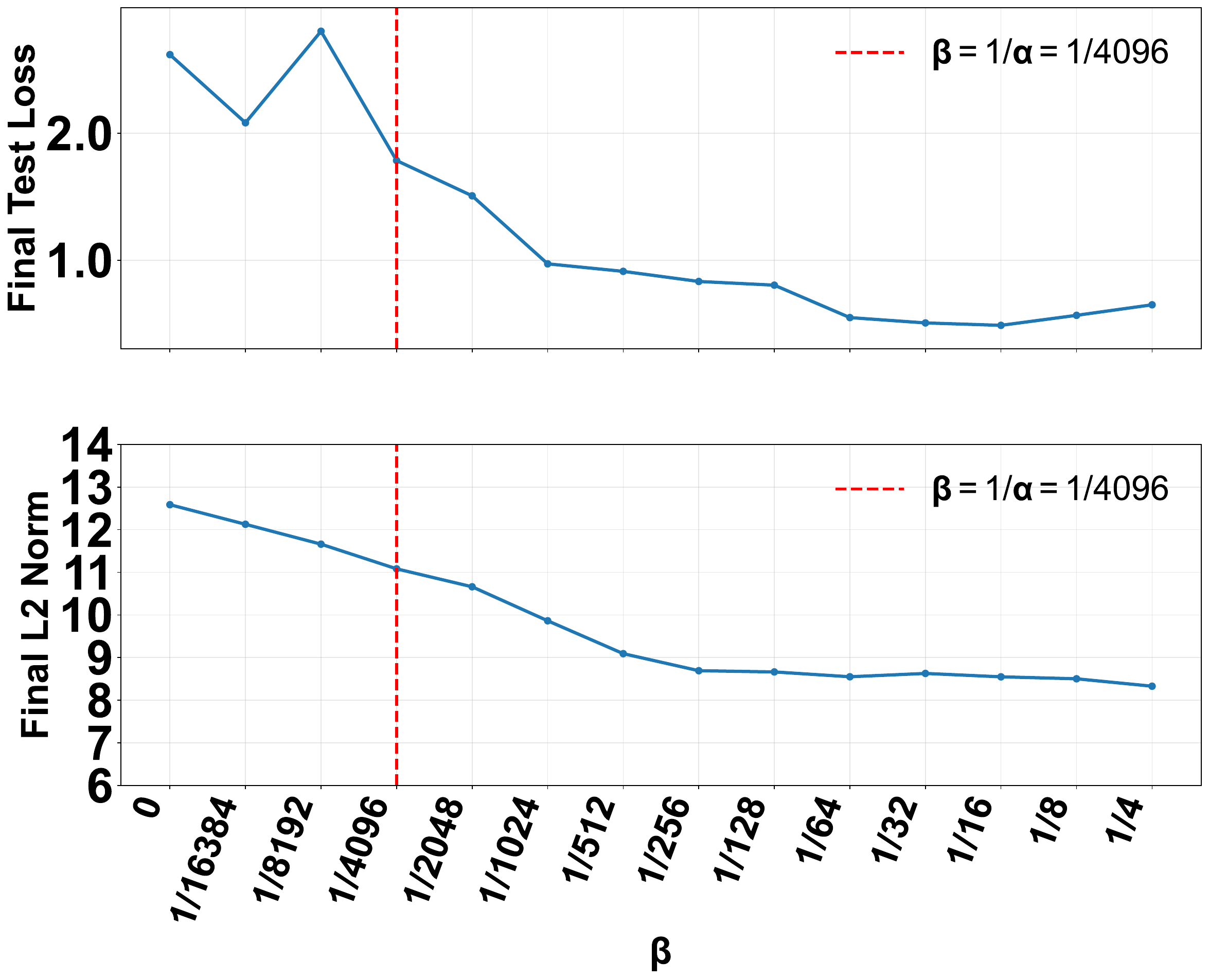} \\[-0.3em]

\includegraphics[width=0.46\textwidth]{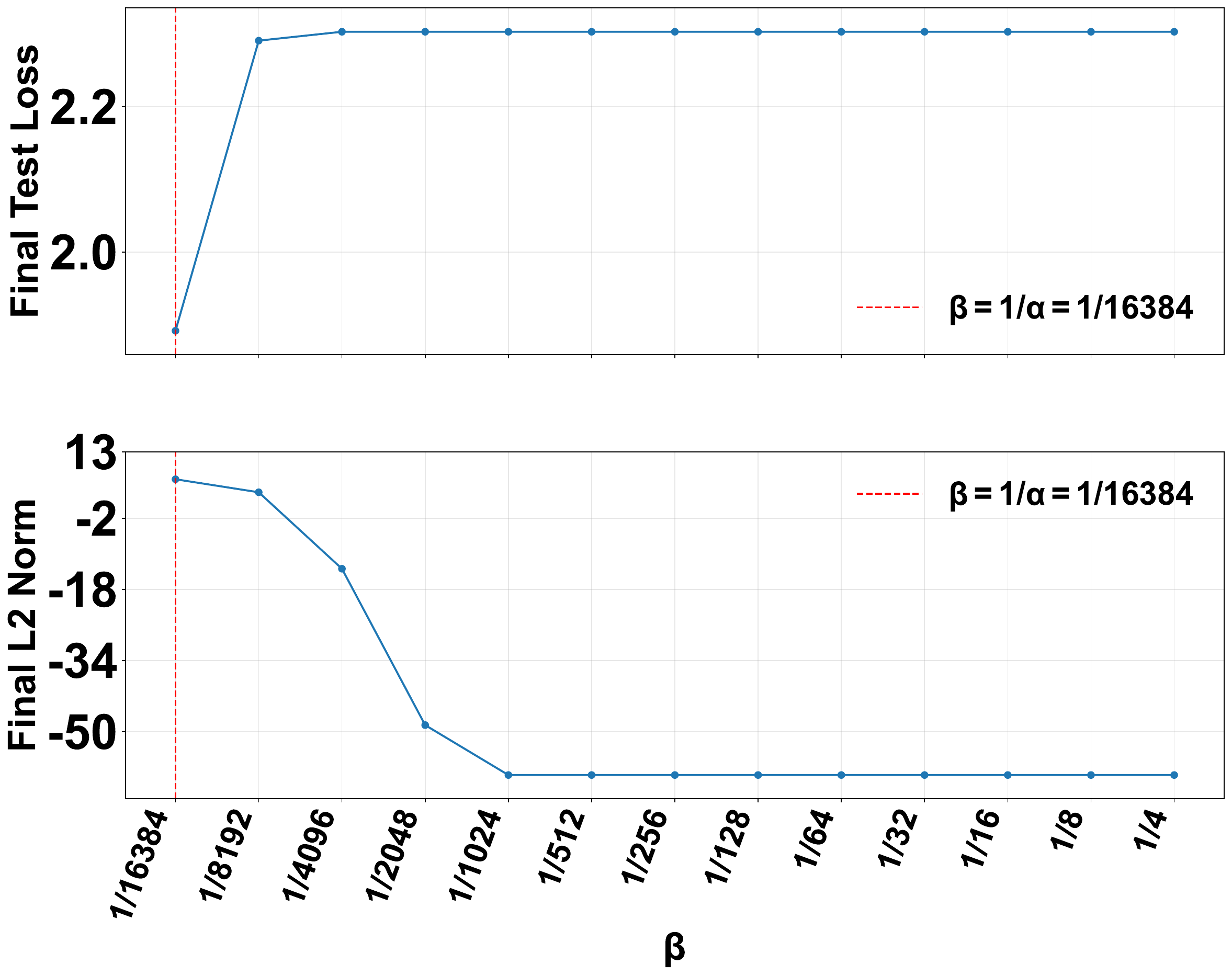} &
\includegraphics[width=0.46\textwidth]{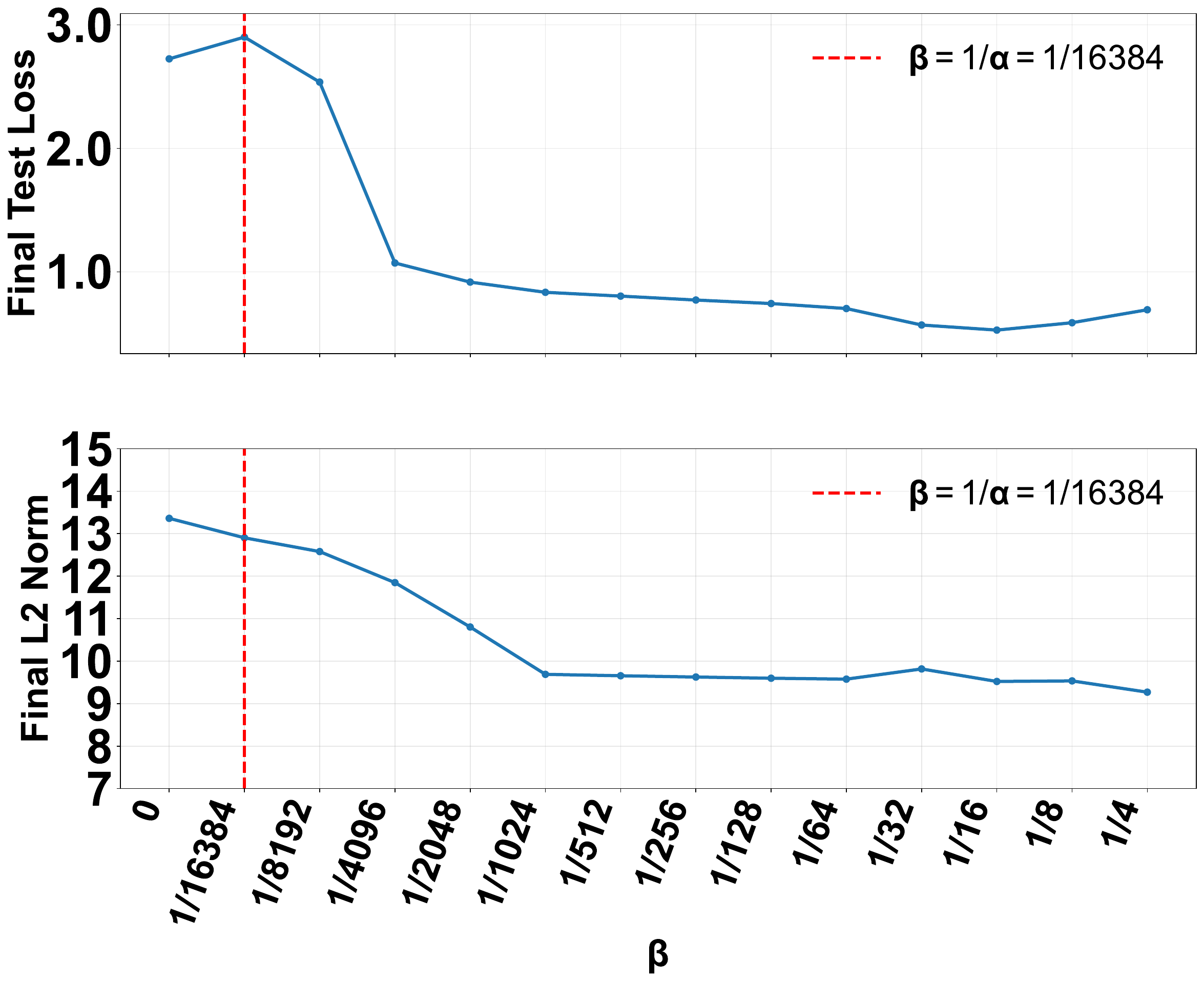}
\end{tabular}

\caption{Two-layer ReLU networks initialized by $\tau_1=\tau_2=0.01$ ($b_1=b_2=0$), scaled by $1/\alpha=1/m~(a=1)$ (mean field regime) and trained for 100000 epochs on MNIST. Left column: SGD. Right column: AdamW. Rows show widths 4096 and 16384.}
\label{fig:MNIST_MF_page2}
\end{figure*}

\subsection{Computing Environment}
Experiments ran on NVIDIA A100-SXM4-80GB GPUs (CUDA 12.2) and AMD EPYC 7763 CPUs. Table~\ref{tab:runtime-env} provides detailed hardware and software specifications.

\begin{table}[t!]
\centering
\caption{Runtime hardware and software.}
\label{tab:runtime-env}
\begin{tabular}{ll}
\toprule
\multicolumn{2}{l}{\textbf{CPU}}\\
\quad Model name & AMD EPYC 7763 64-Core Processor\\
\quad \# CPU(s) & 128\\
\midrule
\multicolumn{2}{l}{\textbf{GPU}}\\
\quad Product Name & NVIDIA A100-SXM4-80GB\\
\quad CUDA Version & 12.2\\
\midrule
\multicolumn{2}{l}{\textbf{PyTorch}}\\
\quad Version & 2.7.1\\
\bottomrule
\end{tabular}
\end{table}

\end{document}